\documentclass[10pt]{article} 
\usepackage[preprint]{tmlr}

\usepackage{amsmath,amsfonts,bm}

\def\eqref#1{equation~\ref{#1}}

\def\1{\bm{1}}

\DeclareMathAlphabet{\mathsfit}{\encodingdefault}{\sfdefault}{m}{sl}
\SetMathAlphabet{\mathsfit}{bold}{\encodingdefault}{\sfdefault}{bx}{n}

\newcommand{\R}{\mathbb{R}}

\DeclareMathOperator*{\argmax}{arg\,max}

\usepackage[hypertexnames=false]{hyperref}
\usepackage{url}

\usepackage{amsmath, amssymb, amsthm}
\usepackage{algorithm}
\usepackage{algpseudocode}
\algtext*{EndIf}
\algtext*{EndFor}
\algtext*{EndWhile}
\usepackage{multirow}
\usepackage{booktabs}
\usepackage{float}
\usepackage{soul}

\newtheorem{definition}{Definition}
\newtheorem{lemma}{Lemma}
\newtheorem{theorem}{Theorem}
\newtheorem{proposition}{Proposition}  

\newcommand{\Xin}{\mathcal{X}_{in}}
\newcommand{\Reach}{\mathcal{R}}

\title{TNODEV: Toolbox for Neural ODE Verification}

\author{\name Abdelrahman Sayed Sayed \email abdelrahman.ibrahim@univ-eiffel.fr \\
      \addr Univ Gustave Eiffel, COSYS-ESTAS, F-59657 Villeneuve d’Ascq, France
      \AND
      \name Pierre-Jean Meyer \email pierre-jean.meyer@univ-eiffel.fr \\
      \addr Univ Gustave Eiffel, COSYS-ESTAS, F-59657 Villeneuve d’Ascq, France
      \AND
      \name Mohamed Ghazel \email mohamed.ghazel@univ-eiffel.fr\\
      \addr Univ Gustave Eiffel, COSYS-ESTAS, F-59657 Villeneuve d’Ascq, France
      }

\def\month{10}  
\def\year{2026} 
\def\openreview{\url{https://openreview.net/forum?id=3lbdS0A1N5}} 

\usepackage{soul}

\begin{document}

\maketitle

\begin{abstract}

Neural ordinary differential equations (neural ODE) gained attention in safety critical settings such as continuous-time controllers for cyber-physical systems and classifiers integrated into automated decision pipelines, raising the question whether their behavior can be formally verified. Existing tools dedicated to neural ODE provide only a single reachability call without iterative input-set refinement, limiting the precision of their verdicts to whatever one reachability call can deliver. We present TNODEV, the first formal verifier for neural ODE that integrates a falsification checker, a fast interval-based reachability backend based on continuous-time mixed monotonicity, a verification and refinement loop with three input-set splitting heuristics, and a parallel scheduler in a single end-to-end pipeline. TNODEV supports safe-set inclusion verification on pure neural ODE, neural ODE in closed loop with a neural network controller and general neural ODE (GNODE), with the safe set specified either as an interval or as the half-space intersection induced by a target classification label. We evaluate TNODEV on a range of benchmarks across safe-set inclusion and classification-robustness properties, including a direct reachability comparison against NNV 2.0 and CORA and a verification comparison against NNV 2.0 on MNIST general neural ODE classifiers.
\end{abstract}









\section{Introduction}
\label{sect:introduction}

Neural ordinary differential equations (neural ODE) \citep{rico1992discrete,chen2018neural} parameterize the derivative of a hidden state by a neural network, so that the forward pass is the solution of an initial value problem rather than the application of a finite stack of layers. Concretely, we consider the autonomous neural ODE

\begin{equation}
\label{eq:node}
\dot{x}(t) \;=\; \frac{\mathrm{d}x(t)}{\mathrm{d}t} \;=\; f(x(t)), \qquad x(0) = u,
\end{equation}

in which the right-hand side $f$ is parameterized by a neural network. This continuous-time formulation has been shown to be effective across a range of tasks such as time series modeling \citep{rubanova2019latent}, generative modeling \citep{grathwohl2018ffjord}, and as continuous-depth substitutes for residual networks \cite{chen2018neural}. As neural ODE are finding applications in higher-stakes settings, including the controllers of safety critical cyber-physical systems \citep{lechner2020neural,xiao2023forward} and classifiers integrated into automated decision pipelines \citep{moon2022survlatent,qian2021integrating}, the question of whether their behavior can be formally verified becomes increasingly relevant.

Formal verification of neural ODE differs from formal verification of standard neural networks in two ways. First, the forward pass is the solution of a continuous-time initial value problem, which means that bounding the reachable set requires propagating an input set through the flow of the differential equation rather than through a finite stack of discrete layers. Second, the trade-offs of standard neural network verification, such as the choice of set representation, the management of branching in non-piecewise-linear activations, and the use of input-set refinement to recover precision lost to over-approximation, all carry over to the neural ODE setting but interact differently with the continuous-time dynamics. As a result, the available verification tools for neural ODE are much less mature than their neural network counterparts, and existing tools differ substantially in the type of guarantees they provide and in the verification workflow they expose to the user.

A look at the current literature (Table~\ref{tab:tool-comparison}) shows that all existing tools dedicated to neural ODE contain only a reachability step, returning a single sound over-approximation of the reachable set per query. Currently, no tool implements the full verification pipeline that has become standard in the neural network verification literature \citep{liu2021algorithms,bunel2020branch,wang2021beta}. Such a pipeline consists of a falsification pass that searches for explicit counter-examples, a reachability call that computes a sound over-approximation, and an iterative input-set refinement loop that splits inconclusive cells and re-runs the reachability call on each sub-cell until the property is certified or a maximum number of refinement iterations is reached. The closest existing tool to a full pipeline is NNV 2.0 \citep{lopez2023nnv}, which checks safety specifications on the output of a single reachability call but does not perform input-set refinement. Since current neural ODE verifiers omit this refinement step, their verification decision is based on a single reachability call, whose over-approximation might be too conservative to guarantee the satisfaction of the desired properties, and no attempt is made to tighten it in the hope of improving the verification result.

In this paper, we present \textbf{TNODEV}, a formal verifier that combines a falsification module, a reachability backend, an iterative input-set refinement loop with three splitting heuristics, and a parallel scheduler in a single end-to-end verification pipeline for the first time in the neural ODE literature. The \emph{reachability} backend is based on the continuous-time mixed monotonicity (CTMM) from the literature \citep{meyer2021interval}, the \emph{falsification} module is a standard sampling-based counter-example search, and the \emph{iterative input-set refinement} is standard practice in neural network verification \citep{bunel2020branch,wang2021beta}, which has never been transferred to the continuous-time setting (Section~\ref{subsec:verif_loop}). The contribution of this paper is therefore not on the design of these individual components of the pipeline, but rather on the full verification architecture and its application to neural ODE: (i) the first end-to-end verification pipeline for neural ODE, (ii) design and comparison of three input-splitting heuristics, (iii) a soundness result for the closed-form linear-ODE interval propagator used in the GNODE setting (Appendix~\ref{apd:classification_nODE_Hybrid_reachability_pipeline}), and an explicit characterization of the scope of the soundness guarantee of the whole pipeline (Appendix~\ref{apd:soundness_scope}), (iv) an open-source MATLAB toolbox and an evaluation that identifies the settings in which TNODEV is the appropriate tool.


The paper is organized as follows. Section~\ref{sect:related_works} reviews existing tools for neural ODE reachability and verification, and positions TNODEV against them. Section~\ref{sect:prelimin} introduces the necessary preliminaries on neural ODE reachability and the CTMM reachability method. Section~\ref{sect:verif_arch} describes TNODEV verifier architecture, including its falsification module, specification interface, verification and refinement loop, three refinement heuristics, and the parallel scheduler. Section~\ref{sect:experiments} reports the experimental evaluation in three parts: a direct reachability comparison against NNV 2.0 and CORA, a verification evaluation on six safe-set inclusion benchmarks across two model classes, and a classification robustness evaluation on MNIST GNODE classifiers. Section~\ref{sect:conclusion} concludes the paper and discusses future directions. 

\section{Related work}
\label{sect:related_works}

The verification of neural ODE is a recent research direction, and the available choice of tools is much narrower than the current available tools for neural networks. Existing tools differ along three main axes: (i) whether they are designed for neural ODE specifically or adapted to them from a more general setting; (ii) the type of guarantees they provide, i.e., deterministic (sound) or probabilistic; (iii) and whether they use only reachability analysis as the main primitive or a complete verifier pipeline that includes a falsification module and input-set refinement. Table~\ref{tab:tool-comparison} summarizes the main tools and techniques currently available for the reachability analysis and verification of neural ODE.

The first work on neural ODE reachability is the \textit{Stochastic Lagrangian Reachability} (SLR) method of \citet{grunbacher2021slr}, which combines Lipschitz-based bounds on the flow map with a stochastic optimization procedure to produce confidence intervals on the reachable set. This reachability method was subsequently extended to longer time horizons in the \textit{GoTube} tool \citep{gruenbacher2022gotube}. Both methods provide probabilistic rather than deterministic guarantees, and the distinction is important for the user, 
as a deterministic reachability method returns a set that is mathematically proven to contain every trajectory of the system, while a probabilistic reachability method returns a set that contains every trajectory only with a user-specified confidence level, 
leaving a non-zero residual probability that a trajectory ends outside the reported set. Such stochastic approaches can be more appropriate for runtime monitoring and statistical assurance use cases but not for safety-critical systems, where any non-zero probability of missing an unsafe behavior is unacceptable.

The first deterministic tool dedicated to neural ODE is \textit{NNVODE} \citep{manzanas2022nnvode}, which introduces a \emph{General Neural ODE} (GNODE) class combining discrete neural network layers (fully connected, convolutional) with continuous ODE layers. \textit{NNVODE} is built as an extension of the neural network verification tool \textit{NNV} \citep{tran2020nnv}, it handles the discrete neural network layers with \textit{NNV}'s star set-based reachability analysis \citep{tran2019starset} and the continuous ODE block with the zonotope-based reachability backend of CORA \citep{althoff2015cora}. This work was later integrated into \textit{NNV~2.0} tool \citep{lopez2023nnv}, which is the closest existing tool to TNODEV in scope. However, the main difference from
TNODEV is that \textit{NNV~2.0} performs only one star-set reachability call per query without any input-set refinement, so the precision of its verdict is limited only to a single reachability analysis call. More recently, \textit{ModelVerification.jl} \citep{wei2025modelverification} a Julia-based toolbox building on the earlier \textit{NeuralVerification.jl} \citep{liu2021algorithms}, provides deterministic guarantees through reachability analysis and lists neural ODE among its supported network architectures alongside feed-forward networks, convolutional networks, and ResNet.

Another line of work uses general-purpose reachability tools that were not originally designed for neural ODE but can be adapted to them. \textit{CORA} \citep{althoff2015cora} is a widely-used reachability library for linear, nonlinear, and hybrid dynamical systems based on zonotopes and their generalizations (matrix zonotopes, polynomial zonotopes). Its \texttt{nonlinearSys} backend can be applied to a neural ODE by treating it as a generic nonlinear vector field. Along these lines, \citet{liang2024topological} combine CORA's zonotope-based reachability with a set-boundary analysis that exploits the homeomorphism property of invertible neural networks (including neural ODE) to propagate only the boundary of the input set rather than the entire set. The most recent addition to the literature on neural ODE reachability analysis is \citet{sayed2025mixed}, relying on the continuous-time mixed monotonicity (CTMM) method from \textit{TIRA} \citep{meyer2019tira}, declined in three versions: computing the final reachable set in a single integration step; incrementally with intermediate time steps; and combined with a boundary-based approach. However, similar to CORA, it only contains a reachability step. TNODEV uses the CTMM method as its default reachability backend and wraps it as a full neural ODE verifier. 

A different approach to neural ODE verification proceeds indirectly, by establishing formal relationships between neural networks and neural ODE. \citet{sayed2025bridging} establish a formal error bound between a neural ODE and the ResNet obtained by treating one Euler step of the ODE as a single residual block, so that any specification certified on one of the two models can be transferred to the other. This approach is complementary to the direct reachability analysis tools for neural ODE discussed above, as it allows for re-using any of the available neural network verification tools.

\begin{table}[htbp]
\caption{Tools and techniques that handle neural ODE reachability and/or verification. ``Designed for nODE'' indicates whether the tool was developed specifically for neural ODE or adapted from a more general setting. The columns ``Use for nODE'' and ``Iterative refinement for nODE'' describe each tool's behavior \emph{when applied to neural ODE} only, as some tools support broader functionality on standard neural networks.}
\label{tab:tool-comparison}
\begin{center}
\resizebox{\textwidth}{!}{%
\begin{tabular}{lccccc}
\multicolumn{1}{c}{\bf Tool/Technique} & \multicolumn{1}{c}{\bf Set Representation} & \multicolumn{1}{c}{\bf Guarantee} &
\multicolumn{1}{c}{\bf Designed} & \multicolumn{1}{c}{\bf Use for nODE} &
\multicolumn{1}{c}{\bf Iterative refinement} \\
\multicolumn{1}{c}{} & \multicolumn{1}{c}{} & \multicolumn{1}{c}{} & 
\multicolumn{1}{c}{\bf for nODE} & \multicolumn{1}{c}{} & 
\multicolumn{1}{c}{\bf for nODE}
\\ \hline \\
SLR \citep{grunbacher2021slr}                        & Lipschitz balls          & Probabilistic & \checkmark & Reachability  & no  \\
GoTube \citep{gruenbacher2022gotube}                 & Lipschitz balls          & Probabilistic & \checkmark & Reachability  & no  \\
NNVODE \citep{manzanas2022nnvode}                    & Star set + zonotopes     & Deterministic & \checkmark & Reachability  & no  \\
NNV~2.0 \citep{lopez2023nnv}                         & Star set + zonotopes     & Deterministic & adapted    & Verification  & no  \\
ModelVerification.jl \citep{wei2025modelverification} & Multiple                 & Deterministic & adapted    & Verification  & no  \\
CORA \citep{althoff2015cora}                         & Zonotopes                & Deterministic & adapted    & Reachability  & no  \\
\citet{liang2024topological} + CORA                  & Zonotopes                & Deterministic & adapted    & Reachability  & no  \\
TIRA + CTMM \citep{sayed2025mixed}                   & Intervals                & Deterministic & \checkmark & Reachability  & no  \\
Formal relationship of \citet{sayed2025bridging}     & Zonotopes                & Deterministic & \checkmark & Verification  & not applicable \\
\textbf{TNODEV (ours)}                               & Intervals                & Deterministic & \checkmark & Verification  & yes \\
\end{tabular}%
}
\end{center}
\end{table}
\vfill

\section{Preliminaries}
\label{sect:prelimin}

\subsection{Neural ODE Reachability}

\begin{definition}[Neural ODE Reachability]
\label{def:nODE_reachability}
Given an initial set $\Xin \subseteq \R^n$ and final time $t_f$ for the neural ODE, the set of reachable outputs is defined as:
\[
    \Reach(\Xin) \;=\; \{\, \Phi(t_f, u) \;\mid\; u \in \Xin \,\},
\]



where $\Phi : \R \times \R^n \to \R^n$ is the solution map of the initial value problem associated with equation (\ref{eq:node}), such that $\Phi(t_f, u)$ is the state reached at time $t_f$ from the initial condition $x(0) = u$.

\end{definition}

Since $\Reach(\Xin)$ cannot be computed exactly, we rely on an over-approximation $\Omega(\Xin)$ such that $\Reach(\Xin) \subseteq \Omega(\Xin)$. 

Definition~\ref{def:nODE_reachability} characterizes the reachable set at the single time instant $t_f$. For the construction of the CTMM decomposition function in Section~\ref{subsec:ctmm}, we need to initially obtain some bounds on the union of all the reachable sets over the entire time interval $[0, t_f]$, which we refer to as the \emph{reachable tube}.

\begin{definition}[Neural ODE Reachable Tube]
\label{def:nODE_reach_tube}
Given an initial set $\Xin \subseteq \R^n$ and a final time $t_f$ for the neural ODE, the reachable tube is defined as:
\[
    \Reach^{\text{tube}}(\Xin) \;=\; \bigcup_{t \in [0, t_f]}
        \{\, \Phi(t, u) \mid\; u \in \Xin \,\}.
\]
\end{definition}

This reachable tube contains every possible reachable state of the system starting from $\Xin$ over the time range $[0, t_f]$, and therefore provides a domain over which the Jacobian of $f$ can be bounded. 

\subsection{Reachability via Continuous-Time Mixed Monotonicity}
\label{subsec:ctmm}

The reachability component at the core of TNODEV relies on CTMM \citep{meyer2021interval,coogan2020mixed}, a property of dynamical systems that allows the original system to be embedded into a higher-dimensional monotone system on which interval bounds can be propagated efficiently.

Given the neural ODE equation (\ref{eq:node}) and an interval initial set $\Xin = [\underline{x}_0, \overline{x}_0] \subseteq \R^n$, the procedure can be carried out in three steps.

\paragraph{Step 1: Reachable tube and Jacobian bounding.} Compute an interval over-approximation $\Omega^{\text{tube}}(\Xin) \supseteq \Reach^{\text{tube}}(\Xin)$ of the reachable tube (Definition~\ref{def:nODE_reach_tube}) using an external reachability tool for neural ODE, e.g. CORA~\citep{althoff2015cora} in our case. This reachable tube is computed \emph{only once} on the full input set $\mathcal{X}_{in}$, and not reevaluated for future reachability computations on subsets of $\mathcal{X}_{in}$. From this over-approximation, derive interval bounds $[\underline{J}_x, \overline{J}_x]$ on the Jacobian $J_x(x) = \partial f / \partial x$ for all $x \in \Omega^{\text{tube}}(\Xin)$, using interval arithmetic applied to the Jacobian of the neural network. Lipschitz continuity of $f$ guarantees that these bounds exist, and they are the only condition required for the next step.

\paragraph{Step 2: Decomposition function.} Construct a decomposition function $g : \R^n \times \R^n \to \R^n$ satisfying $g(x, x) = f(x)$ for all $x$, such that $g$ is increasing in its first argument (off-diagonally) and decreasing in its second argument. This is achieved by a sign-stabilization procedure that shifts each off-diagonal entry of the Jacobian into a constant-sign half-plane using a shifting matrix $L_x$ derived from $[\underline{J}_x, \overline{J}_x]$ \citep{sayed2025mixed}.

\paragraph{Step 3: Embedded simulation.} Embed the neural ODE into the $2n$-dimensional monotone system
\begin{equation}
\label{eq:embedded}
\begin{bmatrix} \dot{x} \\ \dot{\hat{x}} \end{bmatrix}
\;=\;
\begin{bmatrix} g(x, \hat{x}) \\ g(\hat{x}, x) \end{bmatrix},
\qquad
\begin{bmatrix} x(0) \\ \hat{x}(0) \end{bmatrix}
\;=\;
\begin{bmatrix} \underline{x}_0 \\ \overline{x}_0 \end{bmatrix},
\end{equation}
and solve it once with a numerical ODE solver from $t = 0$ to
$t = t_f$. The trajectories $x(t)$ and $\hat{x}(t)$ track the
lower and upper bounds of the interval over-approximation
respectively, so that
$\Omega(\Xin) = [x(t_f),\, \hat{x}(t_f)] \supseteq \Reach(\Xin)$.

The key feature of this reachability step for the verifier described in Section~\ref{sect:verif_arch} lies in its low computational cost, as a single CTMM call requires one Jacobian bounding, one $L_x$ assembly, and one $2n$-dimensional ODE integration.

Soundness of the over-approximation produced by equation \ref{eq:embedded} is guaranteed in theory, i.e. for the exact solution of the embedded system, but the implementation integrates equation~\ref{eq:embedded} with a non-validated \texttt{ode45} solver, which controls an estimate of its own error rather than returning a rigorous over-approximation. The soundness of the computed over-approximation is therefore conditional on the accuracy of this numerical integration. Validated integration methods, which return guaranteed bounds on the integrated trajectory, i.e. an interval proven to contain the exact solution of the integrated ODE, do exist \citep{sandretto2015validated}, but they are substantially more expensive than a standard solver, and wrapping every reachability call of the refinement loop in validated integration would forfeit the computational benefit of CTMM. Appendix~\ref{apd:soundness_scope} states the resulting scope of the soundness guarantee, and Appendix~\ref{apd:next_version} discusses validated integration as a planned future extension.

\section{TNODEV Verifier Architecture}
\label{sect:verif_arch}


TNODEV is a formal verifier for neural ODE systems, with a modular architecture, and it can be easily extended to encompass other reachability analysis methods, model classes, and specification types.

\begin{definition}[Safety specification]
\label{def:spec}
A safety specification is a pair $(\Xin, \mathcal{P})$, where $\Xin \subseteq \R^n$ is an initial input set and $\mathcal{P} : \R^n \to \{\textsf{true}, \textsf{false}\}$ is a property on the state at time $t_f$. The specification is satisfied if $\mathcal{P}\bigl(\Phi(t_f, u)\bigr) = \textsf{true}$ for every $u \in \Xin$. TNODEV currently supports two types of properties, both reducing to safe-set inclusion $x \in \mathcal{X}_s$ for an interval-defined or implicitly-defined safe set $\mathcal{X}_s$:

\begin{itemize}
\item \emph{Interval safe set}: $\mathcal{P}(x) \;\Leftrightarrow\; x \in \mathcal{X}_s$, where $\mathcal{X}_s = [\underline{s}, \overline{s}] \subseteq \R^n$ is an axis-aligned interval given explicitly as part of the specification.
\item \emph{Classification robustness safe set}: $\mathcal{P}(x) \;\Leftrightarrow\; \argmax_j x_j = y^\star$ for a fixed target label $y^\star$, equivalently $x \in \mathcal{X}_s$ with $\mathcal{X}_s = \bigcap_{j \ne y^\star} \{x \in \R^{n_c} : x_{y^\star} \geq x_j\}$ the intersection of half-spaces induced by $y^\star$, and $n_c$ denotes the number of classes of a classification specification.
\end{itemize}
\end{definition}


Any property $\mathcal{P}$ that admits a sound check on the reachable set over-approximation $\Omega(\mathcal{X})$ can be plugged into the verification loop without modifying the rest of the pipeline. Extending the interface to other property types (e.g., reach-avoid specifications, or polyhedral sets defined by linear half-space constraints rather than axis-aligned intervals) requires only implementing the appropriate check on $\Omega(\mathcal{X})$. 
On the model side, TNODEV supports three classes of neural ODE models: a pure neural ODE, a general neural ODE that combines an ODE block with discrete neural network layers~\citep{manzanas2022nnvode}, and a hybrid neural ODE control system in which the ODE plant is wrapped by a neural network controller. 

Given a specification, the verifier returns one of three verdicts: \texttt{SAFE} (the specification is formally verified), \texttt{FALSIFIED} (a concrete counter-example was found), or \texttt{UNKNOWN} (no decision was reached within the maximum number of refinement iterations $K$ or the wall-clock timeout $T_{\max}$). The soundness of the methods used within TNODEV ensures that the \texttt{SAFE} and \texttt{FALSIFIED} outcomes respectively guarantee the satisfaction and falsification of the specification by the considered neural ODE.
However, without infinite computational power, TNODEV cannot be considered as complete, which is why the algorithm may still return \texttt{UNKNOWN} when no other outcome was reached within the allocated computational limits.

The architecture comprises four components: a falsification module, a specification interface, an iterative verification loop that combines reachability and refinement, and a parallel scheduler. Their interaction is summarized in Algorithm~\ref{alg:tnodev} and illustrated in Figure~\ref{fig:architecture}.

\begin{figure}[htbp]
\centering
\includegraphics[width=\textwidth]{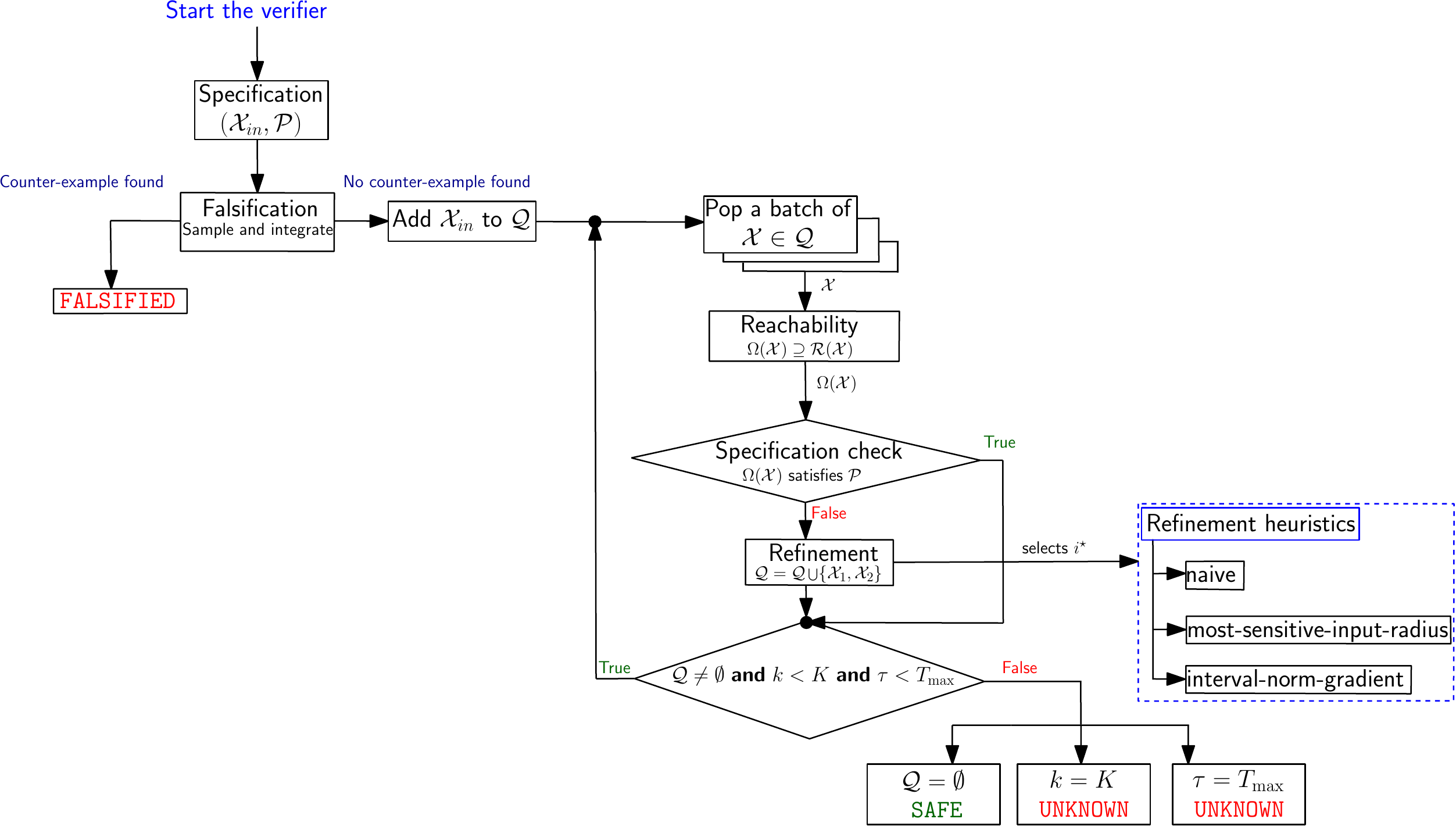}
\caption{TNODEV verifier architecture. The verifier starts with a falsification check on the full initial set $\Xin$. If a counter-example is found, it returns \texttt{FALSIFIED}; otherwise it enters the verification and refinement loop, where the reachability block computes the over-approximation $\Omega(\mathcal{X})$ for each cell, the specification check verifies whether $\Omega(\mathcal{X})$ satisfies the property $\mathcal{P}$, and the refinement block splits inconclusive cells using one among three heuristics. The loop terminates with \texttt{SAFE} when the queue $\mathcal{Q}$ is empty, or \texttt{UNKNOWN} if the refinement iteration limit $K$ or the wall-clock timeout $T_{\max}$ is reached.
}
\label{fig:architecture}
\end{figure}

\subsection{Falsification}
\label{subsec:falsification}

Before any reachability call, TNODEV runs a lightweight falsification pass on the entire initial set $\Xin$. The pass draws $1000$ points uniformly at random from $\Xin$ and integrates each of them forward to time $t_f$ along the neural ODE dynamics \eqref{eq:node}. If any sample trajectory ends in a state where $\mathcal{P}$ is violated, the verifier terminates with a \texttt{FALSIFIED} verdict and returns the violating trajectory as a concrete counter-example, as illustrated in Figure~\ref{fig:verifier_results}(a). Otherwise, the verifier proceeds to the verification loop. Placing falsification before reachability avoids spending refinement iterations on input sets that contain obvious violations, and a \texttt{FALSIFIED} verdict is conclusive since it exhibits an explicit counter-example. Placing the falsification pass before the verification and refinement loop, rather than inside it is a design choice for two reasons. First, it offers a cheap chance of concluding \texttt{FALSIFIED} on $\Xin$, before any reachability computation is exercised. Second, it keeps every refinement iteration as cheap as a single embedded integration, i.e. the objective for which the CTMM backend was chosen, since repeating the sampling pass on every processed cell would add many trajectory simulations to each iteration.

\subsection{Specification Interface}
\label{subsec:spec}

For any cell $\mathcal{X} \subseteq \Xin$ in the refinement queue, the reachability block computes an over-approximation $\Omega(\mathcal{X}) \supseteq \Reach(\mathcal{X})$. The specification interface checks whether $\Omega(\mathcal{X})$ satisfies the property $\mathcal{P}$, with the implementation depending on the property type.

For an interval safe-set inclusion $\mathcal{X}_s = [\underline{s}, \overline{s}]$, the over-approximation is the interval $\Omega(\mathcal{X}) = [\underline{x}, \overline{x}]$ returned by the reachability block, and the check $\Omega(\mathcal{X}) \subseteq \mathcal{X}_s \;\Longleftrightarrow\; \underline{x} \ge \underline{s} \;\;\text{and}\;\; \overline{x} \le \overline{s}$ reduces to component-wise interval comparisons.

For classification robustness, the property is encoded as a set of half-spaces $H_j = \{x \in \R^{n_c} : x_{y^\star} - x_j \le 0\}$ for each $j \ne y^\star$ (each $H_j$ is the region where competing class $j$ ties or beats the target $y^\star$), and the cell is verified if $\Omega(\mathcal{X}) \cap H_j = \emptyset$ for all $j$.

If the check passes, the cell is verified as \texttt{SAFE}. Otherwise, the verdict on this cell is inconclusive, as the violation may correspond to either an unsafe trajectory, or a loose over-approximation. The cell is then forwarded to the refinement step.

\subsection{Verification \& Refinement Loop}
\label{subsec:verif_loop}

If no counter-example is found during the falsification phase, the verifier enters the verification and refinement loop. The loop maintains a queue of input cells $\mathcal{Q}$ initialized to $\{\Xin\}$ and iterates until all cells have been verified, or the refinement iteration limit $K$ is exhausted, or the wall-clock timeout $T_{\max}$ is reached. At each iteration, a batch of cells $\mathcal{B}$ is popped from $\mathcal{Q}$. The batch size $b$ is a fixed parameter, chosen as a small multiple of the number of available parallel workers so that each iteration dispatches several independent cells to each worker of the parallelization pool (Section~\ref{subsec:parallel}). When fewer than $b$ cells remain, the whole queue forms the batch, i.e., $|\mathcal{B}| = \min(|\mathcal{Q}|, b)$. Each cell $\mathcal{X}$ in the batch is processed by a reachability call to obtain $\Omega(\mathcal{X})$, followed by the specification check of Section~\ref{subsec:spec}. Cells that pass the check are marked \texttt{SAFE}, and cells that do not pass are divided into two sub-cells $\mathcal{X}_1, \mathcal{X}_2$ by the refinement step and added to $\mathcal{Q}$. The loop terminates with a \texttt{SAFE} verdict when $\mathcal{Q}$ is empty (i.e., all cells are verified), or with an \texttt{UNKNOWN} verdict if $K$ or $T_{\max}$ are reached while some cells are still in the queue.

\begin{algorithm}[t]
\caption{TNODEV verification loop}
\label{alg:tnodev}
\begin{algorithmic}[1]
\Statex \textbf{Input:} neural ODE $f$, specification
$(\Xin, \mathcal{P})$, final time $t_f$, maximum number of refinement iterations $K$, wall-clock timeout $T_{\max}$, batch size $b$
\State Sample trajectories from $\Xin$ (corners, centroid, random)\Comment{\textbf{Phase 1: Falsification}}
\State Integrate each sample forward to $t_f$
\If{any sample violates $\mathcal{P}$}
    \State \Return \texttt{FALSIFIED}, violating trajectory
\EndIf
\State $\mathcal{Q} \gets \{\Xin\}$ \Comment{\textbf{Phase 2: Verification and Refinement}}
\State $k \gets 0$ and start timer $\tau$
\While{$\mathcal{Q} \neq \emptyset$ \textbf{and} $k < K$ \textbf{and} $\tau < T_{\max}$}
    \State pop a batch $\mathcal{B} \subseteq \mathcal{Q}$ with $|\mathcal{B}| = \min(|\mathcal{Q}|, b)$
    \ForAll{$\mathcal{X} \in \mathcal{B}$ \textbf{in parallel}}
        \State $\Omega(\mathcal{X}) \gets
            \textsc{Reach}(f, \mathcal{X}, t_f)$
        \If{$\Omega(\mathcal{X})$ satisfies $\mathcal{P}$}
            \State mark $\mathcal{X}$ as \texttt{SAFE}
        \Else
            \State $\mathcal{X}_1, \mathcal{X}_2 \gets
                \textsc{Refine}(\mathcal{X})$
            \State push $\mathcal{X}_1, \mathcal{X}_2$
                onto $\mathcal{Q}$
        \EndIf
    \EndFor
    \State $k \gets k + 1$
\EndWhile
\If{$\mathcal{Q} = \emptyset$}
    \State \Return \texttt{SAFE}
\Else
    \State \Return \texttt{UNKNOWN}
\EndIf
\end{algorithmic}
\end{algorithm}


\paragraph{Refinement.} When a cell $\mathcal{X}$ with center $x$ and half width $r$ cannot be verified, the refinement block divides it along a single input dimension $i^\star$ chosen by a heuristic, producing two sub-cells of equal half width along $i^\star$. We implement and compare three heuristics, illustrated on the Spiral~2D benchmark in Figure~\ref{fig:verifier_results}(b) and Figure~\ref{fig:Nonlinear_Spiral_results}.

\smallskip
\noindent\emph{Naive.} Split the dimension with the largest interval radius: $i^\star = \arg\max_i r_i$. This is the standard baseline used by most verification tools.

\smallskip
\noindent\emph{Most-sensitive-input-radius (MSIR).} For each input dimension $i$, compute the product of the cell's radius along $i$ and a worst case Jacobian magnitude in column $i$:
\begin{equation*}
\label{eq:msir}
s_i \;=\; r_i \cdot \max_{k}\,
        \max\!\bigl(|\underline{J}_{ki}|,\; |\overline{J}_{ki}|\bigr),
\end{equation*}
where $[\underline{J}, \overline{J}]$ are interval Jacobian bounds of $f$ over the cell. The split dimension is $i^\star = \arg\max_i s_i$. The Jacobian bounds are obtained through interval propagation on the symbolic Jacobian of the network, so MSIR requires no additional reachability calls.

\smallskip
\noindent\emph{Interval-norm-gradient (ING).} Estimate the partial derivatives of the over-approximation width $w(\mathcal{X}) = \sum_k (\overline{x}_k - \underline{x}_k)$ with respect to the cell's center by finite differences. For each input dimension $j$, perturb the center by a small offset $\delta$ along $j$, keep the radius fixed, re-run the reachability call to obtain $w(\mathcal{X}_{j,\delta})$, and let
$g_j \approx \partial w / \partial x_j$ denote the resulting finite-difference estimate:
\begin{equation*}
\label{eq:ing}
g_j \;\approx\;
\frac{w(\mathcal{X}_{j,\delta}) - w(\mathcal{X})}{\delta}.
\end{equation*}

The split dimension is $i^\star = \arg\max_j |g_j|$. ING requires additional reachability calls per refinement decision, i.e., one per input dimension.

\noindent 
The MSIR and ING heuristics are likewise adaptations of the input-splitting heuristics of CORA's neural network verification module \citep{ladner2023automatic}, with their network-specific ingredients replaced by quantities available in the continuous-time interval setting. Since the relative merits of such heuristics do not transfer from the discrete network case, their comparison on neural ODE and the resulting recommendation are part of the paper's contribution (Appendix~\ref{app:refinement_cost}).
The three heuristics span a cost spectrum: naive uses no extra computation, MSIR adds only Jacobian-bounding cost (negligible relatively to a reachability call), and ING adds one additional reachability call per input dimension of $\Xin \subseteq \R^n$. In practice, MSIR offers the best overall trade-off, as it reduces the number of refinement iterations through informed splitting decisions without incurring the per-decision overhead of ING (see Table~\ref{tab:results} and the per-benchmark analysis in Appendix~\ref{app:refinement_cost} where MSIR achieves the lowest verification time on the FPA benchmark and the largest iteration reduction on the higher-dimensional Cartpole benchmark). This is why MSIR is proposed as the default heuristic for TNODEV, an empirical choice related to the scope of the evaluated benchmarks rather than a general ranking of the three heuristics (Appendix~\ref{app:refinement_cost}).

The verification and refinement loop of TNODEV is an instance of branch and bound (BaB) \citep{bunel2020branch}, the workflow underlying modern neural network verifiers such as the $\alpha$-$\beta$-CROWN \citep{wang2021beta}. However, transferring it to neural ODE changes it in three way. Branching is restricted to the input set, since a neural ODE offers no finite stack of piecewise-linear activations to case-split on. The bounding step is a $2n$-dimensional numerical integration rather than a linear relaxation evaluated in milliseconds, so affordable sub-problem counts drop by orders of magnitude and the choice of branching dimension matters far more here than in the neural network case. Strict \emph{completeness} is lost in a sense, since input branching yields an infinite tree where ReLU branching yields a finite one (Appendix~\ref{apd:completeness}). 

Similarly to many \emph{sound} neural network verifiers in the literature, \emph{asymptotic completeness} can be recovered with unlimited computational capabilities, i.e.\ resolving any inconclusive verdicts \texttt{UNKNOWN} if the iterative refinement is left to run for a sufficiently long time. In the case of TNODEV, this asymptotic completeness is due to the defining property $g(x, x) = f(x)$ of the decomposition function. Indeed, when a cell of initial conditions shrinks to a single point $u \in \Xin$, the embedded system \eqref{eq:embedded} is initialized with $x(0) = \hat{x}(0) = u$ and collapses to two copies of the \emph{exact} trajectory of \eqref{eq:node} from $u$ independent of the Jacobian bounds used to construct $g$. The over-approximation therefore converges to the exact reachable point as cells shrink, so any specification satisfied by the neural ODE with some positive margin can be certified by our tool after finitely many refinements. Appendix~\ref{apd:completeness} discusses this in further details.


\begin{figure}[t] 
\centering
\begin{tabular}{@{}c@{\hspace{0.5em}}c@{}}
\includegraphics[width=0.457\textwidth]{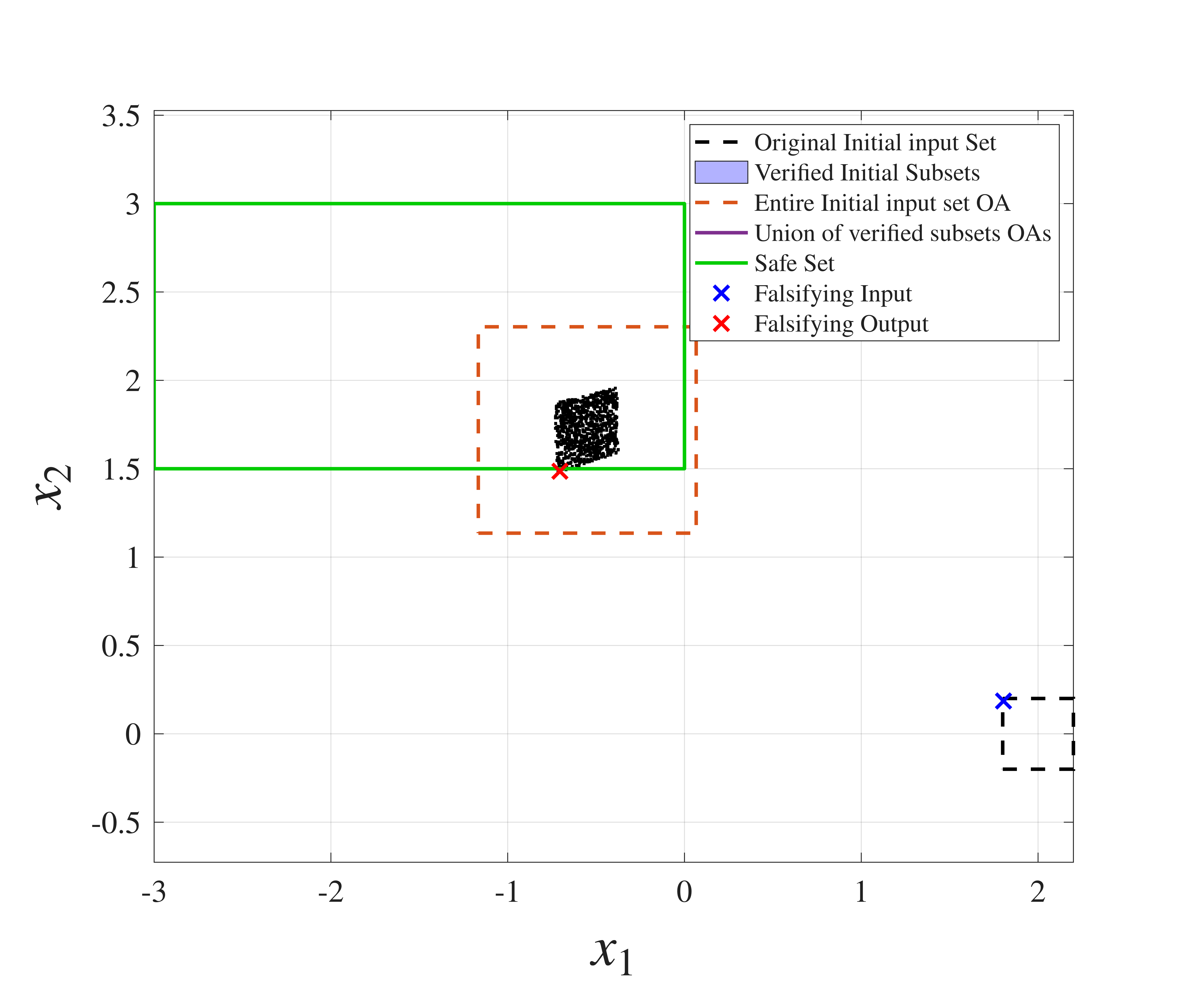} &
\includegraphics[width=0.48\textwidth]{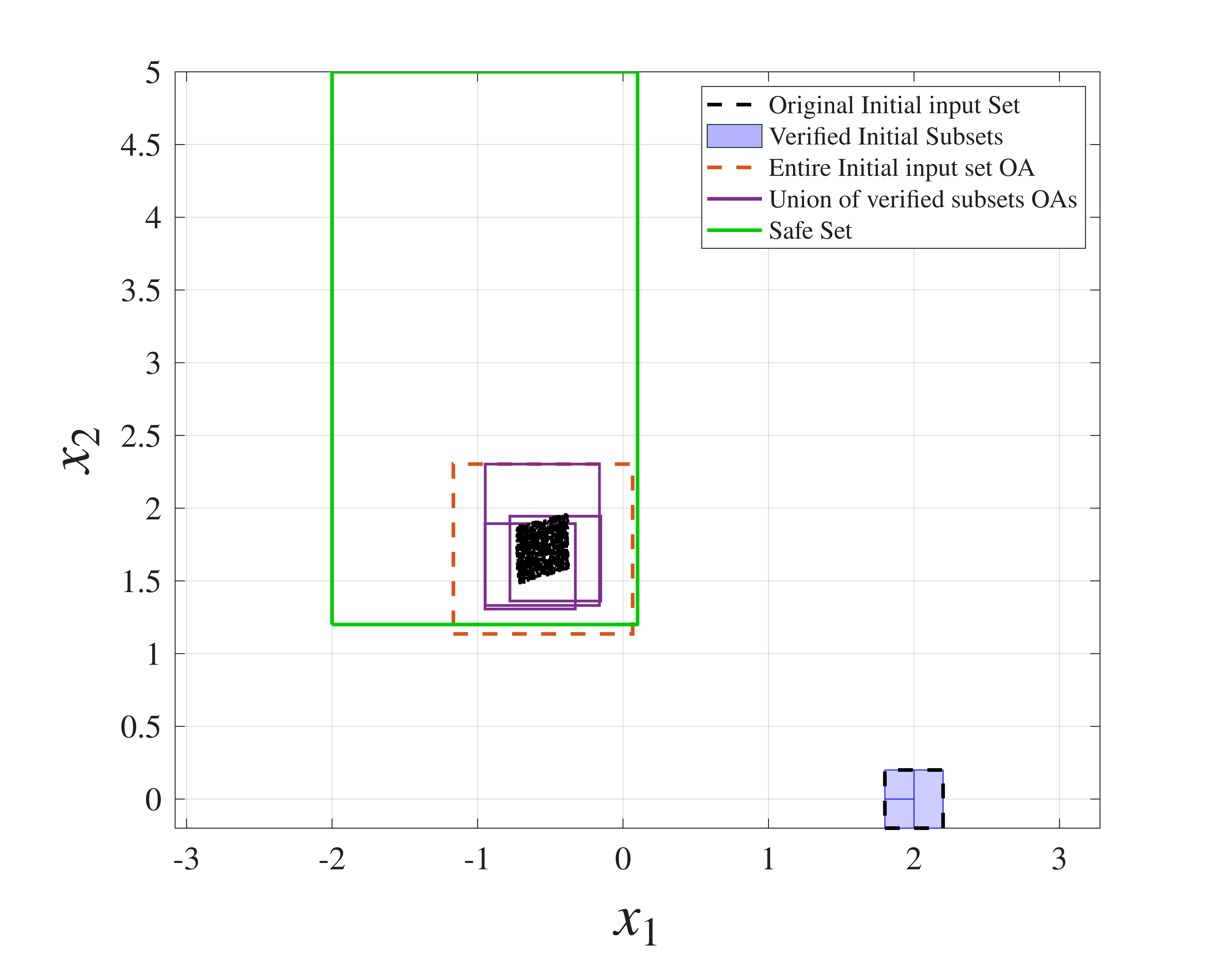} \\
(a) Falsification & (b) MSIR refinement
\end{tabular}
\caption{TNODEV verification results on the Spiral~2D benchmark. In both sub figures, the dashed black rectangle is the original initial set $\Xin$, the dashed \textcolor{red}{red} rectangle is the over-approximation $\Omega(\Xin)$ from a single reachability call on the unsplit initial set, the \textcolor{Green}{green} rectangle is the safe set $\mathcal{X}_s$, the \textcolor{violet}{violet} rectangles represent the verified subsets, and the black dots are sampled successor trajectories.}
\label{fig:verifier_results}
\end{figure}

\subsection{Parallelization}
\label{subsec:parallel}


Because each cell in the refinement queue $\mathcal{Q}$ can be processed independently, the inner loop of Algorithm~\ref{alg:tnodev} (lines~10--18) is run in parallel. At each iteration, TNODEV pops a batch of $|\mathcal{B}| = \min(|\mathcal{Q}|, b)$ cells and distributes them across CPU workers, so that each worker runs the reachability and specification check on its assigned cell and either marks it \texttt{SAFE} or produces the two refined sub-cells. The refined sub-cells are pushed back onto $\mathcal{Q}$ at the end of the iteration, and the next batch is popped only once the whole batch completes. Workers within a batch operate on disjoint cells and need no synchronization among themselves.

The batch size is $b = \lambda W$, where $W$ is the number of workers of the MATLAB parallel pool and $\lambda = 5$ is set so that each worker receives several cells per batch rather than exactly one. On the evaluation machine (Appendix~\ref{apd:System_description}) the pool opens $W = 16$ workers, giving $b = 80$ for every benchmark of Section~\ref{sect:experiments}. Giving each worker several cells matters because per-cell reachability times are unequal, so with $\lambda = 1$ most workers would sit idle at the batch barrier while the slowest cell finishes.
In a controlled comparison, where every pool size processes the same $2037$ cells (for the case of Cartpole benchmark with naive heuristic), the speed-up over one worker grows from $1.7\times$ on $2$ workers to $4.5\times$ on $16$, as mentioned in Table~\ref{tab:parallel_sweep} of Appendix~\ref{apd:parallel_scaling}. Table 10, then compares the serial and the parallel configurations of the toolbox and the speed-up brought by parallelization.

\section{Experiments and Results}
\label{sect:experiments}

\subsection{Reachability Analysis Comparison}
\label{subsec:reachability}

To the best of our knowledge, TNODEV is the first formal verifier for neural ODE to combine reachability, falsification, and iterative refinement in a single pipeline. As a consequence, no other existing tool exposes an equivalent end-to-end workflow to compare against, and the range of head-to-head verification comparisons is consequently limited.
To compensate for that, we start the experimental evaluation with a direct comparison of TNODEV core reachability component against the two closest neural ODE reachability tools in the literature: NNV 2.0 \citep{lopez2023nnv} and CORA \citep{althoff2015cora}. The comparison is carried out on the Spiral nonlinear and FPA benchmarks (defined in Appendix~\ref{apd:System_description}), where all three tools compute a single reachability call. We measure the wall-clock time of one reachability call, and the size of the resulting over-approximation reported as a single dimension-comparable metric, the geometric mean width $\mu(\Omega) = \sqrt[n]{\mathrm{Vol}(\Omega)}$. However, $\mu(\Omega)$ is not computed from the same geometric object for the three tools, as it is the interval volume for TNODEV, the exact zonotope volume for CORA, and, for NNV 2.0 the volume of the tightest box enclosing the Star set, since exact polytope volume computation is intractable in these dimensions \citep{dyer1988complexity,khachiyan1989problem}, more details are available in Appendix \ref{apd:reachability_comparison}.

The results highlight a consistent qualitative trade-off across the three tools. On the Spiral nonlinear benchmark, TNODEV is $14\times$ faster than NNV 2.0 and $83\times$ faster than CORA, as reported in Table \ref{tab:reach_comp_spiral_non}. On FPA, TNODEV is approximately $7\times$ faster than NNV 2.0 and $39\times$ faster than CORA, as reported in Table \ref{tab:reach_comp_fpa}. However, this gain in terms of computation time comes at the cost of the tightness of the over-approximation. Namely, on Spiral nonlinear, NNV 2.0 Star set and CORA  zonotope produce over-approximations respectively $4\times$ and $4.3\times$ tighter than TNODEV interval over-approximation. On FPA, the results are more mixed, as TNODEV interval over-approximation is approximately $1.3\times$ tighter than NNV 2.0 Star set, while CORA's zonotope remains the tightest at $1.4\times$ tighter than TNODEV.

This speed-tightness trade-off motivates the verifier architecture of Section \ref{sect:verif_arch}, as instead of relying on a single reachability call to be tight, TNODEV wraps a fast reachability backend in the verification and refinement loop that splits the input set and re-compute the reachability on each sub-cell until the over-approximation becomes tight enough for the property to be certified, or the refinement iteration limit is reached. 
The remainder of the section evaluates this loop in two regimes. Section~\ref{subsec:safe_set_benchmarks} exercises the loop on six safe-set inclusion benchmarks across pure and hybrid neural ODE model classes. Section~\ref{subsec:exp_classification} then evaluates the scalability of TNODEV on the high-dimensional regime of GNODE classifiers, and provides a direct verification comparison against NNV 2.0, which can run on these benchmarks but offers only a single reachability call without iterative refinement.

\subsection{Interval safe set inclusion benchmarks} 

\label{subsec:safe_set_benchmarks}

We evaluate TNODEV on four pure neural ODE benchmarks (Spiral~2D linear and nonlinear, FPA~5D, Cartpole~12D) and two hybrid neural ODE control system benchmarks (ACC~8D linear and nonlinear plant variants) as described in Appendix~\ref{apd:System_description}. Each benchmark is run with the three refinement heuristics introduced in Section~\ref{subsec:verif_loop}. We set $K=5000$ as the maximum number of refinement iterations and $T_{\max}=120$ minutes as the wall-clock timeout, for all benchmarks. Table~\ref{tab:results} reports the results for each benchmark. Out of the eighteen benchmark-heuristic combinations, eleven terminate with a \texttt{SAFE} verdict, three terminate with a \texttt{FALSIFIED} verdict on Spiral nonlinear in less than a fifth of a second ($0.12$--$0.18$~s), and the other four terminating with \texttt{UNKNOWN} verdicts all come from the ING heuristic, which either exhausts the refinement iterations (Spiral linear, Cartpole) or reaches the wall-clock timeout (ACC variants) and ING converges only on FPA. The eleven \texttt{SAFE} verdicts demonstrate that TNODEV verifies the specified safety properties across input dimensions ranging from $2$ to $12$. The four \texttt{UNKNOWN} verdicts are analyzed in Appendix~\ref{app:refinement_cost}. All the benchmarks used are taken from the neural ODE literature, with one deliberate change on the ACC benchmark, as the initial velocities are widened to $v_{\text{lead}} \in [31, 33]$ and $v_{\text{ego}} \in [29, 31]$ compared to $[31, 33]$ and $[29, 31]$ as in the case of \citet{manzanas2022nnvode}, the reason for that is that the later is verified by a single reachability call and refinement was never triggered. The safe sets are introduced by us for Spiral, FPA and Cartpole benchmarks, because these three appear in the literature as reachability benchmarks only without a property to verify, while the ACC safe-distance property is that of \citet{tran2020nnv}. All specifications are given in Appendices~\ref{apd:Linear_Spiral_spec}--\ref{apd:ACC_spec}.

\begin{table}[htbp]
  \centering
  \caption{TNODEV verification results across all benchmarks. Each benchmark is run with three refinement heuristics. Iterations and verified subsets refer to the verification and refinement loop (Algorithm~\ref{alg:tnodev}). Time is the total verification time.}
  \label{tab:results}
  \small
  \begin{tabular}{lllrrlr}
    \toprule
    Benchmark & Dim.\ & Heuristic & Time & Iter.\ & Verdict & Verified subsets\\
    \midrule
    \multirow{3}{*}{Spiral linear} & \multirow{3}{*}{2} & Naive & 6.79 s & 31 & \texttt{SAFE} & 16\\
      & & MSIR & \textbf{6.50 s} & 29 & \texttt{SAFE} & 15\\
      & & ING & 13.16 min & 5000 & \texttt{UNKNOWN} & 242\\
    \midrule
    \multirow{3}{*}{Spiral nonlinear} & \multirow{3}{*}{2} & Naive & 0.18 s & 0 & \texttt{FALSIFIED} & 0\\
      & & MSIR & 0.16 s & 0 & \texttt{FALSIFIED} & 0\\
      & & ING & \textbf{0.12 s} & 0 & \texttt{FALSIFIED} & 0\\
    \midrule
    \multirow{3}{*}{FPA} & \multirow{3}{*}{5} & Naive & 53.78 s & 17 & \texttt{SAFE} & 9\\
      & & MSIR & \textbf{34.50 s} & 9 & \texttt{SAFE} & 5\\
      & & ING & 53.78 s & 3 & \texttt{SAFE} & 2\\
    \midrule
    \multirow{3}{*}{Cartpole} & \multirow{3}{*}{12} & Naive & 2.73 min & 2037 & \texttt{SAFE} & 1019\\
      & & MSIR & \textbf{56.43 s} & 63 & \texttt{SAFE} & 32\\
      & & ING & 54.03 min & 5000 & \texttt{UNKNOWN} & 0\\
    \midrule
    \multirow{3}{*}{ACC linear} & \multirow{3}{*}{8} & Naive & 4.78 min & 15 & \texttt{SAFE} & 8\\
      & & MSIR & \textbf{2.72 min} & 7 & \texttt{SAFE} & 4\\
      & & ING & 120 min & 1087 & \texttt{UNKNOWN} & 0\\
    \midrule
    \multirow{3}{*}{ACC nonlinear} & \multirow{3}{*}{8} & Naive & \textbf{20.03 min} & 41 & \texttt{SAFE} & 21\\
      & & MSIR & 29.61 min & 37 & \texttt{SAFE} & 19\\
      & & ING & 120 min & 607 & \texttt{UNKNOWN} & 0\\
    \bottomrule
  \end{tabular}
\end{table}

\paragraph{Pure neural ODE benchmarks.} On the small scale benchmarks (Spiral 2D and FPA 5D), the verification and refinement loop runs for multiple refinement iterations ranging from $3$ to $5000$, depending on the heuristic and the property $\mathcal{P}$ being verified. The number of verified subsets and the verification time vary between heuristics as illustrated in Figures \ref{fig:Linear_Spiral_results}--\ref{fig:FPA_ING_results} of Appendix \ref{apd:Detailed_results}, and we discuss this trade-off in Appendix~\ref{app:refinement_cost}. On the Spiral linear variant, the ING heuristic reaches the refinement limit without converging, producing $242$ verified subsets but ultimately returning an \texttt{UNKNOWN} verdict in both cases, due to repeated subdivision without sufficient progress; further discussion on that is in Appendix~\ref{app:refinement_cost}. The Spiral nonlinear variant terminates with a \texttt{FALSIFIED} verdict, produced by the falsification pass before the verification and refinement loop starts, which is why the iteration count is zero and the heuristics are never invoked, and the three runs differ only in the randomness of the drawn points. The returned counter-example $(1.8012,\, 0.1846) \in \Xin$ whose final state $\Phi(t_f, u) = (-0.7036,\, 1.4828)$ lies outside $\mathcal{X}_s$, as illustrated in Figure~\ref{fig:verifier_results}(a). The FPA benchmark is also used to measure at the verification level what refinement adds over a single reachability call. In Appendix~\ref{apd:FPA_spec}, one NNV 2.0 reachability call followed by the same specification check is inconclusive, and so is a single CTMM call of TNODEV on the entire input set, whereas TNODEV with the MSIR heuristic certifies the specification \texttt{SAFE} after $9$ refinement iterations. Refinement, rather than the tightness of any single reachability call, is therefore what verifies this property. Figure~\ref{fig:nnv_vs_tnodev} illustrates the two cases. On the 12D Cartpole benchmark, the verification and refinement loop capabilities are illustrated for the naive heuristic resulting in $2037$ refinement iterations and $1019$ verified subsets, and the MSIR heuristic reduces this amount to $63$ refinement iterations and $32$ verified subsets. The ING heuristic also fails to converge on this benchmark, reaching the refinement limit without verifying a single subset and returns \texttt{UNKNOWN}, reflecting the compounding cost of its extra reachability calls per refinement decision, which grows on higher dimensional input space. 

\paragraph{Hybrid neural ODE control system benchmark.} On the ACC 8D benchmark, both the linear and nonlinear plant variants are verified \texttt{SAFE} using the naive and MSIR heuristics, exercising the refinement loop in the hybrid closed-loop setting where a neural network controller wraps the neural ODE plant. On the linear variant, the verification achieved based on the MSIR heuristic takes $2.72$ min over $7$ iterations against $4.78$ min over $15$ iterations for the naive heuristic. On the nonlinear variant, the verification based on the naive heuristic takes $20.03$ min over $41$ iterations while it takes $29.61$ min over $37$ iterations for the MSIR heuristic. It is worth noting that the higher cost of the nonlinear variant over the linear one reflects the higher per-step reachability cost of the $\tanh$-activated plant. The ING heuristic fails to converge on both variants, reaching the wall-clock timeout without verifying any subset and returning \texttt{UNKNOWN}, consistent with the compounding cost of its $n$ extra reachability calls per refinement decision observed on the other higher-dimensional benchmarks (see the per-benchmark analysis in Appendix~\ref{app:refinement_cost}).


\subsection{Classification neural ODE benchmark}
\label{subsec:exp_classification}

We additionally evaluate TNODEV on a classification robustness benchmark built around a general neural ODE (GNODE)  classifier \citep{manzanas2022nnvode} for the MNIST handwritten-digit data set. Full architectural details, the hybrid reachability pipeline, and the bound-extraction modes used at the ODE-block boundary are given in Appendix~\ref{apd:gnode}. The verification problem on this benchmark is to certify $L_\infty$ adversarial robustness around individual test images, which we formalize as a specific instance of Definition~\ref{def:spec}.

\begin{definition}[$L_\infty$ adversarial robustness]
\label{def:robustness}
Let $f$ be a neural ODE classifier whose final state at time $t_f$ is the logit vector $\Phi(t_f, u) \in \mathbb{R}^{n_c}$ over $n_c$ classes of the task ($n_c = 10$ on MNIST), and let $(x_0, y^\star)$ be a clean input label pair for which $f$ is correct, i.e., $\argmax_j \Phi_j(t_f, x_0) = y^\star$. We say $f$ is robust at $(x_0, y^\star)$ if the safety specification $(\Xin, \mathcal{P})$ is satisfied, where
\[
\Xin = B_\infty(x_0, \epsilon) = \{ u : \|u - x_0\|_\infty \le \epsilon \}
\]
is the $L_\infty$ ball of radius $\epsilon$ around $x_0$, and $\mathcal{P}$ is the classification robustness property with target label $y^\star$:
\[
\mathcal{P}(x) \;\Leftrightarrow\; \argmax_j x_j = y^\star.
\]
\end{definition}

A classifier verified $\epsilon$-robust at $(x_0, y^\star)$ is guaranteed to predict $y^\star$ on all adversarially perturbed inputs $u \in B_\infty(x_0, \epsilon)$, formally certifying it against any $L_\infty$-bounded attack of magnitude up to $\epsilon$. 

We compare TNODEV with NNV 2.0 on the two convolutional GNODE classifiers CNODE\textsubscript{S} and CNODE\textsubscript{M} of \citet{manzanas2022nnvode}. TNODEV current version targets the verification of neural ODE, but does not yet provide reachability for the neural network layers that surround the ODE block in a GNODE architecture. We therefore construct a hybrid pipeline in Appendix \ref{apd:classification_nODE_Hybrid_reachability_pipeline} in which NNV 2.0 handles the pre-ODE and post-ODE layers, while TNODEV handles the linear ODE block via a closed-form matrix-exponential propagator using interval sets.
We state its construction formally and prove its soundness in Appendix~\ref{apd:classification_nODE_Hybrid_reachability_pipeline}, and discuss the empirical justification for the design choice there as well. We report TNODEV under two bound-extraction modes (LP-tight and fast) to convert into an interval the sets coming into the ODE block, and this is discussed in more detail in Appendix~\ref{apd:classification_nODE_star_to_interval}. We evaluate on the first 50 MNIST test images at $\epsilon = 0.5/255$, matching the sample size used in \citet{manzanas2022nnvode}, and report per-image and total computational times for both tools. Verification is carried out on this benchmark without any refinement, as the obstacle is the dimension of the set that refinement would have to split, i.e. the $784$-dimensional image at the classifier input. Each refinement step divides a single coordinate of a single cell, so reducing every radius of the input set only once already requires $2^{784}$ sub-cells. The number of cells that input-set refinement needs therefore grows exponentially with the input dimension, which is also what the refinement depth $d^\star$ in Proposition~\ref{prop:completeness} expresses, since $d^\star$ is proportional to $n$. Refinement is therefore an effective precision-recovery mechanism the moderate input dimensions, but at high dimensions this is not feasible, as each sub-cell would moreover repeat the full hybrid pipeline of Appendix~\ref{apd:classification_nODE_Hybrid_reachability_pipeline} per call.


Table \ref{tab:mnist-cnode} shows that NNV 2.0 verifies nearly all images on both networks (49/50 on CNODE\textsubscript{S}, and 50/50 on CNODE\textsubscript{M}). TNODEV verifies 29/50 (58\%) on CNODE\textsubscript{S} and 30/50 (60\%) on CNODE\textsubscript{M}, with the LP-tight and fast bound-extraction modes returning identical verdict counts on this benchmark, indicating that the interval over-approximation through the linear ODE is the binding constraint, rather than the predicate-constraint information discarded at the ODE-block boundary (Appendix~\ref{apd:classification_nODE_star_to_interval}).

\begin{table}[tbh]
\centering
\caption{Robustness verification of MNIST GNODE classifiers under a full-image $L_\infty$ perturbation of $\epsilon = 0.5/255$. $n$ denotes the dimension of the linear ODE state in each network (equal to the flattened pre-ODE feature dimension, refer to Table~\ref{tab:mnist-arch}).}
\label{tab:mnist-cnode}
\small
\begin{tabular}{lllrr}
\toprule
\textbf{Network} & \textbf{Method} & \textbf{Robust/50} & \textbf{Time/img.\ (s)} & \textbf{Total (s)} \\
\midrule
\multirow{3}{*}{CNODE\textsubscript{S} ($n=676$)}
 & NNV~2.0              & 49 & 10.83 & 541.6  \\
 & TNODEV (LP-tight)    & 29 & 22.97 & 1148.6 \\
 & TNODEV (fast)        & 29 &  2.20 & 109.9  \\
\midrule
\multirow{3}{*}{CNODE\textsubscript{M} ($n=1690$)}
 & NNV~2.0              & 50 & 166.4  & 8321.9 \\
 & TNODEV (LP-tight)    & 30 &  35.87 & 1793.5 \\
 & TNODEV (fast)        & 30 &   0.83 & 41.4   \\
\bottomrule
\end{tabular}
\end{table}

\paragraph{Where NNV 2.0 wins.}
The verification gap originates entirely from the linear ODE block, or more precisely from the fact that the true reachable set of the linear ODE block from the input interval is in general a polytope rather than an axis-aligned box. NNV 2.0 represents this reachable set exactly using a predicate-constrained Star set, while TNODEV represents it by the tightest axis-aligned box that contains it, which is a sound enclosure but discards the polytopal structure whenever the reachable set is not itself axis-aligned. The empirical widening observed through the linear ODE is roughly $7\times$ in aggregate width on CNODE\textsubscript{S} and $8\times$ on CNODE\textsubscript{M} at $\epsilon = 0.5/255$. The MNIST benchmark consists of a single linear ODE block, a setting in which Star set propagation is exact and interval based propagation is not, which is why TNODEV is at a structural disadvantage here. The verification rate shortfall is therefore the tightness side of the same speed-tightness trade-off observed in the reachability comparison of Appendix \ref{apd:reachability_comparison}, as TNODEV is fast but with wider over-approximations, and on a benchmark dominated by a single exactly-representable linear block, that wideness translates directly into fewer verified images (Appendix~\ref{apd:classification_nODE_verification_gap} expands on this comparison). The interval over-approximation through the linear ODE also imposes an empirical limitation on the perturbation magnitudes for which TNODEV produces useful verdicts on this benchmark, Appendix \ref{apd:classification_nODE_applicability} discusses this applicability scope in more detail.

\paragraph{Where TNODEV wins.}
TNODEV is between $5\times$ and $200\times$ faster than NNV 2.0 on a per-image basis, with the gap widening sharply as the ODE state dimension increases. On CNODE\textsubscript{M} the closed-form interval propagation requires only $0.83$ s per image against NNV 2.0's $166.4$ s, which corresponds to a $200\times$ speed up. On the smaller CNODE\textsubscript{S}, NNV 2.0's per-image cost drops to $10.83$ s, slightly faster than TNODEV LP-tight mode at $22.97$ s because the LP-based bound extraction at the ODE-input boundary becomes the dominant cost on small networks, but the fast mode bypasses this bottleneck and recovers a $5\times$ speedup at $2.20$ s per image. 

\section{Conclusion}
\label{sect:conclusion}

We present TNODEV, a formal verifier for neural ODEs. To the best of our knowledge, TNODEV is the first verifier for neural ODE to integrate refinement-based precision recovery into the workflow, addressing a structural gap in the current neural ODE verification literature where existing tools provide only a single reachability call without any iterative refinement.

The experimental results in Section~\ref{sect:experiments} feature both the strengths and the limitations of TNODEV current version. On the reachability comparison against NNV 2.0 and CORA, TNODEV is one to two orders of magnitude faster while producing a bit looser over-approximations, illustrating the speed and tightness trade-off intrinsic to interval-based propagation. On the safe-set inclusion benchmarks across pure neural ODE and hybrid neural ODE control system model classes, TNODEV verifies benchmarks ranging from $2$ to $12$ input dimensions, with the MSIR heuristic giving the best overall trade-off among the three heuristics on the evaluated benchmarks, which makes it the proposed default of the current version. On the high dimensional MNIST classification benchmark, the verification rate of TNODEV is lower than NNV 2.0 due to the structural disadvantage of interval propagation on linear ODE blocks, but the per-image verification time is between $5\times$ and $200\times$ lower.

Based on these experimental findings, TNODEV is the recommended tool when the dynamics are nonlinear, when the first-computed reachable set is inconclusive but close enough to the boundary of the safe set $\mathcal{X}_s$ so that its iterative refinement can reach a conclusion, and when the input dimension is moderate so that input-set refinement remains affordable. NNV 2.0 remains the suitable choice when the ODE block is linear, since its Star-set representation is exact.

Some of the limitations of the current TNODEV version identified in the previous sections suggest possible directions for future improvements of the next version of the toolbox, which we summarize here and detail in Appendix~\ref{apd:next_version}. On the reachability side, richer set representations (zonotopes or polynomial zonotopes) would remove the structural disadvantage of intervals on linear ODE blocks, and validated integration of the embedded system. On the refinement side, a fallback for the gradient-based ING heuristic would retain ING's positive splitting behavior while reducing its higher computational cost, and moving the falsification pass inside the loop may allow the verifier to catch counter-examples at any iteration of the refinement loop rather than only before it starts. Finally, implementing reachability for discrete neural network layers would remove the dependence on an external tool in the GNODE setting.




\subsubsection*{Acknowledgments}

This project has received funding from the European Union’s Horizon 2020 research and innovation programme under the Marie Sklodowska-Curie COFUND grant agreement no. 101034248.

\bibliography{main}
\bibliographystyle{tmlr}

\clearpage
\appendix
\newpage

{\LARGE\bf\sffamily \textbf{Appendix}} 

\textbf{Experimental Settings:} All the experiments\footnote{Code available in the following repository: \url{https://github.com/ab-sayed/TNODEV}} herein are run on MATLAB R2026a with the Continuous Reachability Analyzer (CORA) version 2025.1.1, the Toolbox for Interval Reachability Analysis (TIRA) version 2, and the Neural Network Verification Software Tool (NNV 2.0) on an AMD Ryzen™ 9 5950X CPU (16 cores, 32 threads) with 96 GB of RAM.

\section{System description}\label{apd:System_description}

This appendix describes the neural ODE benchmarks used in the experiments, organized by the three model classes introduced in Section~\ref{sect:verif_arch}. Spiral and FPA are pure neural ODE in which the entire right-hand side $f$ of \eqref{eq:node} is parameterized by a neural network. Cartpole is also presented as a single ODE, but its right-hand side combines a neural ODE controller with cartpole physics equations rather than being entirely a neural ODE. ACC is a hybrid neural ODE control system in which a neural ODE plant is controlled by a separate neural network controller in a discrete-time closed loop setting. MNIST is a general neural ODE (GNODE) in which an ODE block is embedded in a feed-forward classifier between discrete neural network layers. 

\paragraph{Origin of the benchmarks and specifications.} Every model used in this paper is taken from the neural ODE literature. Spiral (linear and nonlinear) from \citet{chen2018neural} and \citet{manzanas2022nnvode}, FPA from \citet{beer1995dynamics} as adopted by \citet{musau2018continuous}, Cartpole from \citet{gruenbacher2020lagrangian,gruenbacher2022gotube}, ACC from \citet{tran2020nnv} with the neural ODE plant of \citet{manzanas2022nnvode}, and the MNIST GNODE classifiers CNODE\textsubscript{S} and CNODE\textsubscript{M} from \citet{manzanas2022nnvode}. The initial sets follow the published reachability setups. The \emph{safe sets}, in contrast, are introduced by us for Spiral, FPA and Cartpole, because these three appear in the literature as reachability benchmarks only, i.e. they specify an initial set and a horizon, but no property to check, we state this in Appendices~\ref{apd:Linear_Spiral_spec}--\ref{apd:Cartpole_spec}. 
The two specifications that are not ours are the ACC whose safe-distance property comes from \citet{tran2020nnv}, and MNIST robustness whose $\epsilon$ and image set follow \citet{manzanas2022nnvode}. To the best of our knowledge no shared benchmark suite currently exists for neural ODE verification, so each tool, ours included, is evaluated on instances assembled by its authors. To help fill this gap, we are working on a public benchmark suite for neural ODE verification, spanning different control applications and pairing each model with an explicit safety specification, which we intend to release to the community to support comparable evaluations in this area.

\subsection{Pure neural ODE benchmarks}

\subsubsection{Spiral}\label{apd:Spiral_System_description}

The spiral system is a 2-dimensional neural ODE \citep{chen2018neural} whose trajectories form a spiral in the state space. We consider two variants from the literature, a \emph{nonlinear} variant from \citet{chen2018neural} and a \emph{linear} variant from \citet{manzanas2022nnvode}. The two variants share the same architecture of a 2-layer feed-forward network with a hidden layer of $10$ neurons, in which the input is in $\mathbb{R}^2$, the hidden layer is in $\mathbb{R}^{10}$, and the output is again in $\mathbb{R}^2$. The two variants differ only in the activation function applied to the hidden layer.

\subsubsection*{Nonlinear}\label{apd:Spiral_System_description_nonlinear}

In the nonlinear variant, the activation is $\tanh$, and the neural ODE has the following dynamics
\[
\dot{x} = f(x) = W_2 \tanh(W_1 x + b_1) + b_2,
\]
where $x \in \mathbb{R}^2$, $W_1 \in \mathbb{R}^{10 \times 2}$, $b_1 \in \mathbb{R}^{10}$, $W_2 \in \mathbb{R}^{2 \times 10}$, $b_2 \in \mathbb{R}^{2}$, and $\tanh(\cdot)$ is applied element-wise to the hidden vector $W_1 x + b_1 \in \mathbb{R}^{10}$. The exact values of the weight matrices and bias vectors are defined within the MATLAB function \emph{System\_description.m}.

\subsubsection*{Linear}\label{apd:Spiral_System_description_linear}

In the linear variant, there is no activation function, and the neural ODE has the following dynamics
\[
\dot{x} = f(x) = W_2 (W_1 x + b_1) + b_2,
\]
with $W_1, b_1, W_2, b_2$ of the same shapes as in the nonlinear variant above. Since the hidden layer applies no activation, the dynamics reduce to a 2-dimensional linear ODE $\dot{x} = A x + b$ with $A = W_2 W_1 \in \mathbb{R}^{2 \times 2}$ and $b = W_2 b_1 + b_2 \in \mathbb{R}^{2}$. The exact values of the weight matrices and bias vectors are defined within the MATLAB function \emph{System\_description.m}.

\subsubsection{FPA}\label{apd:FPA_System_description}

The Fixed-Point Attractor (FPA) system is a 5-dimensional neural ODE that admits a fixed-point equilibrium under appropriate conditions, originating from a continuous-time recurrent neural network model in \citet{beer1995dynamics} and adopted as a verification benchmark by \citet{musau2018continuous}. We consider here the same 5-dimensional formulation, with the following neural ODE dynamics
\[
\dot{x} = f(x) = \tau x + W \tanh(x),
\]
where $x \in \mathbb{R}^5$, $\tau = -10^{-6}$ is a small linear-leak coefficient, and $W \in \mathbb{R}^{5 \times 5}$ is a composite recurrent weight matrix
\[
W = \begin{pmatrix} 0_{2 \times 2} & A \\ 0_{3 \times 2} & B A \end{pmatrix},
\]
The matrices $A$ and $B$ together fully specify $W$, and the corresponding values are loaded directly from the MATLAB function \emph{System\_description.m}.

\subsubsection{Cartpole}\label{apd:Cartpole_System_description}

The Cartpole benchmark we use was introduced as a neural ODE reachability problem in \citet{gruenbacher2020lagrangian, gruenbacher2022gotube}. The cart is moving along a frictionless track and it must keep an attached pole upright by applying a horizontal force, with the controller realized as a continuous-time recurrent neural network (CT-RNN). The physical state collects the cart position $x_{\text{cart}}$, cart velocity $v_{\text{cart}}$, pole angle $\theta$, and pole angular velocity $\omega$, and the controller maintains $8$ recurrent hidden states $h_1, \ldots, h_8$ whose dynamics are integrated jointly with the physics, giving an overall $12$-dimensional neural ODE with state vector $x = (x_{\text{cart}}, v_{\text{cart}}, \theta, \omega, h_1, \ldots, h_8)^\top \in \mathbb{R}^{12}$.

The dynamics consists of two coupled blocks. The first is the standard cartpole physics, in which the four physical components of $\dot{x}$ are given by
\[
\dot{x}_{\text{cart}} = v_{\text{cart}}, \quad \dot{v}_{\text{cart}} = a_{\text{cart}}, \quad \dot{\theta} = \omega, \quad \dot{\omega} = a_{\text{pole}},
\]
with the cart and pole accelerations $a_{\text{cart}}$ and $a_{\text{pole}}$ obtained from the standard cartpole equations of motion as functions of $\theta$, $\omega$, and the applied force $F = F_{\text{mag}} \cdot u$, where $F_{\text{mag}} = 10~\text{N}$ is the actuator magnitude and $u \in [-1, 1]$ is the controller's scalar action defined below. 

The second block is the CT-RNN controller, in which the hidden-state vector $h = (h_1, \ldots, h_8)^\top$ evolves according to
\[
\dot{h} = -h + \tanh\!\bigl(W^{\text{rec}} h + W^{\text{in}} p + b\bigr),
\]
where $p = (x_{\text{cart}}, v_{\text{cart}}, \theta, \omega)^\top$ is the physical sub-state fed to the controller, $W^{\text{rec}} \in \mathbb{R}^{8 \times 8}$ is the recurrent weight matrix, $W^{\text{in}} \in \mathbb{R}^{8 \times 4}$ is the input weight matrix, $b \in \mathbb{R}^{8}$ is the bias vector, and $\tanh(\cdot)$ is applied element-wise. The self-decay term $-h$ causes the hidden state to relax exponentially toward zero in the absence of input, while the $\tanh$ term provides a bounded driving signal that mixes the recurrent and input contributions. The scalar action $u$ that drives the cartpole physics in the first block is then computed from the hidden state as
\[
u = \tanh\!\bigl(W^{\text{out}} h\bigr), \qquad W^{\text{out}} \in \mathbb{R}^{1 \times 8},
\]
which produces the bounded action $u \in [-1, 1]$ used in the force expression $F = F_{\text{mag}} \cdot u$. The exact values of the recurrent weights $W^{\text{rec}}$, input weights $W^{\text{in}}$, biases $b$, and readout weights $W^{\text{out}}$ are defined within the MATLAB function \emph{CartpoleCTRNN.m}.

\subsection{Hybrid neural ODE control system benchmarks}

\subsubsection{ACC}\label{apd:ACC_System_description}

The Adaptive Cruise Control (ACC) benchmark, originally introduced as a neural network controlled system reachability problem in \citet{tran2020nnv} and adapted to a neural ODE plant model in \citet{manzanas2022nnvode}, is a hybrid neural ODE control system in which an ego vehicle equipped with an ACC controller must maintain a safe distance from a lead vehicle. The system operates in two modes: a \emph{speed control} mode, in which the ego vehicle tracks a driver-specified set speed $v_{\text{set}} = 30~\text{m/s}$, and a \emph{spacing control} mode, in which it maintains a safe distance from the lead vehicle. The controller is a feed-forward neural network with $5$ hidden layers of $20$ ReLU neurons each and a linear output layer, with a control period of $0.1~\text{s}$. The verification scenario considers the two vehicles to be initially operating at a safe distance with the ego vehicle in speed control mode, after which the lead vehicle suddenly decelerates at $-2~\text{m/s}^2$. The safety specification requires that the relative distance $D_{\text{rel}} = x_{\text{lead}} - x_{\text{ego}}$ stays above the safe distance $D_{\text{safe}} = D_{\text{default}} + t_{\text{gap}} \cdot v_{\text{ego}}$ throughout a $5~\text{s}$ time horizon, with $D_{\text{default}} = 10~\text{m}$ and $t_{\text{gap}} = 1.4~\text{s}$.

We consider here the version with a neural ODE plant model from \citet{manzanas2022nnvode}, consisting of a 3rd-order $8$-dimensional ODE with state vector $x = (x_{\text{lead}}, v_{\text{lead}}, \gamma_{\text{lead}}, x_{\text{ego}}, v_{\text{ego}}, \gamma_{\text{ego}}, a_{\text{ego}}, a_{\text{lead}})^\top \in \mathbb{R}^{8}$, where $x_{\text{lead}}$, $v_{\text{lead}}$, and $\gamma_{\text{lead}}$ denote the lead vehicle position, velocity, and internal acceleration, $x_{\text{ego}}$, $v_{\text{ego}}$, and $\gamma_{\text{ego}}$ denote the corresponding ego vehicle quantities, and $a_{\text{ego}}$, $a_{\text{lead}}$ denote the derivative of the target acceleration commands for the ego and lead vehicles.

The dynamics consists of two kinematic chains, one per vehicle of the following form
\[
\dot{x}_{\text{lead}} = v_{\text{lead}}, \quad \dot{v}_{\text{lead}} = \gamma_{\text{lead}}, \quad \dot{x}_{\text{ego}} = v_{\text{ego}}, \quad \dot{v}_{\text{ego}} = \gamma_{\text{ego}},
\]
together with a 2-layer feed-forward neural network $f_{\text{nn}}$ that drives the four non-kinematic components of $\dot{x}$. The network takes the 4-dimensional sub-state $u = (\gamma_{\text{lead}}, \gamma_{\text{ego}}, a_{\text{ego}}, a_{\text{lead}})^\top$ as input and produces the 4-dimensional output $(\dot{\gamma}_{\text{lead}}, \dot{\gamma}_{\text{ego}}, \dot{a}_{\text{ego}}, \dot{a}_{\text{lead}})^\top = f_{\text{nn}}(u)$. We consider two variants of the plant that differ only in the activation function of $f_{\text{nn}}$.

\subsubsection*{Linear}\label{apd:ACC_System_description_linear}

In the linear variant, $f_{\text{nn}}$ has no activation function and is given by
\[
f_{\text{nn}}(u) = W_2 (W_1 u + b_1) + b_2,
\]
with $W_1 \in \mathbb{R}^{h \times 4}$, $b_1 \in \mathbb{R}^{h}$, $W_2 \in \mathbb{R}^{4 \times h}$, and $b_2 \in \mathbb{R}^{4}$, where
$h$ denotes the hidden layer width of the parameterizing network (with $h=20$ in the trained model used in our experiments). Since the hidden layer applies no activation function, $f_{\text{nn}}$ reduces to an affine map $f_{\text{nn}}(u) = A u + b$ with $A = W_2 W_1 \in \mathbb{R}^{4 \times 4}$ and $b = W_2 b_1 + b_2 \in \mathbb{R}^{4}$. The exact values of the weight matrices and bias vectors are loaded from the MATLAB data file \emph{plant\_3rd\_order\_node.mat}.

\subsubsection*{Nonlinear}\label{apd:ACC_System_description_nonlinear}

In the nonlinear variant, $f_{\text{nn}}$ applies a $\tanh$ activation between the two layers and is given by
\[
f_{\text{nn}}(u) = W_2 \tanh(W_1 u + b_1) + b_2,
\]
where $W_1, b_1, W_2, b_2$ are as in the linear variant and $\tanh(\cdot)$ is applied element-wise to the hidden vector $W_1 u + b_1 \in \mathbb{R}^{h}$. The exact values of the weight matrices and bias vectors are loaded from the MATLAB data file \emph{plant\_3rd\_order\_tanh.mat}.

\subsection{General neural ODE (GNODE) classification benchmarks}
\label{apd:gnode}
\subsubsection{MNIST Classification neural ODE}\label{apd:classification_nODE_System_description}

The MNIST classification benchmark uses the standard MNIST handwritten-digit data set~\citep{lecun1998gradient}, embedded in a general neural ODE (GNODE) classifier architecture introduced by \citet{manzanas2022nnvode}, in which an ODE block is added between standard convolutional layers in a feed-forward digit classifier. The input is a $28 \times 28$ grayscale image and the output is a $10$-dimensional vector of logits, one per digit class. We use the same architecture as in \citet{manzanas2022nnvode}, which combines convolutional pre-processing, a continuous linear ODE block, and a fully-connected output layer. 

\paragraph{Architectures.}
We evaluate the two convolutional variants CNODE\textsubscript{S} and CNODE\textsubscript{M}, which share a layer-wise architecture as illustrated in Table \ref{tab:mnist-arch} and differ only in the convolutional channel count. The two convolutions reduce the spatial dimensions from $28 \times 28$ at the input to $13 \times 13$ at the ODE input, the first convolution does not use padding and produces a $26 \times 26$ feature map, and the second convolution uses padding $1$ and stride $2$ to divide the spatial dimensions to $13 \times 13$. After flattening, the ODE input dimension is $676$ for CNODE\textsubscript{S} and $1690$ for CNODE\textsubscript{M}.

\begin{table}[h]
\centering
\caption{Layer-wise architecture of the MNIST GNODE classifiers CNODE\textsubscript{S} and CNODE\textsubscript{M}.}
\label{tab:mnist-arch}
\small
\begin{tabular}{rlll}
\toprule
\textbf{Layer} & \textbf{Operation} & \textbf{Output (CNODE\textsubscript{S})} & \textbf{Output (CNODE\textsubscript{M})} \\
\midrule
1 & Conv2D                                              & $26 \times 26 \times 4$ & $26 \times 26 \times 10$ \\
2 & BatchNorm                                           & $26 \times 26 \times 4$ & $26 \times 26 \times 10$ \\
3 & ReLU                                                & $26 \times 26 \times 4$ & $26 \times 26 \times 10$ \\
4 & Conv2D                                              & $13 \times 13 \times 4$ & $13 \times 13 \times 10$ \\
5 & BatchNorm                                           & $13 \times 13 \times 4$ & $13 \times 13 \times 10$ \\
6 & ReLU                                                & $13 \times 13 \times 4$ & $13 \times 13 \times 10$ \\
7 & Flatten                                             & $\mathbb{R}^{676}$ & $\mathbb{R}^{1690}$ \\
8 & Linear ODE block (see below)                        & $\mathbb{R}^{676}$ & $\mathbb{R}^{1690}$ \\
9 & FullyConnected ($\mathbb{R}^n \to \mathbb{R}^{10}$) & $\mathbb{R}^{10}$ & $\mathbb{R}^{10}$ \\
\bottomrule
\end{tabular}
\end{table}

\paragraph{ODE block.}
The ODE block (layer 8) has linear dynamics of the following form
\[
  \dot{x} \;=\; A x + b, \qquad x \in \mathbb{R}^n,
\]
with $A \in \mathbb{R}^{n \times n}$ and $b \in \mathbb{R}^n$. The values of $A$, $b$, the convolutional weights, the batch-normalization parameters, and the output-layer weights are obtained by loading the trained MATLAB data files \emph{odecnn\_mnist\_tiny.mat} for CNODE\textsubscript{S} and \emph{odecnn\_mnist\_mid2.mat} for CNODE\textsubscript{M}.

\subsubsection{Hybrid reachability pipeline}\label{apd:classification_nODE_Hybrid_reachability_pipeline}

Verification of these GNODE classifiers requires computing the reachability throughout the entire network. Since the pre-ODE convolutional and ReLU layers and the post-ODE fully-connected layer are discrete neural network layers that TNODEV does not support natively, we therefore construct a hybrid pipeline that uses NNV 2.0 ImageStar reachability for the discrete pre-ODE layers (layers 1--7), TNODEV's reachability block for the linear ODE block (layer 8), and NNV 2.0  Star set reachability for the final fully-connected layer (layer 9).

TNODEV main reachability block is based on CTMM as discussed in Section \ref{subsec:ctmm}, and CTMM is the backend used for all benchmarks in Section~\ref{subsec:safe_set_benchmarks}. For linear ODE block of the MNIST GNODE classifier specifically, we instead use a closed-form matrix-exponential propagator that gives the exact linear flow over the integration interval $[0, t_f]$, followed by the standard interval analysis construction \citep{moore2009introduction} that produces the tightest axis-aligned box enclosing the affine image of a box.

We attempted to apply CTMM to the linear ODE block of this benchmark on a full image $L_\infty$ attack (i.e., all $784$ pixels are perturbed), but CTMM produced no verified images on either network within the refinement iteration $K$ and wall-clock timeout $T_{\max}$, and only on a single pixel attack on individual images did we obtain verified verdicts in approximately $26$ minutes per image. We attribute this to the wrapping effect of integrating the embedded $2n$-dimensional monotone system over the time horizon on a large, densely coupled $A$.
The closed-form matrix-exponential propagator avoids this by encoding the exact linear flow of $\dot{x} = Ax + b$ over $[0, t_f]$ in a single step, with no time integration of an embedded system.
A similar matrix-exponential construction is implemented in NNV 2.0 \texttt{LinearODE} class and is used by~\citet{manzanas2022nnvode} for the linear ODE blocks GNODE benchmarks. The difference between the two pipelines is in what is propagated through this construction, as NNV 2.0 propagates a predicate-constrained Star set, whereas TNODEV propagates an axis-aligned box.

We now make the construction precise and establish its soundness. We state two lemmas characterizing the two pieces of the construction (the closed-form propagator and the tightest box enclosure of the affine image of a box), and combine them in a sound theorem.

\begin{lemma}[Augmented matrix exponential propagator]\label{lem:expm-propagator}
Let $A \in \mathbb{R}^{n \times n}$, $b \in \mathbb{R}^n$, $t_f \in \mathbb{R}$. Define the augmented matrix
\[
  \bar{A} \;=\; \begin{bmatrix} A & b \\ 0 & 0 \end{bmatrix} \in \mathbb{R}^{(n+1) \times (n+1)},
\]
and let $E = \exp(\bar{A} \, t_f)$. Then $E$ has the block form
\[
  E \;=\; \begin{bmatrix} M & c \\ 0 & 1 \end{bmatrix}
\]
for some $M \in \mathbb{R}^{n \times n}$ and $c \in \mathbb{R}^n$.
In addition, for every initial condition $x_0 \in \mathbb{R}^n$, the unique solution $x(\cdot)$ of the linear ODE $\dot{x} = A x + b$ with $x(0) = x_0$ satisfies
\[
  x(t_f) \;=\; M x_0 + c.
\]
\end{lemma}

\begin{proof}
We first establish the block structure of $E$. The bottom row of $\bar{A}$ is zero, so for every series index $j \ge 1$ the bottom row of $\bar{A}^j$ is zero (the bottom row of a matrix product is a weighted some of the bottom row of its left factor). The matrix exponential admits the absolutely convergent series expansion $\exp(\bar{A} \, t_f) = \sum_{j=0}^{\infty} \frac{(\bar{A} \, t_f)^j}{ j!}$, in which only the $j=0$ term contributes a nonzero bottom row, namely $[0, \ldots, 0, 1]$ from the identity. Hence $E$ has bottom row $[0, \ldots, 0, 1]$, giving the claimed block form with $M = E_{[1:n, 1:n]}$ and $c = E_{[1:n, n+1]}$.

We now derive the solution formula. Define the augmented state $z(t) := (x(t)^\top, 1)^\top \in \mathbb{R}^{n+1}$. Differentiating gives
\[
  \dot{z}(t) \;=\; \begin{pmatrix} \dot{x}(t) \\ 0 \end{pmatrix} \;=\; \begin{pmatrix} A x(t) + b \\ 0 \end{pmatrix} \;=\; \begin{bmatrix} A & b \\ \mathbf{0} & 0 \end{bmatrix} \begin{pmatrix} x(t) \\ 1 \end{pmatrix} \;=\; \bar{A} z(t),
\]
so $z$ satisfies a homogeneous linear ODE with constant matrix $\bar{A}$. The unique solution of this ODE with initial condition $z(0) = (x_0^\top, 1)^\top$ is $z(t_f) = \exp(\bar{A} \, t_f) \, z(0) = E \, (x_0^\top, 1)^\top$. Reading off the top $n$ components yields $x(t_f) = M x_0 + c$.
\end{proof}

The construction in Lemma~\ref{lem:expm-propagator} relies on the standard affine-to-homogeneous lift of a linear time-invariant system \citep{antsaklis2007linear}, in which a constant unit coordinate is appended to the state so that the affine system $\dot{x} = A x + b$ becomes the homogeneous system $\dot{z} = \bar{A} z$ for the augmented state $z = (x^\top, 1)^\top$. The matrix $M$ defined in Lemma~\ref{lem:expm-propagator} then coincides with the state transition matrix $\exp(A \, t_f)$, and the offset $c$ coincides with the variation of constants integral $\int_0^{t_f} \exp(A (t_f - s)) \, b \, ds$. The augmented matrix construction recovers both quantities from a single matrix exponential of $\bar{A}$, and no explicit numerical integration is required.

\begin{lemma}[Tightest box enclosure of the affine image of a box]\label{lem:interval-affine-image}
Let $M \in \mathbb{R}^{n \times n}$, $c \in \mathbb{R}^n$, and let $[\underline{x}, \overline{x}] \subseteq \mathbb{R}^n$ with $\underline{x} \le \overline{x}$ component-wise. Define $M_+ = \max(M, \mathbf{0})$ and $M_- = \min(M, \mathbf{0})$ element-wise. Then the smallest axis-aligned box $[\underline{y}, \overline{y}]$ containing the set $\{M x + c : x \in [\underline{x}, \overline{x}]\}$ is given by
\[
  \underline{y} \;=\; M_+ \underline{x} + M_- \overline{x} + c, \qquad \overline{y} \;=\; M_+ \overline{x} + M_- \underline{x} + c.
\]
\end{lemma}

\begin{proof}
For a fixed row index $i \in \{1, \ldots, n\}$, the $i$-th component of the affine map is $(M x + c)_i = \sum_{j=1}^n M_{ij} \, x_j + c_i$. The terms decouple across $j$, so the minimum of $(M x + c)_i$ over $x \in [\underline{x}, \overline{x}]$ is the sum of the per-coordinate minima of $M_{ij} \, x_j$ over $x_j \in [\underline{x}_j, \overline{x}_j]$, plus $c_i$. We compute these per-coordinate minima as follows:
\begin{itemize}
\item If $M_{ij} \ge 0$, then $M_{ij} \, x_j$ is non-decreasing in $x_j$, so its minimum over $[\underline{x}_j, \overline{x}_j]$ is $M_{ij} \, \underline{x}_j$. In this case $(M_+)_{ij} = M_{ij}$ and $(M_-)_{ij} = 0$, so $M_{ij} \, \underline{x}_j = (M_+)_{ij} \, \underline{x}_j + (M_-)_{ij} \, \overline{x}_j$.
\item If $M_{ij} < 0$, then $M_{ij} \, x_j$ is decreasing in $x_j$, so its minimum over $[\underline{x}_j, \overline{x}_j]$ is $M_{ij} \, \overline{x}_j$. In this case $(M_+)_{ij} = 0$ and $(M_-)_{ij} = M_{ij}$, so $M_{ij} \, \overline{x}_j = (M_+)_{ij} \, \underline{x}_j + (M_-)_{ij} \, \overline{x}_j$.
\end{itemize}
In both cases the per-coordinate minimum equals $(M_+)_{ij} \, \underline{x}_j + (M_-)_{ij} \, \overline{x}_j$. Summing over $j$ and adding $c_i$ yields
\[
  \min_{x \in [\underline{x}, \overline{x}]} (M x + c)_i \;=\; \sum_{j=1}^n \left[ (M_+)_{ij} \, \underline{x}_j + (M_-)_{ij} \, \overline{x}_j \right] + c_i \;=\; (M_+ \underline{x} + M_- \overline{x} + c)_i \;=\; \underline{y}_i.
\]
A symmetric argument with minima replaced by maxima yields $\max_{x \in [\underline{x}, \overline{x}]} (M x + c)_i = (M_+ \overline{x} + M_- \underline{x} + c)_i = \overline{y}_i$. Both extrema are attained by points in $[\underline{x}, \overline{x}]$ (specifically, by the corners identified above, which may differ for different rows $i$). Therefore $[\underline{y}, \overline{y}]$ contains the image set and no smaller axis-aligned box has this property.
\end{proof}

We can now combine the two lemmas to establish the soundness of the linear ODE interval reachability used by TNODEV on this benchmark.

\begin{theorem}[Soundness of the closed-form interval reachability for the linear ODE block]\label{thm:linear-ode-interval-soundness}
Let $A \in \mathbb{R}^{n \times n}$, $b \in \mathbb{R}^n$, $t_f \in \mathbb{R}$, and let $[\underline{x}_0, \overline{x}_0] \subseteq \mathbb{R}^n$. Let $M, c$ be as in Lemma~\ref{lem:expm-propagator} and $M_+, M_-$ as in Lemma~\ref{lem:interval-affine-image}. Define
\[
  \underline{x}_{t_f} \;=\; M_+ \underline{x}_0 + M_- \overline{x}_0 + c, \qquad \overline{x}_{t_f} \;=\; M_+ \overline{x}_0 + M_- \underline{x}_0 + c.
\]
Then for every initial condition $x_0 \in [\underline{x}_0, \overline{x}_0]$, the unique solution $x(\cdot)$ of $\dot{x} = A x + b$ with $x(0) = x_0$ satisfies $x(t_f) \in [\underline{x}_{t_f}, \overline{x}_{t_f}]$. Moreover, $[\underline{x}_{t_f}, \overline{x}_{t_f}]$ is the smallest axis-aligned box containing the reachable set $\{\Phi(t_f, x_0) : x_0 \in [\underline{x}_0, \overline{x}_0]\}$. 
\end{theorem}

\begin{proof}
Let $x_0 \in [\underline{x}_0, \overline{x}_0]$ be an arbitrary initial condition, and let $x(\cdot)$ denote the unique solution of $\dot{x} = A x + b$ with $x(0) = x_0$. By Lemma~\ref{lem:expm-propagator}, $x(t_f) = M x_0 + c$, so $x(t_f) \in \mathcal{R} := \{M x + c : x \in [\underline{x}_0, \overline{x}_0]\}$. By Lemma~\ref{lem:interval-affine-image}, the smallest axis-aligned box containing $\mathcal{R}$ is exactly $[\underline{x}_{t_f}, \overline{x}_{t_f}]$. Hence $x(t_f) \in [\underline{x}_{t_f}, \overline{x}_{t_f}]$, and no smaller axis-aligned box can contain the reachable set $\mathcal{R}$ since by Lemma~\ref{lem:expm-propagator} the reachable set is in bijection with the affine image $\{M x_0 + c : x_0 \in [\underline{x}_0, \overline{x}_0]\}$.
\end{proof}

The output bounds in Theorem~\ref{thm:linear-ode-interval-soundness} are evaluated in a single pass, as computing $\underline{x}_{t_f}$ and $\overline{x}_{t_f}$ requires one matrix exponential (precomputed once per network) and two matrix-vector products per image (two for the lower bound  and two for the upper bound). The output interval $[\underline{x}_{t_f}, \overline{x}_{t_f}]$ is the tightest sound axis-aligned over-approximation of the reachable set. The reachable set itself is in general a zonotope-like polytope (the affine image of an axis-aligned box under $M$) and not exactly a box. This is precisely why TNODEV's interval representation is lossy on this benchmark compared to NNV 2.0 Star set representation, which can capture the polytope structure exactly.

\subsubsection{Star-set to interval conversion at the ODE input}\label{apd:classification_nODE_star_to_interval}

The pre-ODE layers (layers~1--7) produce a Star set representation $R_{\mathrm{pre}}$ with predicate variables $\alpha$ constrained by linear inequalities $C \alpha \le d$. NNV 2.0 approximate ReLU reachability adds these inequality constraints whenever a neuron's input range contains both positive and negative values, the constraints carry information that the interval box hull alone does not. Since our closed-form propagator operates on intervals, the pipeline involves a Star set $\to$ interval $\to$ Star set conversion before and after the ODE block boundary. We extract per-coordinate bounds $[\underline{x}_{\mathrm{pre}}, \overline{x}_{\mathrm{pre}}]$ from $R_{\mathrm{pre}}$ at the ODE input, propagate them through the closed-form propagator (Theorem~\ref{thm:linear-ode-interval-soundness}) to obtain $[\underline{x}_{\mathrm{ode}}, \overline{x}_{\mathrm{ode}}]$, and then re-encode this interval as a Star set (specifically, a Star set whose predicate variables are bounded by the box $[\underline{x}_{\mathrm{ode}}, \overline{x}_{\mathrm{ode}}]$) so that NNV 2.0 exact Star set reachability can run on the final fully-connected layer.

The forward conversion (Star set $\to$ interval) discards the predicate constraints $C \alpha \le d$. We provide two sound modes that differ in whether they exploit these constraints when extracting bounds:
\begin{itemize}
\item \emph{LP-tight} uses NNV~2.0  \texttt{getRanges()}, which solves one linear program per output coordinate to obtain the tightest sound per-coordinate bounds that respects the predicate constraints $C \alpha \le d$. The mode inherits NNV 2.0 reliance on \texttt{linprog}, which we observed to occasionally encounter degenerate LP geometries on individual Star sets.
\item \emph{Fast} uses NNV 2.0 \texttt{getBoxFast()}, which ignores the predicate constraints and treats $\alpha$ as a hyper-rectangle. The resulting bounds remain sound but are looser whenever predicate constraints are active.
\end{itemize}

The reverse conversion (interval $\to$ Star set) is direct: given the post-propagator interval $[\underline{x}_{\mathrm{ode}}, \overline{x}_{\mathrm{ode}}]$, we construct a Star set with no inequality predicate constraints whose predicate variables vary independently within these bounds.

We report both forward-conversion modes in Table~\ref{tab:mnist-cnode} to separate the contribution of bound-extraction tightness at the boundary from the interval-propagation lossiness incurred by the closed-form propagator. On this benchmark the two modes return identical verdict counts, indicating that the interval over-approximation through the linear ODE is the binding constraint rather than the predicate-constraint information at the boundary.

\paragraph{Iterative refinement.}
Neither of the two tools uses iterative refinement for this benchmark. Indeed, NNV 2.0 does not currently implement input-set refinement, so to keep the comparison meaningful we also use TNODEV without refinements. We note that if refinement is added, it will be applied to the global GNODE input set rather than to the ODE block input only, even though NNV 2.0 Star set propagation through $\exp(A t_f)$ is exact at the ODE layer, the pre-ODE convolutional and ReLU layers use approximate reachability, so exactness at the ODE layer alone does not make global refinement unnecessary. For TNODEV, refinement on the global GNODE input would split the input set into sub-sets, propagate each through the full hybrid pipeline, and take the union of the resulting verdicts. We leave this global input-set refinement strategy as a future work.


\subsubsection{Where the verification gap comes from}\label{apd:classification_nODE_verification_gap}

The 20-image gap between TNODEV and NNV 2.0 originates entirely in the linear ODE block. NNV 2.0 propagates the predicate-constrained Star set representation through the linear ODE exactly, while TNODEV interval propagation produces the bounding box of the affine image of the input box (Theorem~\ref{thm:linear-ode-interval-soundness}): a sound enclosure, but a strict super-set of the actual image set whenever the input box's image under $M$ is not itself a box. The characteristic worst-case widening of this enclosure can be summarized by the propagator's $L_\infty$ norm: $\|M\|_\infty \approx 10.5$ on CNODE\textsubscript{S} and $\|M\|_\infty \approx 14.1$ on CNODE\textsubscript{M}, where $M$ is the propagator matrix from Lemma~\ref{lem:expm-propagator}. The quantity $\|M\|_\infty$ is the maximum row sum of $|M|$; if every input coordinate has width $\delta$ then the maximum output coordinate width is bounded by $\|M\|_\infty \cdot \delta$. This is a diagnostic characterization of the interval hull wrapping, not an additional factor applied during reachability, and the interval propagation itself is a single closed-form evaluation. The actual aggregate width expansion observed on representative test images is roughly $7\times$ on CNODE\textsubscript{S} ($\sum w_{\mathrm{post}} / \sum w_{\mathrm{pre}} = 6.91$) and $8\times$ on CNODE\textsubscript{M} (ratio $8.36$) at $\epsilon=0.5/255$. After this widening, NNV2.0 tighter post-ODE Star set representation still produces logit bounds where the target-class lower bound dominates the other classes' upper bounds on nearly every test image (the condition required for $\epsilon$-robustness). TNODEV looser post-ODE bounding box only achieves this dominance on the more confidently classified images.

\subsubsection{Applicability of TNODEV on this benchmark}\label{apd:classification_nODE_applicability}

TNODEV is sound and applicable across this benchmark in the sense that the hybrid pipeline runs end-to-end on every test image and produces a sound (over-approximation) verdict. However, the interval-based propagation through the linear ODE block makes TNODEV best suited to small adversarial perturbations on this class of GNODE classifiers, where the interval hull enclosure of the linear ODE image remains tight enough to satisfy the per-class dominance condition required for $\epsilon$-robustness. At $\epsilon = 0.5/255$ (Table~\ref{tab:mnist-cnode}) TNODEV verifies $58$--$60\%$ of test images. At $\epsilon = 1/255$, twice this magnitude, the verification rate falls to $1/50$ on both networks under both bound-extraction modes, while NNV 2.0 retains $49/50$ on CNODE\textsubscript{S} and $50/50$ on CNODE\textsubscript{M}~\citep{manzanas2022nnvode}. For verification under larger adversarial perturbations on this class of linear ODE based classifiers, NNV 2.0 Star set representation remains a more appropriate tool choice than TNODEV interval based propagation.

\section{Example Results}\label{apd:Detailed_results}

\subsection{Reachability comparison of TNODEV against NNV~2.0 and CORA}
\label{apd:reachability_comparison}

This appendix provides more details on the comparison of the reachability analysis results from Section~\ref{subsec:reachability} between TNODEV, NNV 2.0, and CORA on the Spiral nonlinear and FPA benchmarks. For each tool, we measure the wall-clock time to compute a single reachable set over-approximation $\Omega$ from the initial input set $\Xin$, together with the size of the resulting over-approximation. The size metric we report is the \emph{geometric mean width}:

\begin{equation}
\mu(\Omega) \;=\; \sqrt[n]{\mathrm{Vol}(\Omega)},
\label{eq:vol_per_dim}
\end{equation}

where $n$ is the state dimension and $\mathrm{Vol}(\Omega)$ is the volume of the over-approximation. The $n$-th root makes the metric dimensionally comparable across benchmarks of different state dimensions, i.e., $\mu(\Omega)$ has the same units as a single state coordinate, and is exactly the side length of the $n$-dimensional cube of the same volume as $\Omega$. Smaller $\mu(\Omega)$ indicates a tighter over-approximation. For an axis-aligned interval over-approximation $\Omega = [\underline{x}, \overline{x}]$, \eqref{eq:vol_per_dim} simplifies to the geometric mean of the component widths:

\begin{equation}
\mu(\Omega) \;=\; \Bigl(\prod_{i=1}^{n} (\overline{x}_i - \underline{x}_i)\Bigr)^{1/n}.
\end{equation}

\paragraph{Implementation per tool.}
The three tools represent the reachable set differently, so the volume in \eqref{eq:vol_per_dim} is computed differently for each of them:

\begin{itemize}
    \item \textbf{TNODEV} (intervals): $\Omega$ is an axis-aligned interval and $\mathrm{Vol}(\Omega) = \prod_{i=1}^{n}(\overline{x}_i - \underline{x}_i)$.
    \item \textbf{CORA} (zonotopes): $\Omega$ is a zonotope and $\mathrm{Vol}(\Omega)$ is obtained from CORA built-in \texttt{volume()} function, which sums the parallelotope contributions induced by the zonotope's generators.
    \item \textbf{NNV 2.0} (Star set): $\Omega$ is a Star set and $\mathrm{Vol}(\Omega)$ is computed as the volume of the tightest axis-aligned box enclosing the Star set, obtained from NNV 2.0 built-in \texttt{getBox()} function.\footnote{Exact polytope-volume computation in higher dimensions is prohibitively expensive \citep{dyer1988complexity,khachiyan1989problem}, so the box-hull volume is used uniformly across both benchmarks for the NNV 2.0.}
\end{itemize}

\paragraph{Spiral nonlinear results.}
Table \ref{tab:reach_comp_spiral_non} reports the comparison on the Spiral nonlinear 2D benchmark. TNODEV is \textbf{14} times faster than NNV 2.0 and \textbf{83} times faster than CORA, while the over-approximations produced by NNV 2.0 and CORA are approximately \textbf{4} and \textbf{4.3} times tighter than TNODEV. This reflects the cost of axis-aligned interval propagation when the actual reachable set is not axis-aligned, i.e., the spiral dynamics rotate the input box and the resulting interval over-approximation widens substantially because it must contain the entire rotated set in an axis-aligned bounding box, an effect that NNV 2.0 and CORA mitigate through their more complex set representations.

\begin{table}[H]
\centering
\caption{Reachability comparison on Spiral nonlinear. 
}
\label{tab:reach_comp_spiral_non}
\small
\begin{tabular}{lccr}
\toprule
\textbf{Tool} & \textbf{Set representation} & \textbf{$\mu(\Omega)$} & \textbf{Time (s)} \\
\midrule
TNODEV   & Intervals  & 2.031 & 0.30  \\
NNV~2.0  & Star set   & 0.507  & 4.20  \\
CORA     & Zonotopes  & 0.467  & 24.99 \\
\bottomrule
\end{tabular}
\end{table}

\paragraph{FPA results.}
Table \ref{tab:reach_comp_fpa} reports the comparison on the FPA 5D benchmark. TNODEV is approximately \textbf{7} times faster than NNV 2.0 and \textbf{39} times faster than CORA, while TNODEV interval over-approximation is approximately \textbf{1.3} times tighter than NNV 2.0 Star set, while CORA's zonotope remains the tightest at approximately \textbf{1.4} times tighter than TNODEV.

\begin{table}[H]
\centering
\caption{Reachability comparison on FPA.}
\label{tab:reach_comp_fpa}
\small
\begin{tabular}{lccccr}
\toprule
\textbf{Tool} & \textbf{Set representation} & \textbf{$\mu(\Omega)$} & \textbf{Time (s)} \\
\midrule
TNODEV   & Intervals  & 0.030 & 0.32  \\
NNV~2.0  & Star set   & 0.040  & 2.16  \\
CORA     & Zonotopes  & 0.022  & 12.34 \\
\bottomrule
\end{tabular}
\end{table}

Across the two benchmarks, the reachability comparison exposes a consistent trade-off: TNODEV provides the lowest per-call wall-clock time, at the cost of wider over-approximations when the reachable set is far from axis-aligned. The Spiral nonlinear benchmark illustrates a setting where intervals are not the natural set representation, as rotational dynamics produce a reachable set that is poorly approximated by any axis-aligned box. The FPA benchmark illustrates a setting where the reachable set is closer to an axis-aligned box, and interval propagation is competitive with the box hull of richer representations. The verifier architecture of Section~\ref{sect:verif_arch} offers a solution for the Spiral nonlinear case, as instead of relying on a single reachability call to be tight, TNODEV wraps a fast reachability backend based on CTMM in the verification and refinement loop, splitting the input set and re-computing the reachability on each sub-cell until each sub-cell's over-approximation becomes tight enough for the property to be certified, or the refinement iteration limit is reached.





\subsection{Spiral Linear}\label{apd:Linear_Spiral_spec}

The safety specification is:
\begin{itemize}
\item \textit{Initial set:} $\Xin = [1.8, 2.2] \times [-0.2, 0.2]$
\item \textit{Safe set:} $\mathcal{X}_s = [-2,\, 0.1] \times [1.2,\, 5]$
\end{itemize}
The initial set follows the reachability setup of \citet{manzanas2022nnvode}, while the safe set is defined by us for this verification task. Figure \ref{fig:Linear_Spiral_results} shows the verified subsets obtained with the naive and MSIR heuristics, both certify the specification as \texttt{SAFE}, with the union of the \textcolor{violet}{violet} verified subsets covering the initial set and remaining inside the \textcolor{Green}{green} safe set.

\begin{figure}[H]
\centering
\begin{tabular}{@{}c@{\hspace{0.5em}}c@{}}
\includegraphics[width=0.5\textwidth]{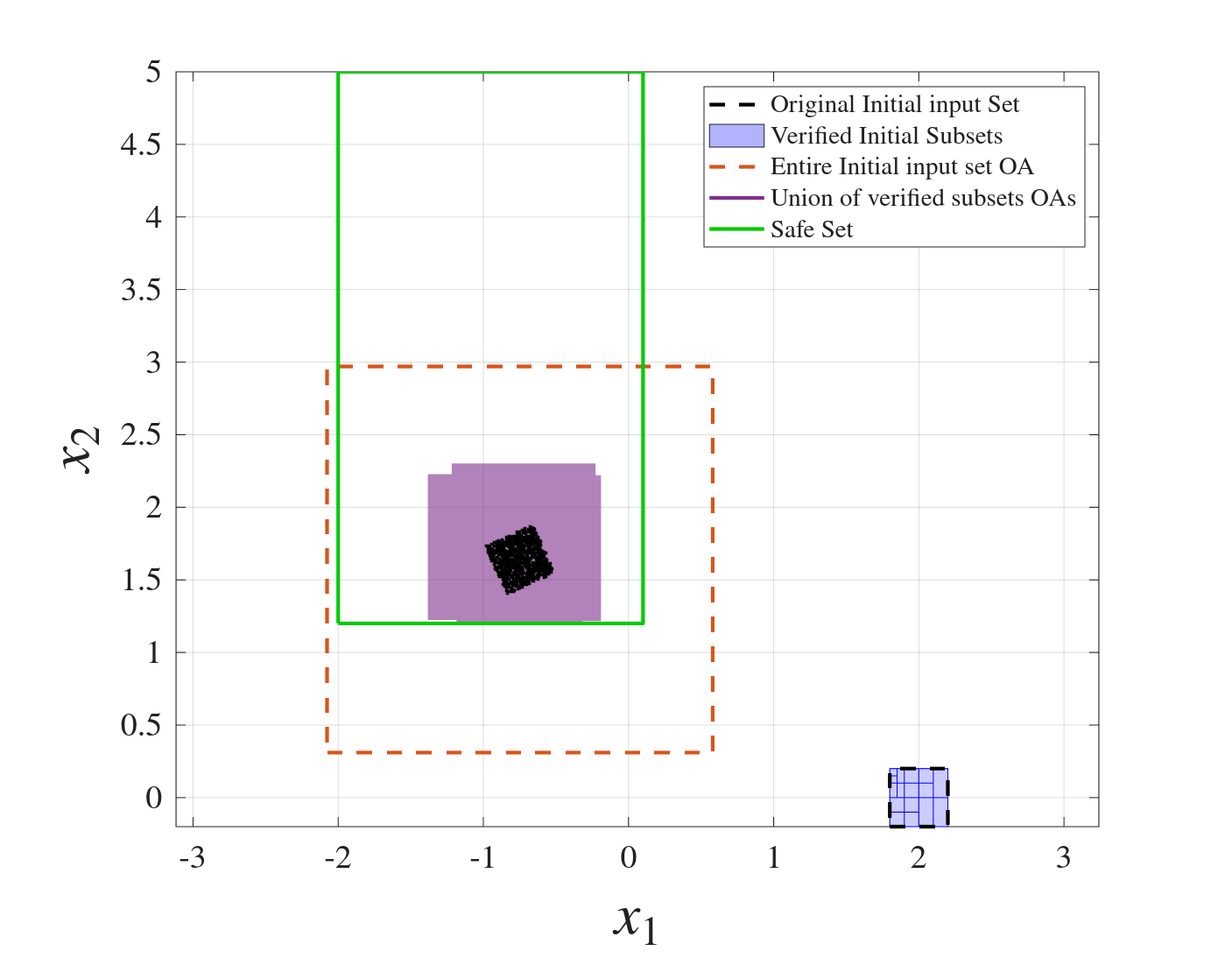} &
\includegraphics[width=0.5\textwidth]{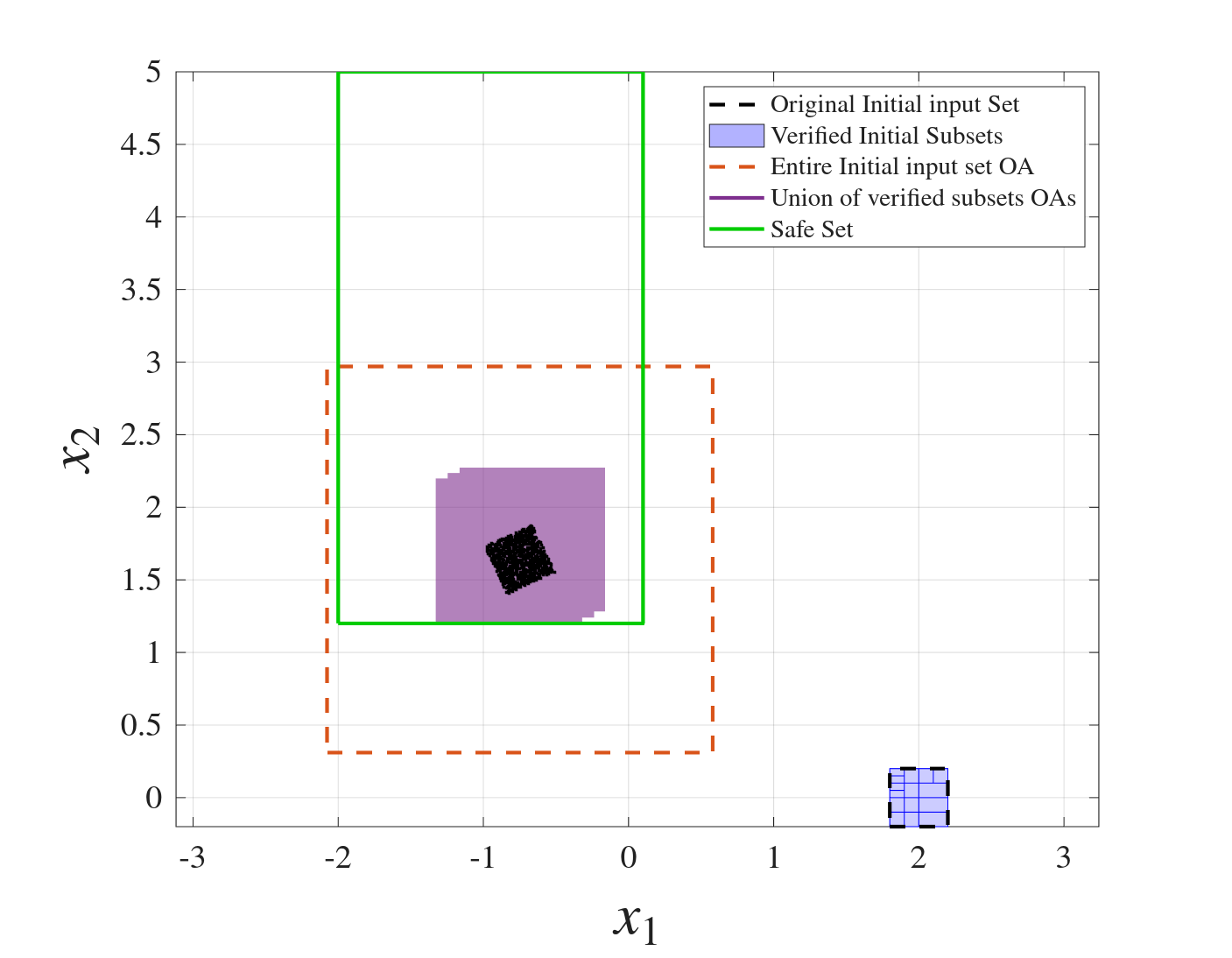} \\
(a) naive & (b) MSIR
\end{tabular}
\caption{TNODEV verification results on the linear spiral~2D benchmark using naive and MSIR refinement. In all sub figures, the dashed \textcolor{red}{red} rectangle is the over-approximation $\Omega(\Xin)$ from a single reachability call on the unsplit initial set, the \textcolor{Green}{green} rectangle is the safe set $\mathcal{X}_s$, the \textcolor{violet}{violet} filled rectangles represent the union of the verified subsets, and the black dots are sampled successor trajectories.}
\label{fig:Linear_Spiral_results}
\end{figure}

\subsection{Spiral nonlinear}\label{apd:Nonlinear_Spiral_spec}

The initial set and safety specification is similar to spiral linear in Appendix \ref{apd:Linear_Spiral_spec} in Figure~\ref{fig:Nonlinear_Spiral_results} which is used only for this illustration and does not appear in Table~\ref{tab:results}. However for the case in Table~\ref{tab:results} and Figure~\ref{fig:verifier_results}(a), the safe set differs and it is defined as $\mathcal{X}_s = [-3,\, 0] \times [1.5,\, 3]$, which the counter-example $(1.8012,\, 0.1846) \in \Xin$ violates, since its final state $(-0.7036,\, 1.4828)$ lies below the lower bound $x_2 = 1.5$ of $\mathcal{X}_s$.

\begin{figure}[H]
\centering
\begin{tabular}{@{}c@{\hspace{0.5em}}c@{}}
\includegraphics[width=0.5\textwidth]{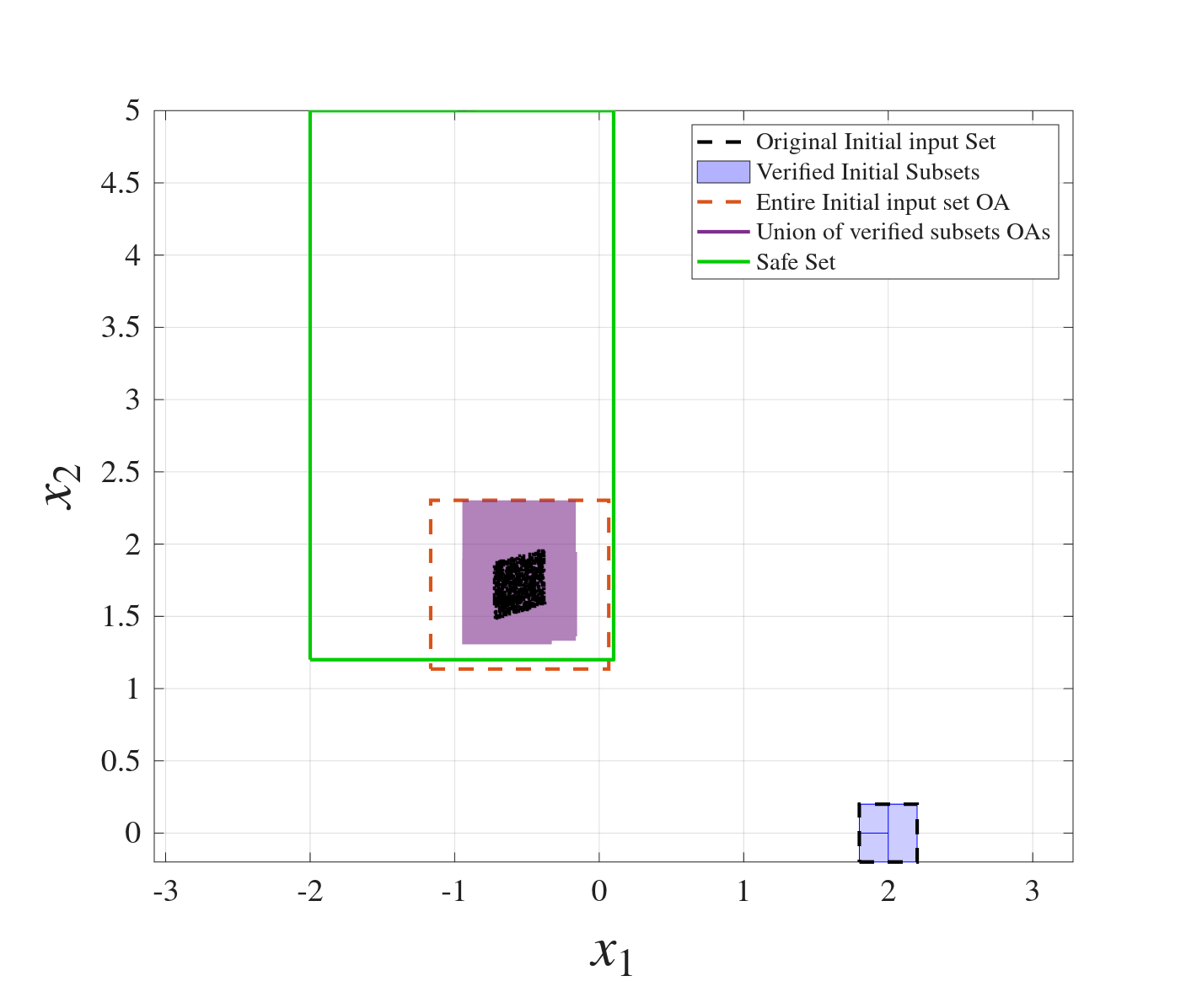} &
\includegraphics[width=0.5\textwidth]{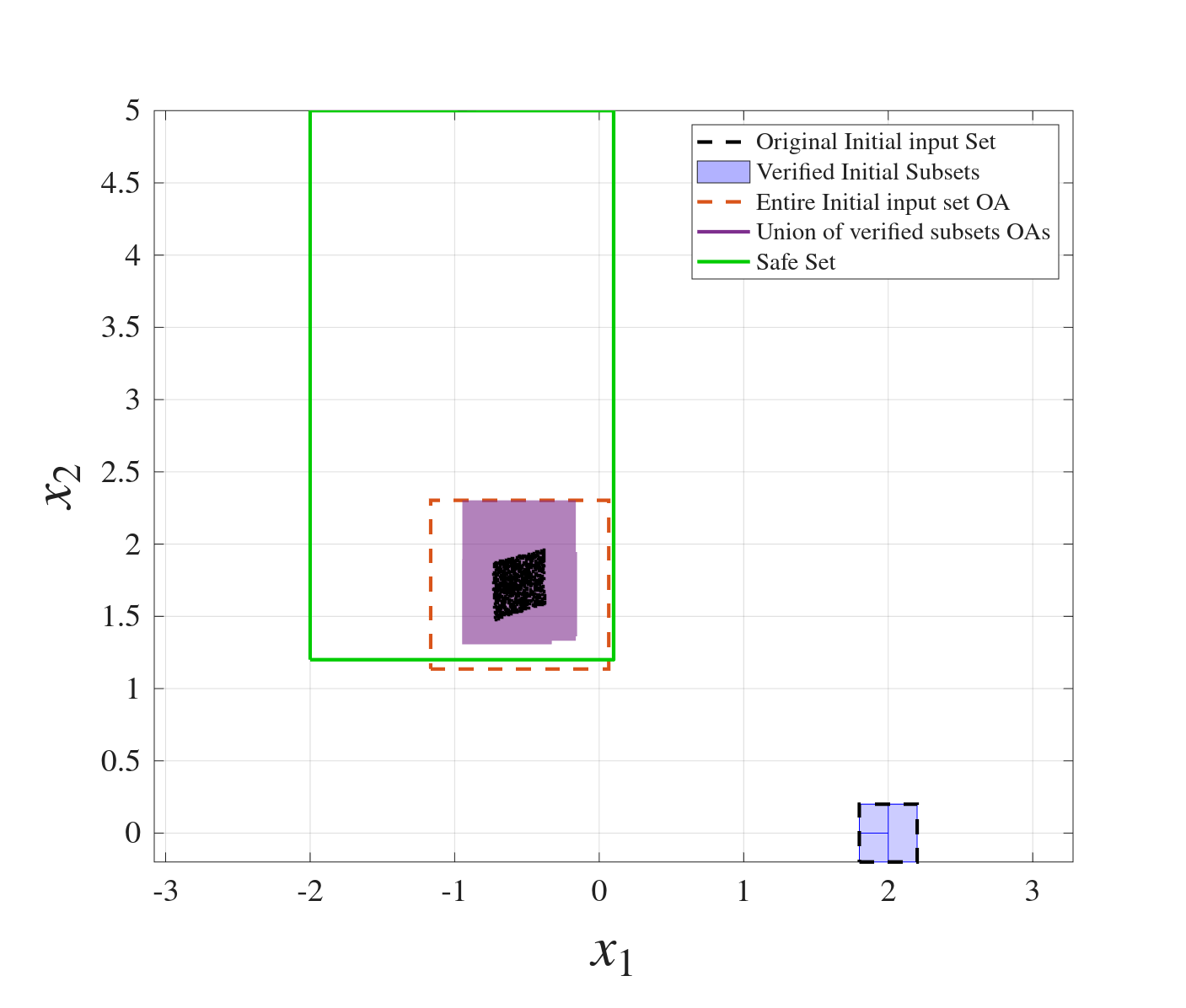} \\
(a) naive & (b) MSIR \\[0.8em]
\multicolumn{2}{c}{\includegraphics[width=0.5\textwidth]{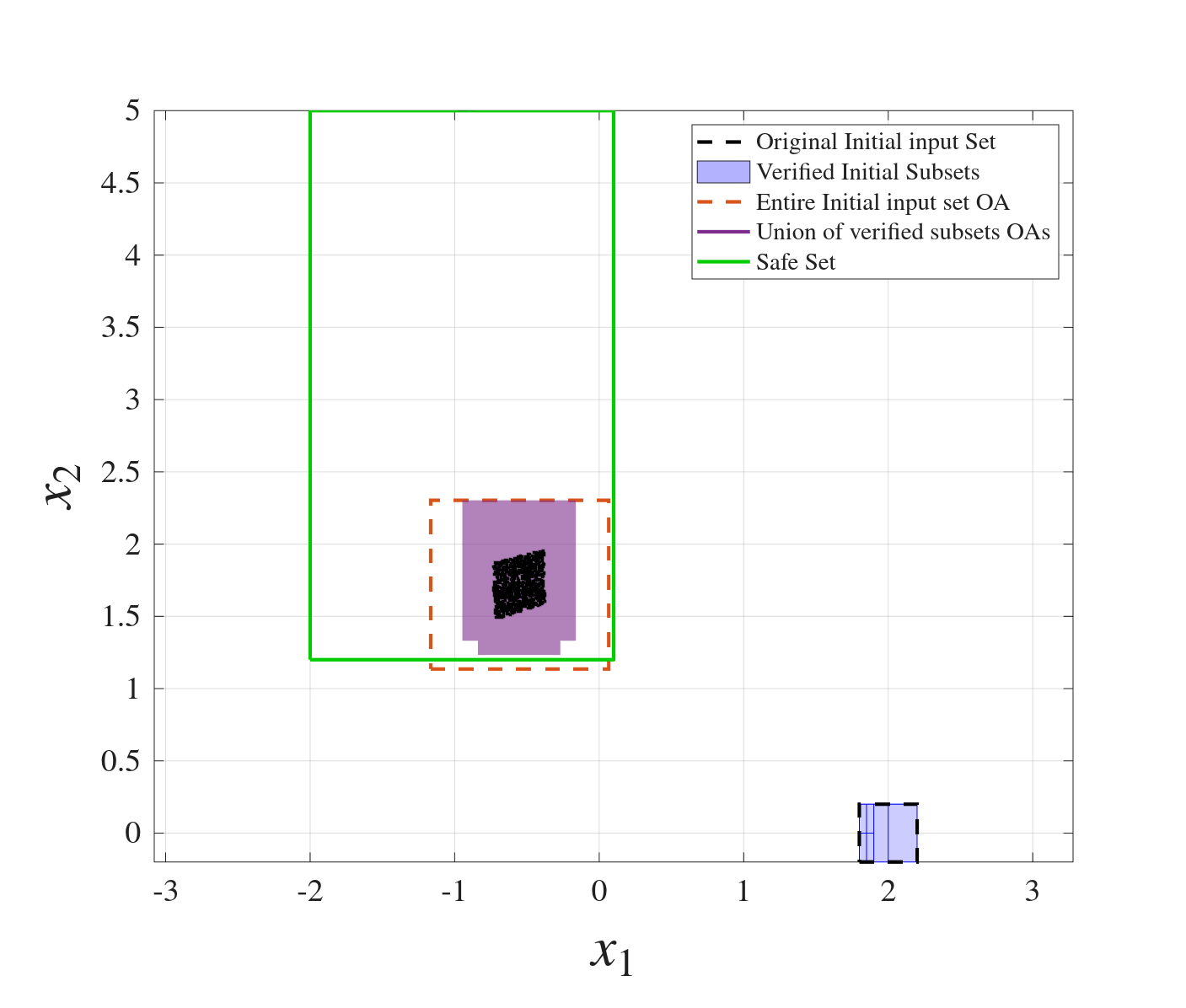}} \\
\multicolumn{2}{c}{(c) ING}
\end{tabular}
\caption{TNODEV verification results on the nonlinear spiral~2D benchmark using naive, MSIR and ING refinement. In all sub figures, the dashed \textcolor{red}{red} rectangle is the over-approximation $\Omega(\Xin)$ from a single reachability call on the unsplit initial set, the \textcolor{Green}{green} rectangle is the safe set $\mathcal{X}_s$, the \textcolor{violet}{violet} filled rectangles represent the union of the verified subsets, and the black dots are sampled successor trajectories.}
\label{fig:Nonlinear_Spiral_results}
\end{figure}

\subsection{FPA}\label{apd:FPA_spec}

The safety specification is:
\begin{itemize}
\item \textit{Initial set:} $\Xin = [-0.01, 0.01] \times [-0.59587, -0.57587] \times [0.79, 0.81] \times [0.51323, 0.53323] \times [0.69, 0.71]$
\item \textit{Safe set:} $\mathcal{X}_s = [-1.29, -1.23] \times [-0.70, -0.59] \times [-0.75, -0.72] \times [0.35, 0.38] \times [2.38, 2.46]$
\end{itemize}

The initial set follows the reachability setup of \citet{musau2018continuous}, the safe set is defined by us, as FPA appears in the literature as a reachability benchmark without an associated property (see Appendix~\ref{apd:System_description}).

\begin{figure}[H]
\centering
\begin{tabular}{@{}c@{\hspace{0.5em}}c@{}}
\includegraphics[width=0.48\textwidth]{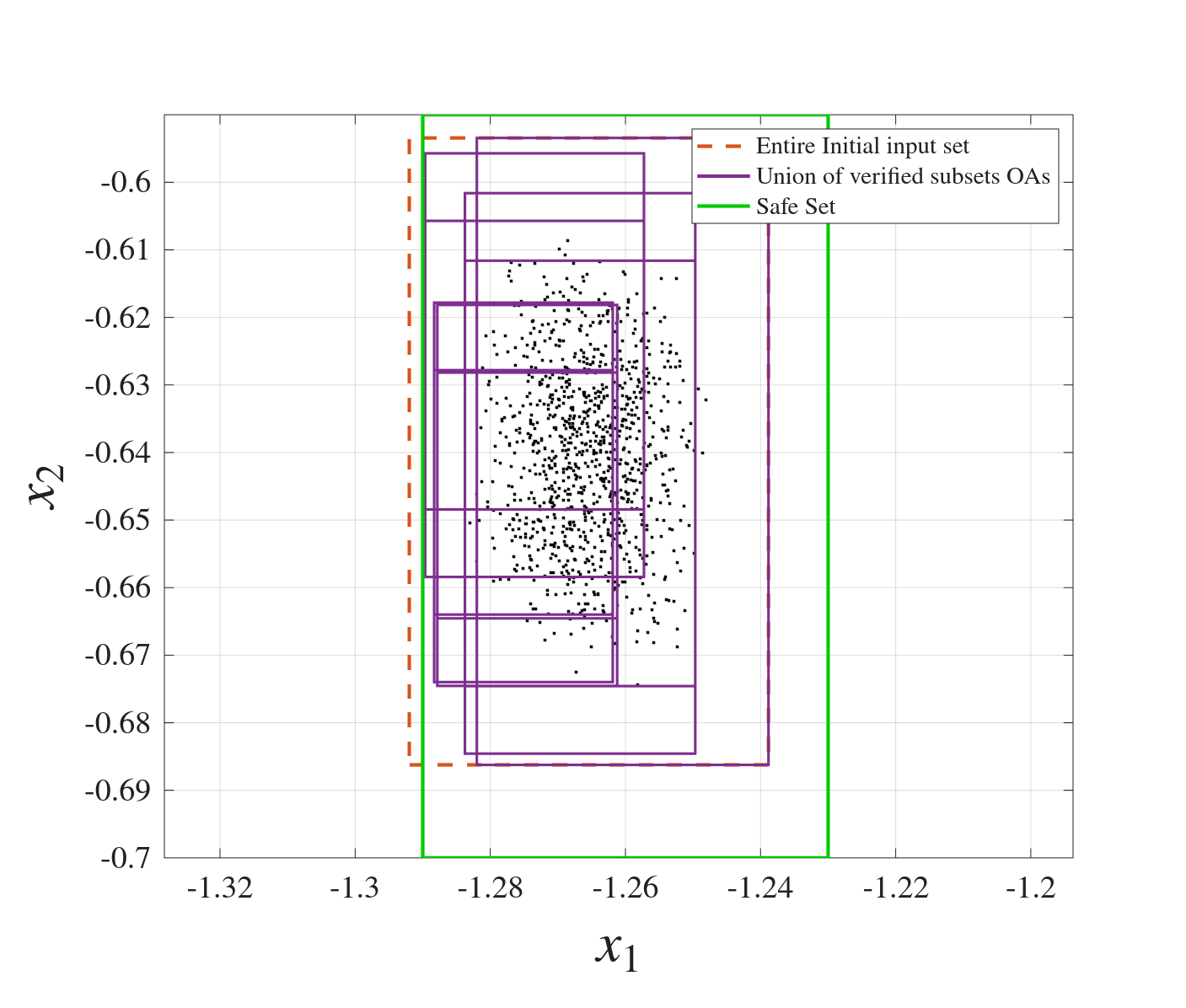} &
\includegraphics[width=0.5\textwidth]{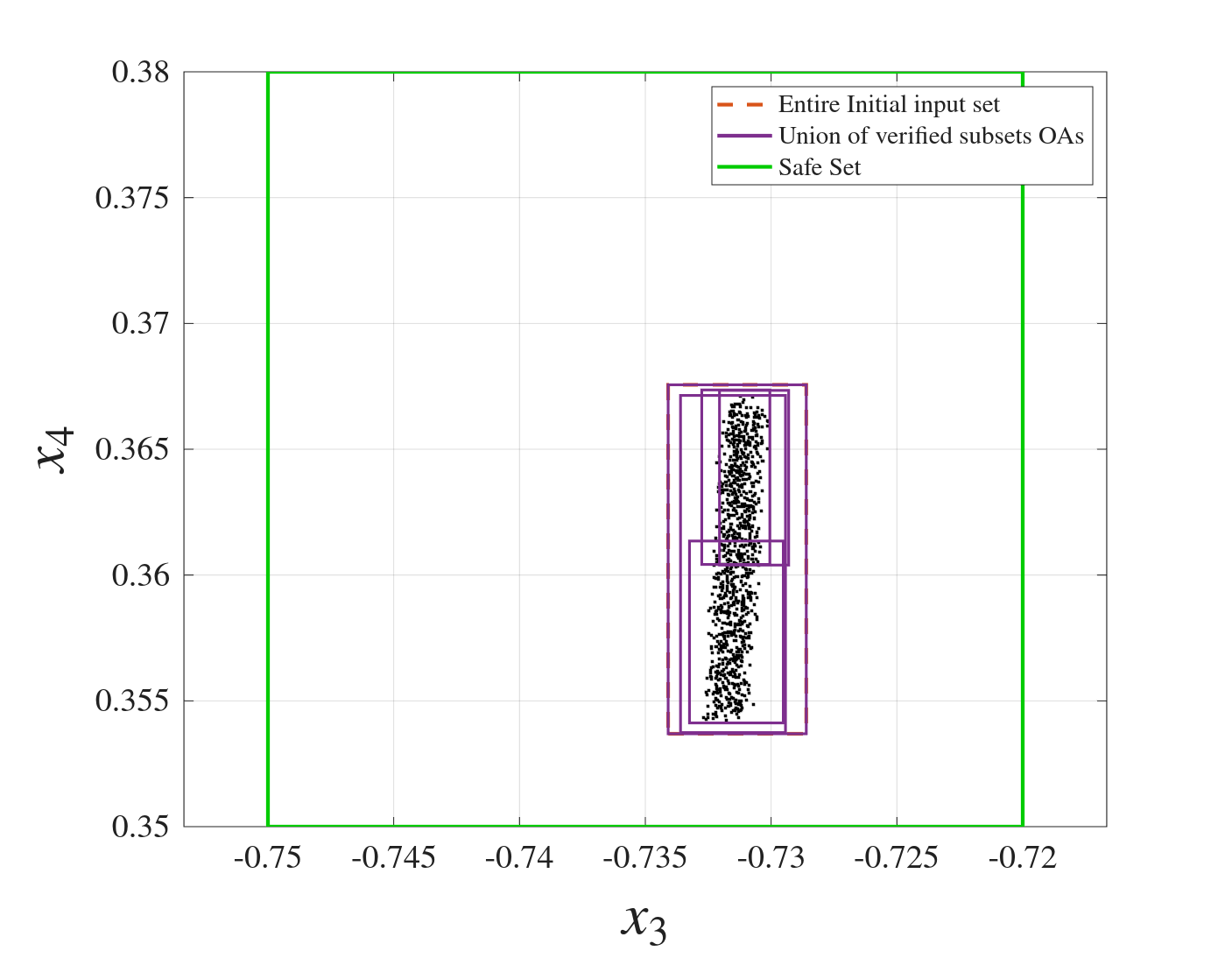} \\
(a) $x_1 - x_2$ & (b) $x_3 - x_4$ \\[0.8em]
\multicolumn{2}{c}{\includegraphics[width=0.5\textwidth]{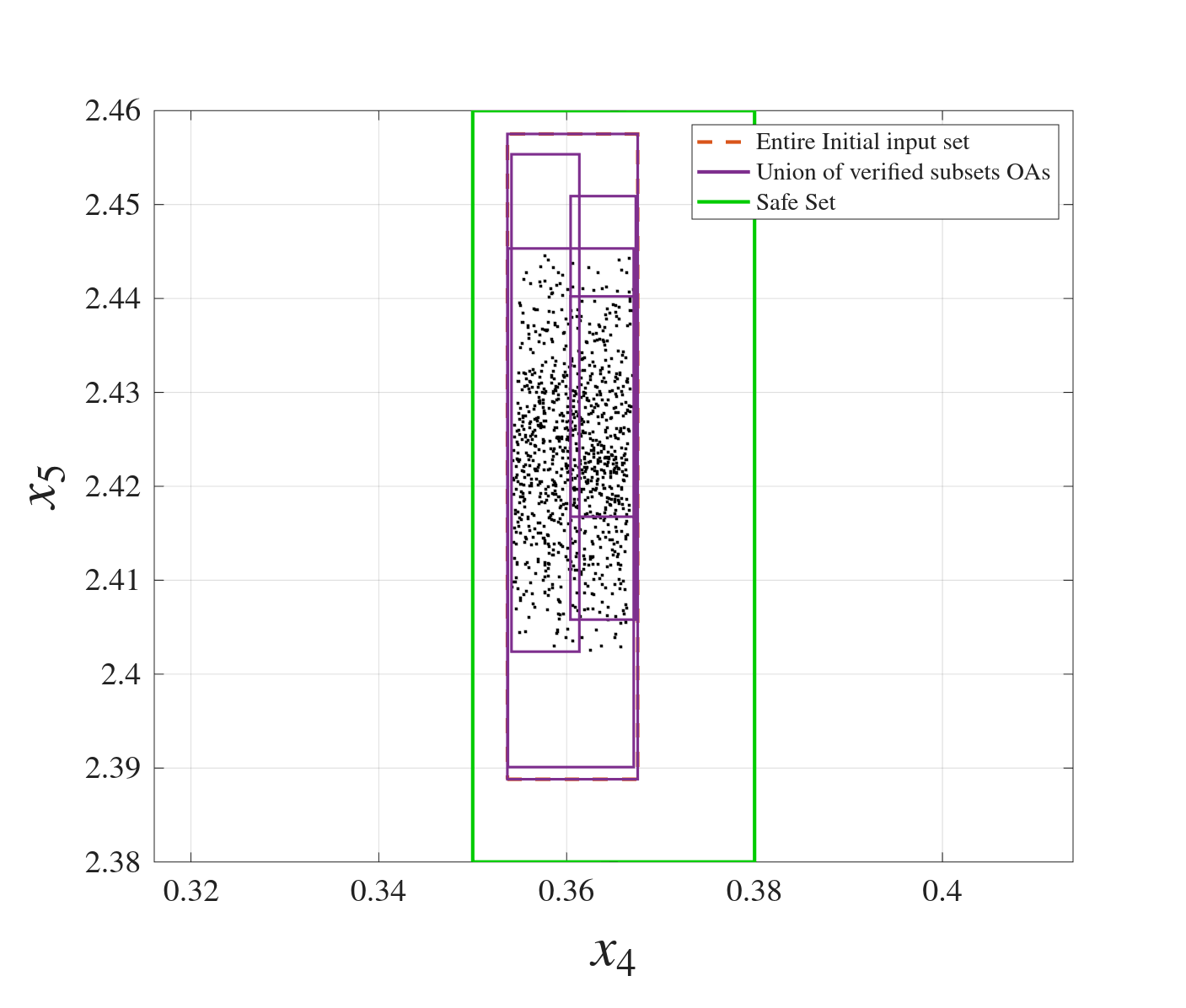}} \\
\multicolumn{2}{c}{(c) $x_4 - x_5$}
\end{tabular}
\caption{TNODEV verification results on the FPA~5D benchmark using naive refinement. In all sub figures, the dashed \textcolor{red}{red} rectangle is the over-approximation $\Omega(\Xin)$ from a single reachability call on the unsplit initial set, the \textcolor{Green}{green} rectangle is the safe set $\mathcal{X}_s$, the \textcolor{violet}{violet} rectangles represent the verified subsets, and the black dots are sampled successor trajectories.}
\label{fig:FPA_naive_results}
\end{figure}

\begin{figure}[H]
\centering
\begin{tabular}{@{}c@{\hspace{0.5em}}c@{}}
\includegraphics[width=0.5\textwidth]{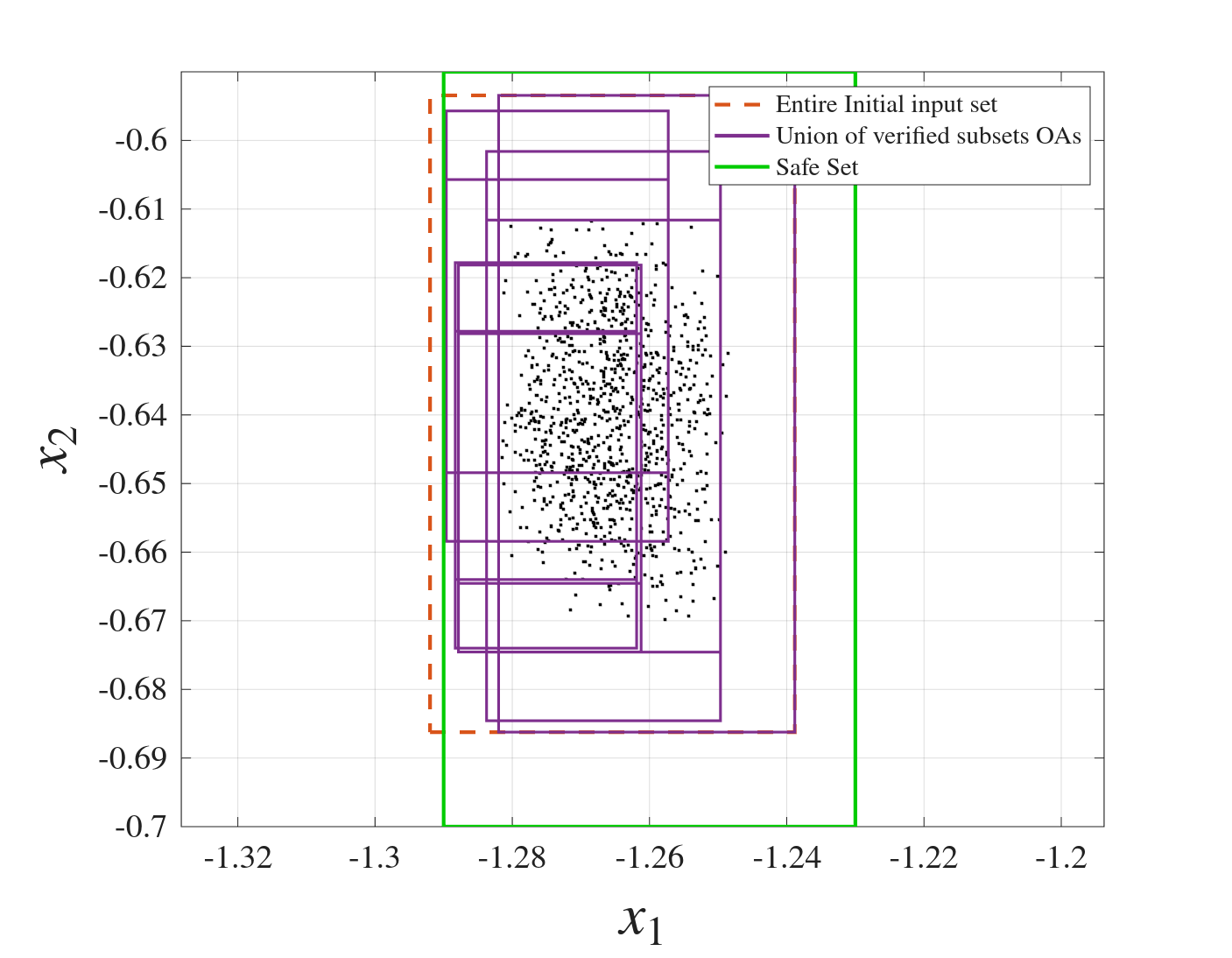} &
\includegraphics[width=0.48\textwidth]{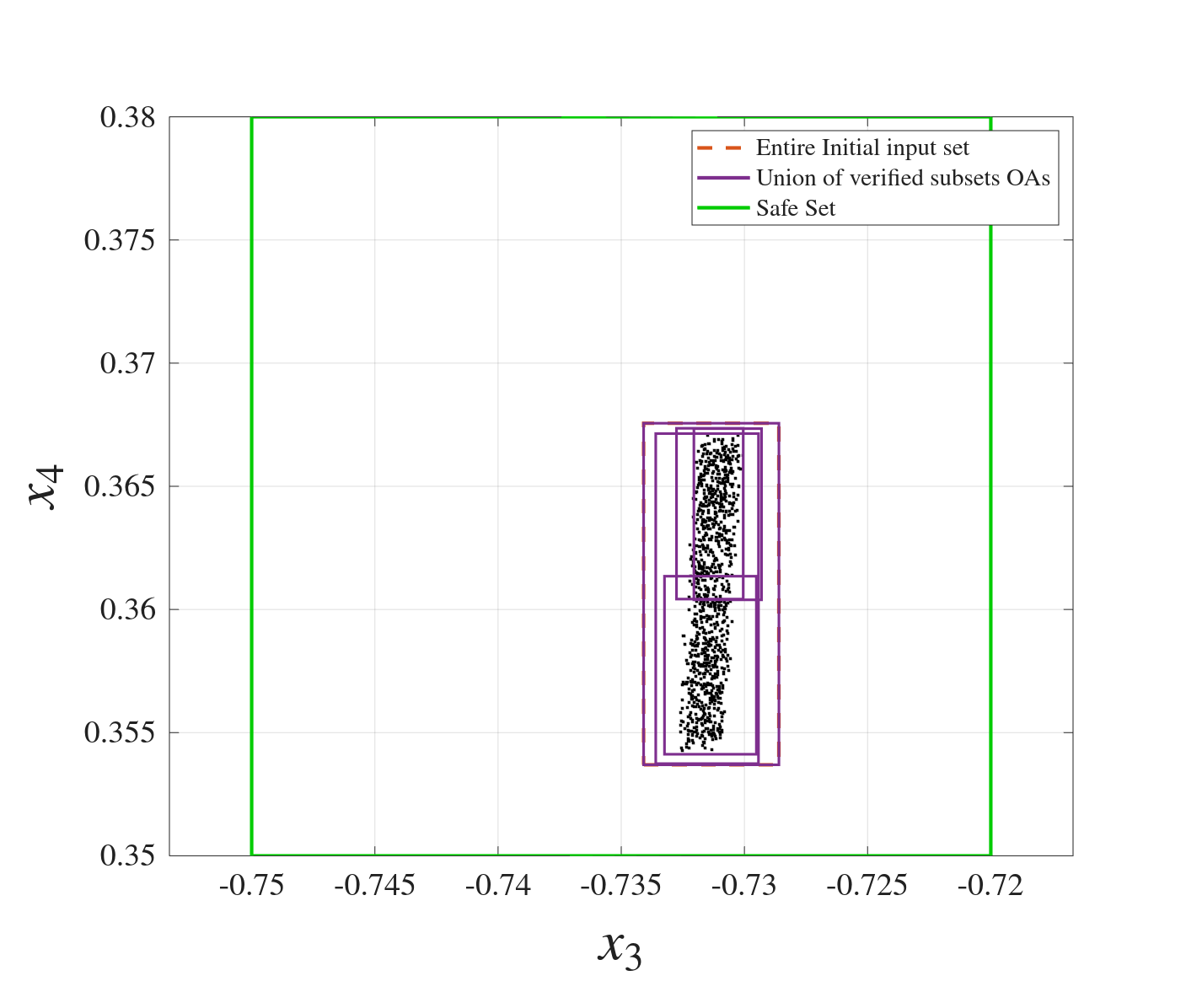} \\
(a) $x_1 - x_2$ & (b) $x_3 - x_4$ \\[0.8em]
\multicolumn{2}{c}{\includegraphics[width=0.5\textwidth]{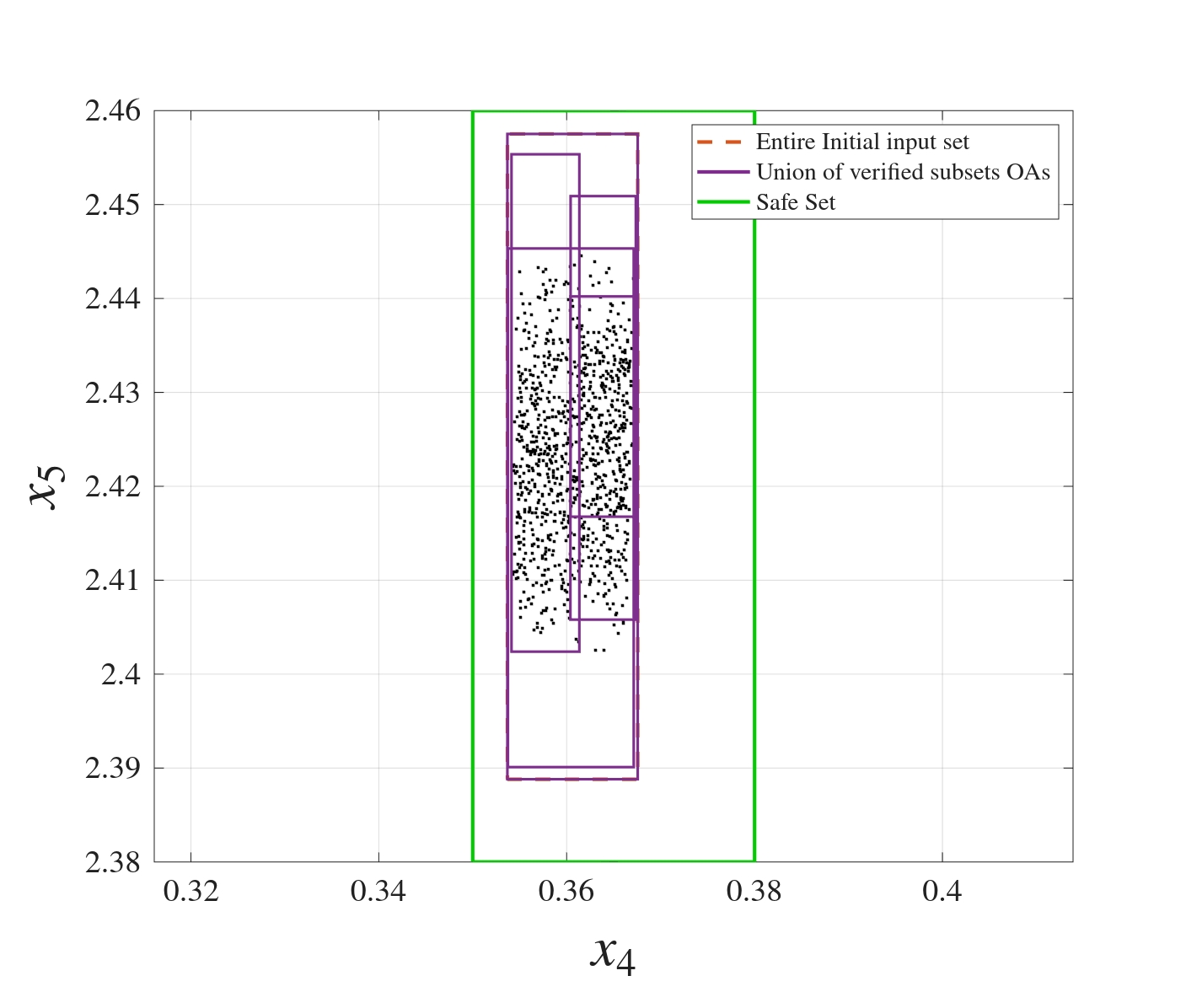}} \\
\multicolumn{2}{c}{(c) $x_4 - x_5$}
\end{tabular}
\caption{TNODEV verification results on the FPA~5D benchmark using MSIR refinement. In all sub figures, the dashed \textcolor{red}{red} rectangle is the over-approximation $\Omega(\Xin)$ from a single reachability call on the unsplit initial set, the \textcolor{Green}{green} rectangle is the safe set $\mathcal{X}_s$, the \textcolor{violet}{violet} rectangles represent the verified subsets, and the black dots are sampled successor trajectories.}
\label{fig:FPA_MSIR_results}
\end{figure}

\begin{figure}[H]
\centering
\begin{tabular}{@{}c@{\hspace{0.5em}}c@{}}
\includegraphics[width=0.5\textwidth]{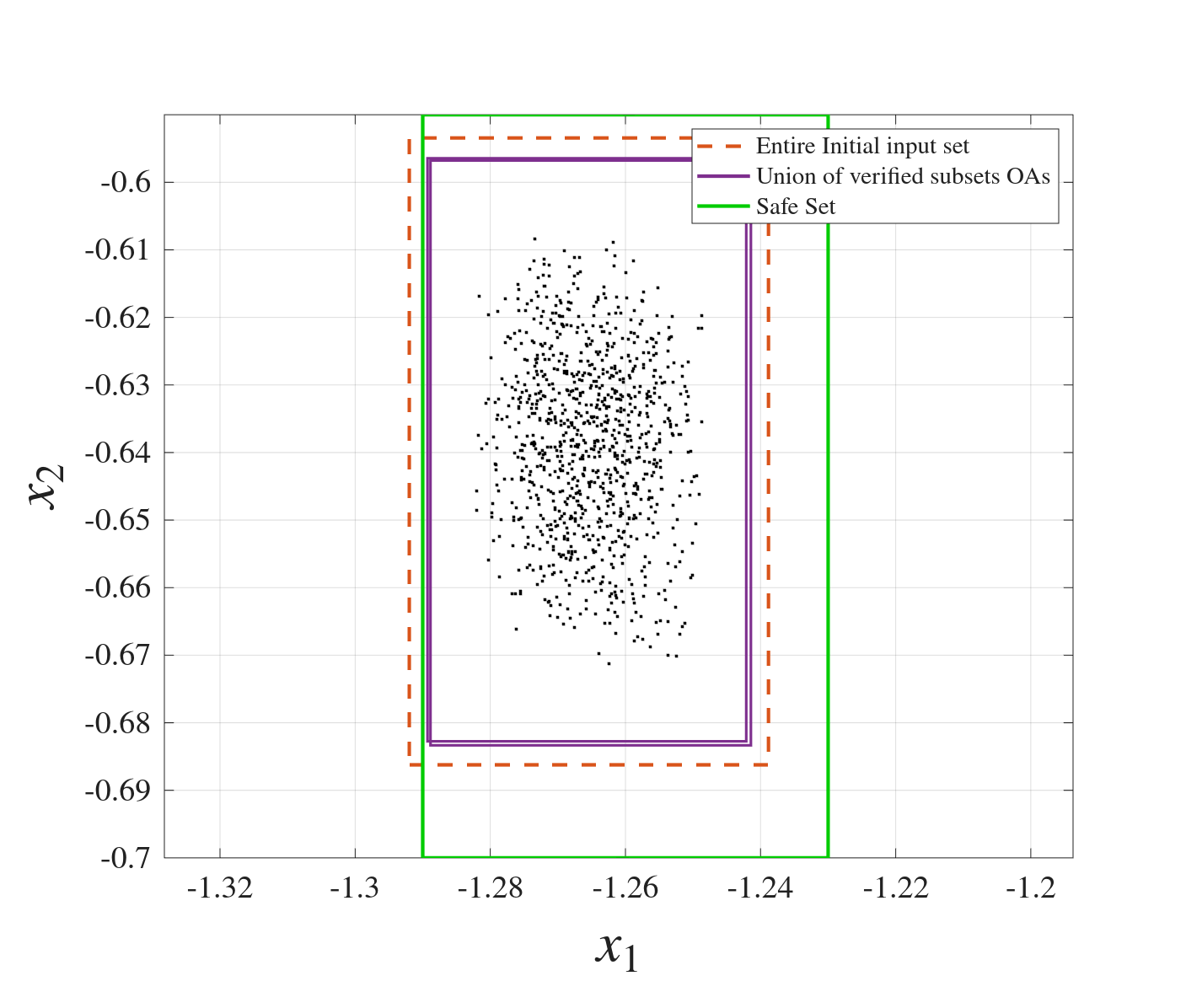} &
\includegraphics[width=0.5\textwidth]{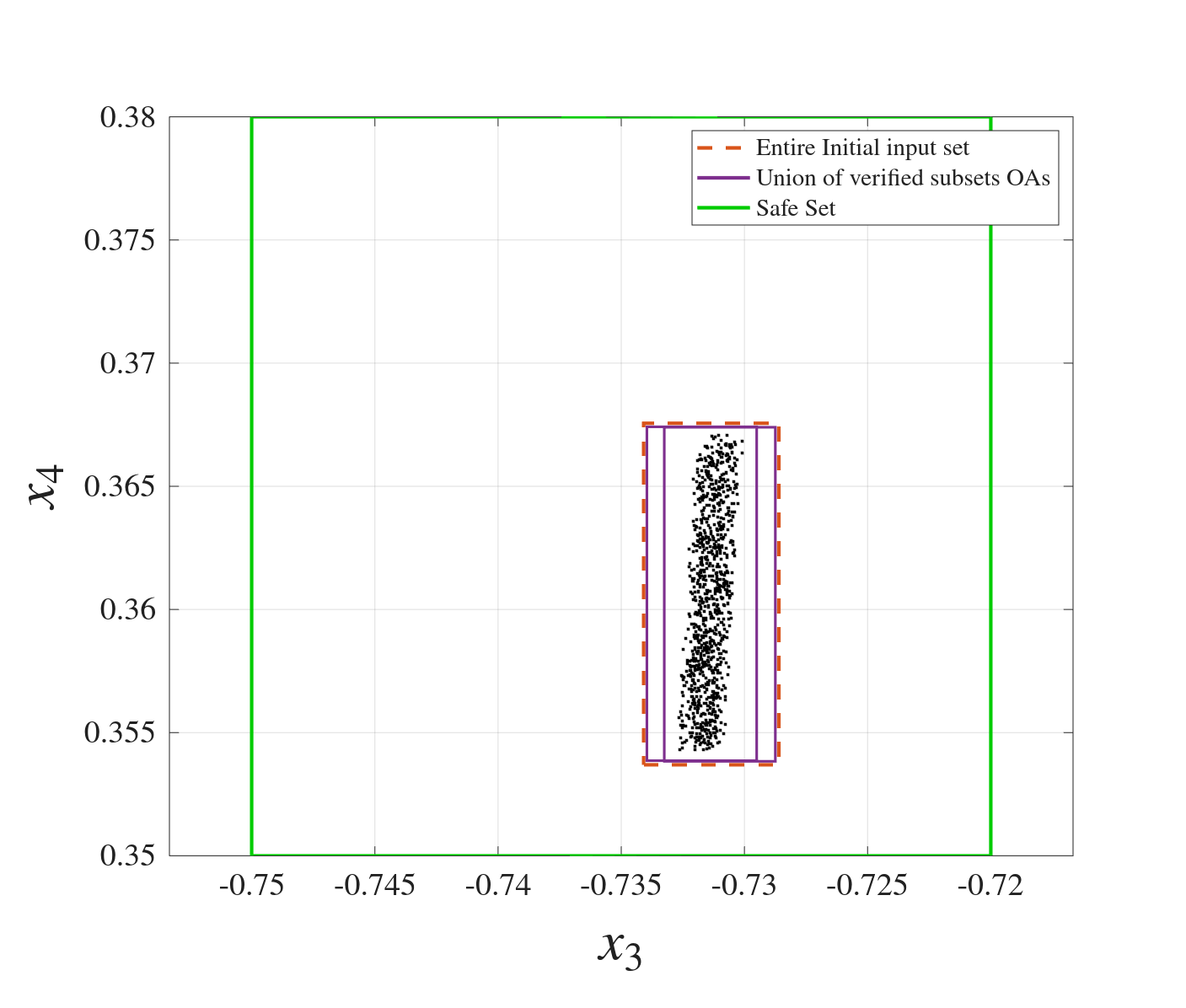} \\
(a) $x_1 - x_2$ & (b) $x_3 - x_4$ \\[0.8em]
\multicolumn{2}{c}{\includegraphics[width=0.5\textwidth]{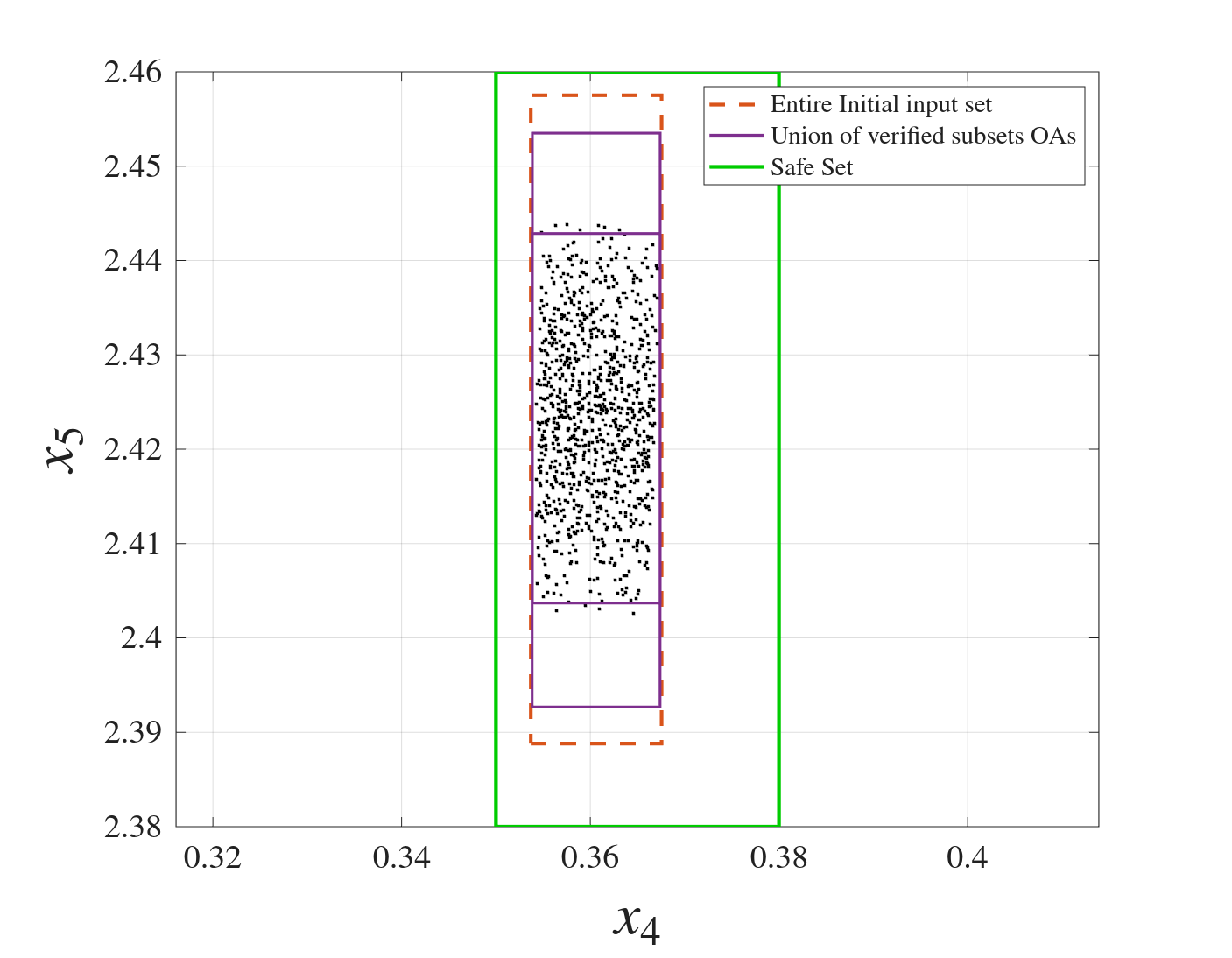}} \\
\multicolumn{2}{c}{(c) $x_4 - x_5$}
\end{tabular}
\caption{TNODEV verification results on the FPA~5D benchmark using ING refinement. In all sub figures, the dashed \textcolor{red}{red} rectangle is the over-approximation $\Omega(\Xin)$ from a single reachability call on the unsplit initial set, the \textcolor{Green}{green} rectangle is the safe set $\mathcal{X}_s$, the \textcolor{violet}{violet} rectangles represent the verified subsets, and the black dots are sampled successor trajectories.}
\label{fig:FPA_ING_results}
\end{figure}

Figures \ref{fig:FPA_naive_results}--\ref{fig:FPA_ING_results} show the verified subsets for the naive, MSIR, and ING heuristics, projected onto the three planes $x_1$--$x_2$, $x_3$--$x_4$, and $x_4$--$x_5$. All three heuristics certify the specification \texttt{SAFE}.

\paragraph{Verification-level comparison against a verifier without refinement.} For the FPA specification stated and used in Table~\ref{tab:results}, we also run NNV 2.0 on the entire initial set $\Xin$, i.e. one Star set reachability without any refinements, followed by the same safe set check,  with the per-coordinate output ranges extracted from the Star set by linear programming, which are the tightest sound bounds it supports. The resulting ranges are not contained in $\mathcal{X}_s$, so the single call is inconclusive in $0.69$~s, with a violation margin $\eta = 7.7 \cdot 10^{-3}$, where $\eta$ refers to the largest amount by which the over-approximation extends outside $\mathcal{X}_s$ in any coordinate. A single CTMM call of TNODEV on the same entire initial set $\Xin$ is inconclusive as well, with $\eta = 2.0 \cdot 10^{-3}$, so its interval enclosure is tighter here than the one extracted from the Star set, but still it is not tight enough to conclude. TNODEV with the MSIR heuristic then certifies the specification \texttt{SAFE} after $9$ refinement iterations, producing $5$ verified subsets. What separates the two inconclusive single-call outcomes from the \texttt{SAFE} verdict is therefore the \emph{refinement loop}, rather than the tightness of any single reachability call. Figure~\ref{fig:nnv_vs_tnodev} illustrates the output results from both cases.

\begin{figure}[H]
\centering
\begin{tabular}{@{}c@{\hspace{0.5em}}c@{}}
\includegraphics[width=0.45\textwidth]{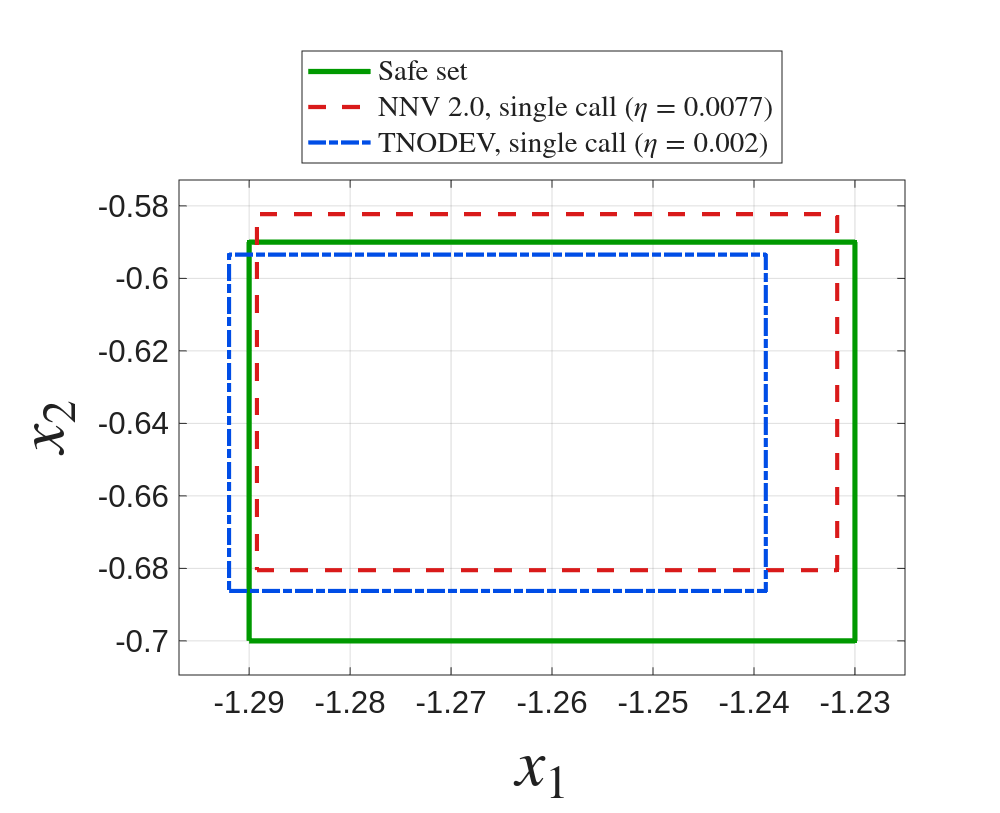} &
\includegraphics[width=0.45\textwidth]{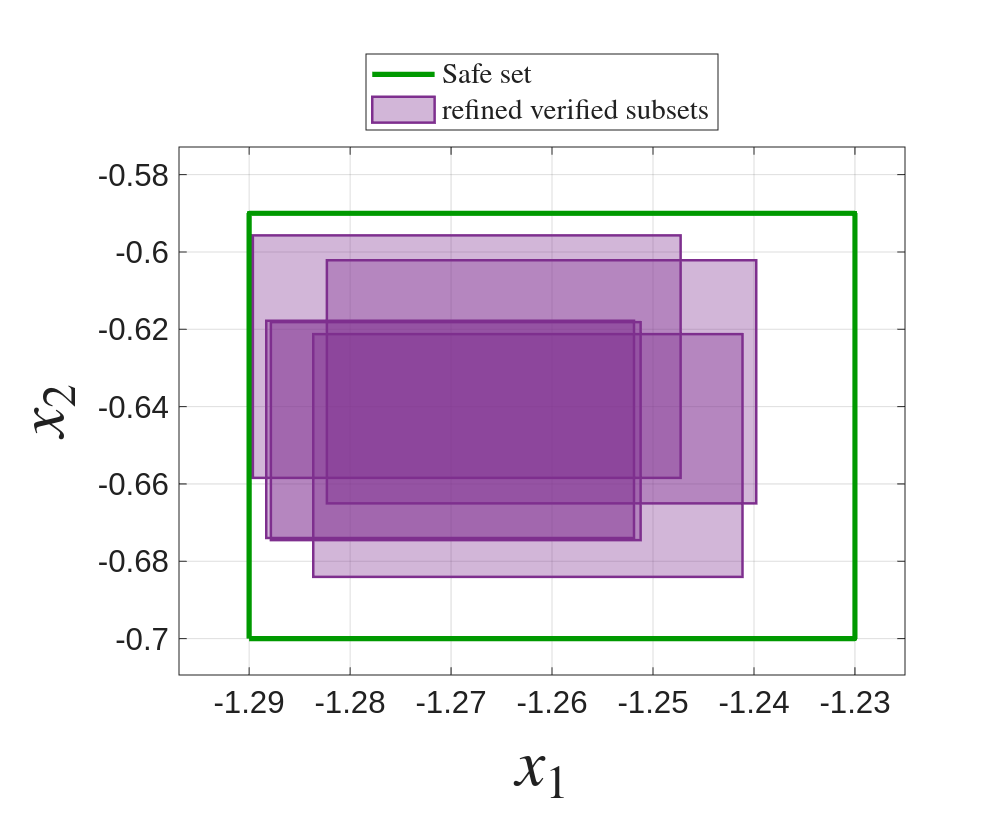} \\
(a) single call, $x_1$--$x_2$ & (b) with refinement, $x_1$--$x_2$ \\[0.6em]
\includegraphics[width=0.45\textwidth]{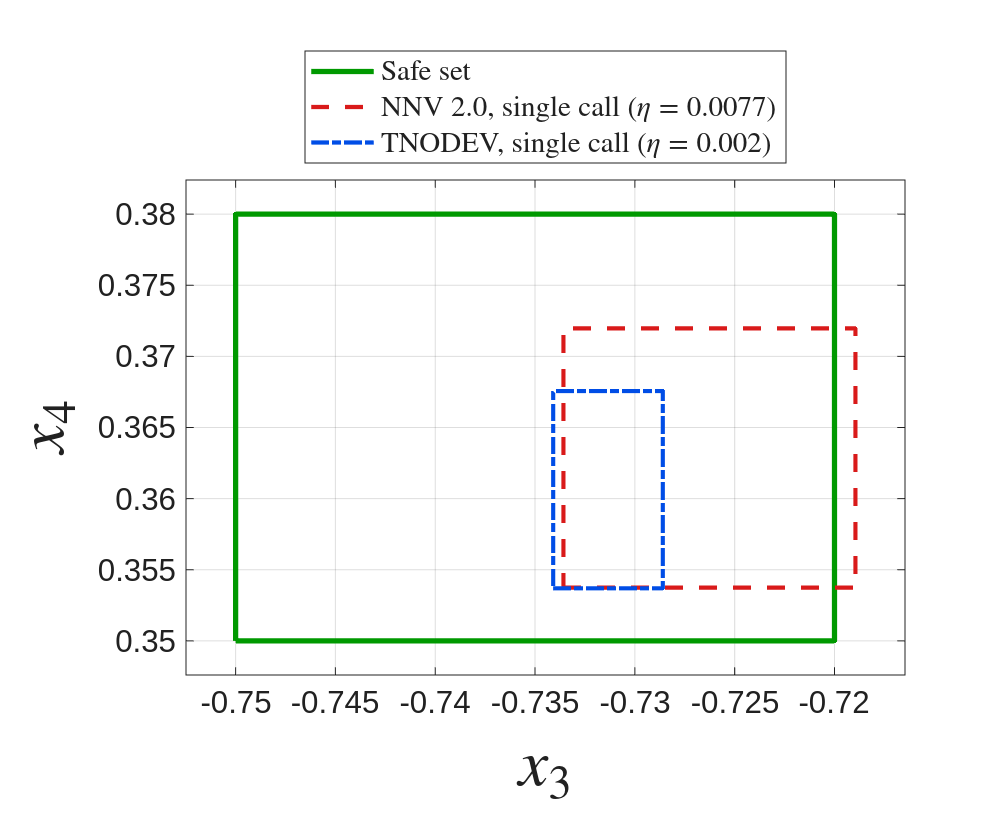} &
\includegraphics[width=0.45\textwidth]{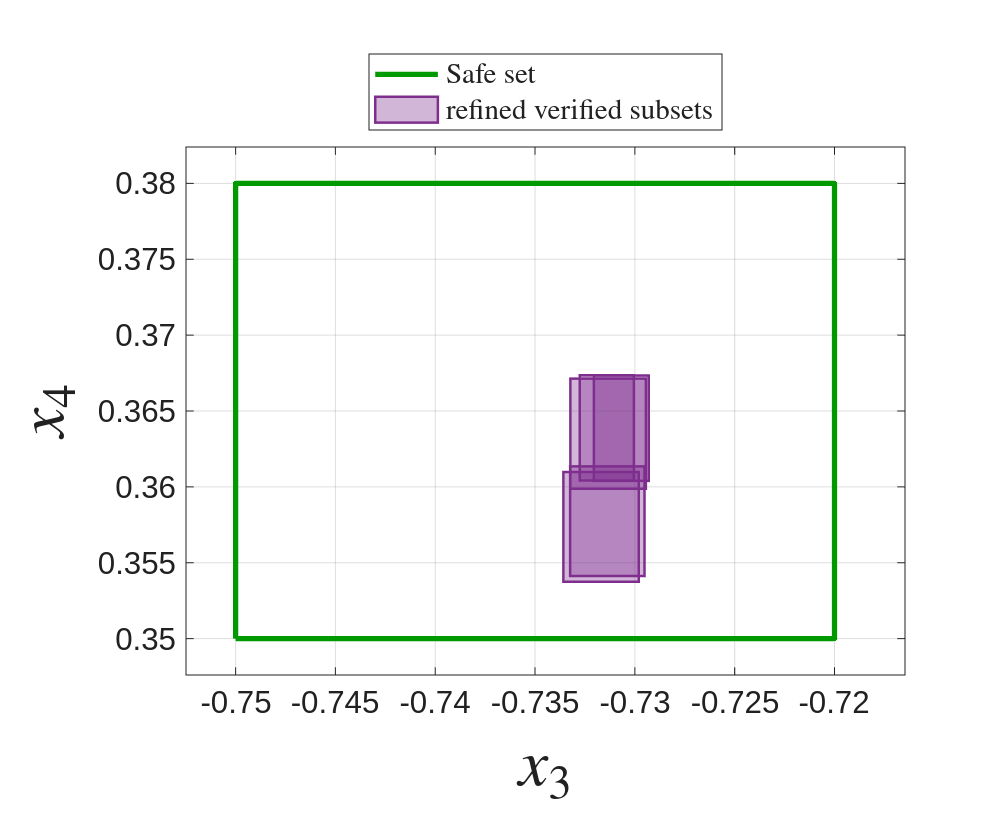} \\
(c) single call, $x_3$--$x_4$ & (d) with refinement, $x_3$--$x_4$ \\[0.6em]
\includegraphics[width=0.45\textwidth]{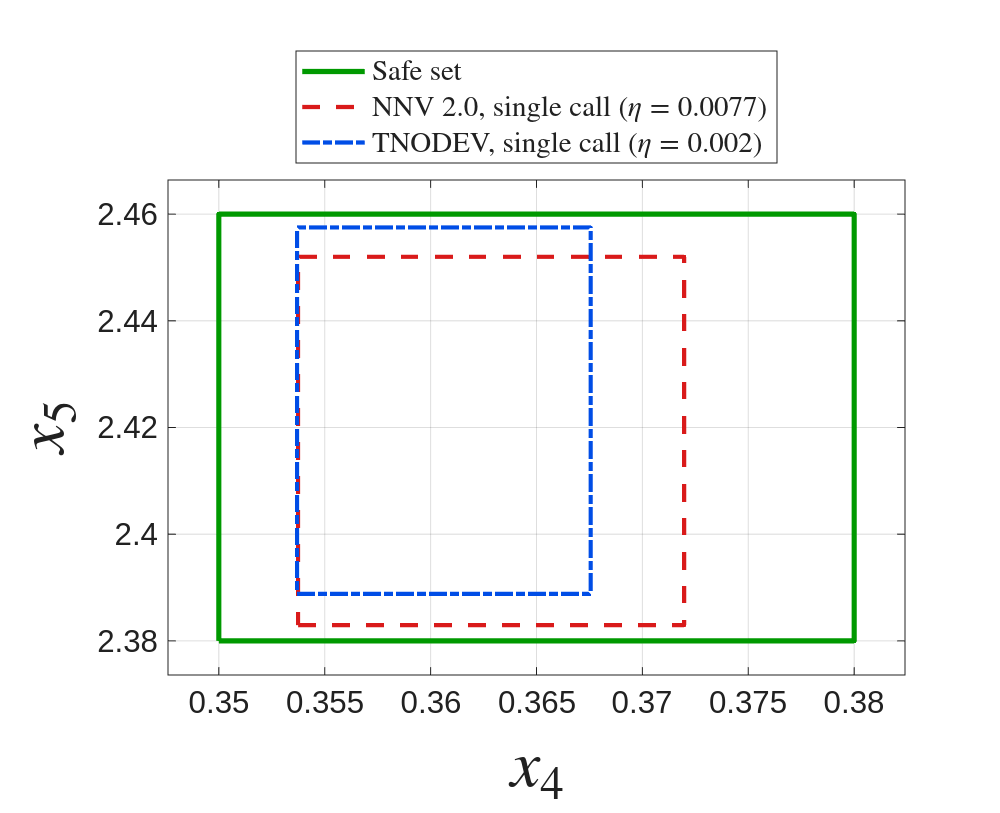} &
\includegraphics[width=0.45\textwidth]{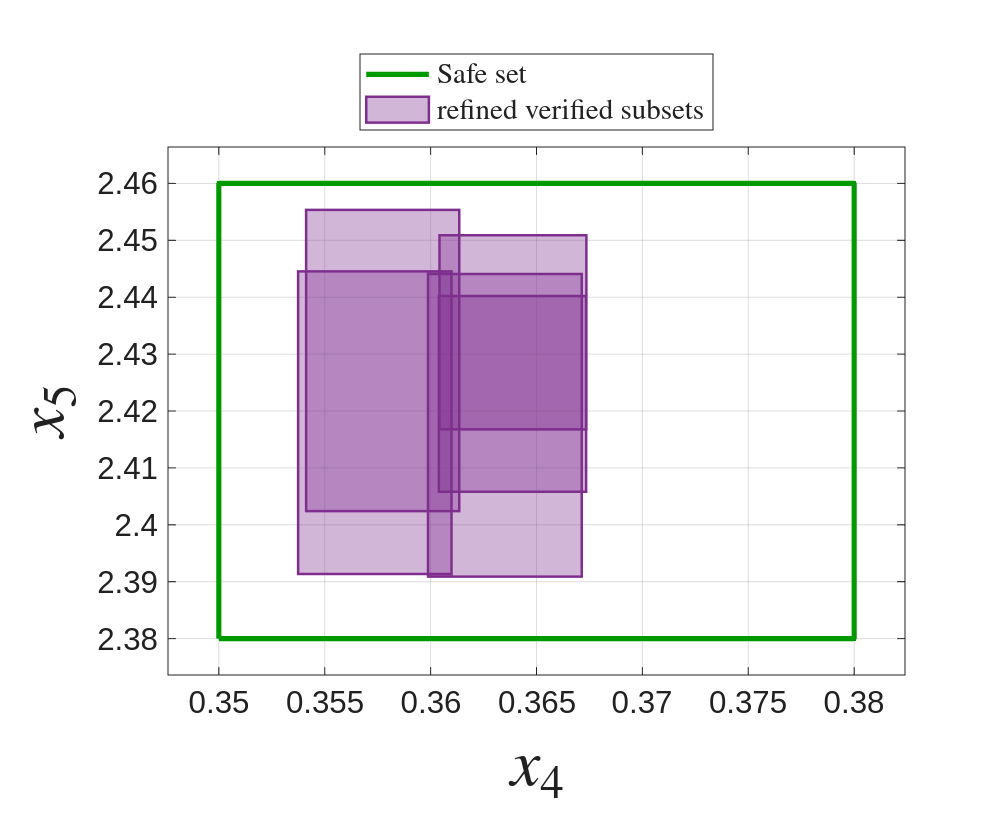} \\
(e) single call, $x_4$--$x_5$ & (f) with refinement, $x_4$--$x_5$
\end{tabular}
\caption{Verification-level comparison on FPA. In every sub figure the solid \textcolor{Green}{green} rectangle is the safe set $\mathcal{X}_s$. In the left column no single call concludes, as the dashed \textcolor{red}{red} rectangle of NNV~2.0 and the dash-dotted \textcolor{blue}{blue} rectangle of TNODEV CTMM call on $\Xin$ both extend outside $\mathcal{X}_s$. In the right column, the \textcolor{violet}{violet} filled rectangles are the $5$ verified subsets produced by TNODEV with MSIR heuristic, each contained in $\mathcal{X}_s$, so the specification is certified \texttt{SAFE}.}
\label{fig:nnv_vs_tnodev}
\end{figure}

\subsection{Cartpole}\label{apd:Cartpole_spec}

The safety specification is:


\begin{itemize}
\item \textit{Initial set:} $\Xin = [x_0 - 10^{-4},\, x_0 + 10^{-4}]$ with center $x_0 = (0,\, 0,\, 0.001,\, 0,\, 0,\, 0,\, 0,\, 0,\, 0,\, 0,\, 0,\, 0)$, i.e., a small initial pole angle $\theta_0 = 0.001$ with all other states at zero.
\item \textit{Safe set:} $\mathcal{X}_s = [-0.015, 0.015] \times [-0.110, 0.060] \times [-0.015, 0.020] \times [-0.050, 0.160] \times [-0.2, 0.2]^{8}$, bounding the cart position, cart velocity, pole angle, pole angular velocity, and the eight hidden states.
\end{itemize}
The model and initial set follow \citet{gruenbacher2022gotube}, the safe set is defined by us, as Cartpole is published as a reachability benchmark without an associated property (see Appendix~\ref{apd:System_description}).

\subsection{ACC}\label{apd:ACC_spec}

The safety specification is:
\begin{itemize}
\item \textit{Initial set:} $x_{\text{lead}} \in [90, 110]$, $v_{\text{lead}} \in [31, 33]$, $x_{\text{ego}} \in [10, 11]$, $v_{\text{ego}} \in [29, 31]$, with the remaining plant states initialized as in \citet{manzanas2022nnvode}.
\item \textit{Safe set:} $D_{\text{rel}} \ge D_{\text{safe}}$ with $D_{\text{default}} = 10$~m, and $t_{\text{gap}} = 1.4$~s, evaluated over the $5$~s horizon. Concretely, the specification is checked at each of the $50$ control steps of the discrete closed loop (one check per $0.1$~s control period), so a \texttt{SAFE} verdict certifies the safe-distance property at every sampling instant of the $5$~s horizon.
\end{itemize}
The safe set specification is taken directly from the original ACC benchmark of \citet{tran2020nnv}. The velocity ranges of the initial set are widened relative to the nominal benchmark initialization in order to exercise the verification and refinement loop, as under the nominal narrow initial set, the specification is verified by a single reachability call without triggering any refinements.

\section{Cost analysis of refinement heuristics}
\label{app:refinement_cost}

This appendix details the computational cost of the three refinement heuristics introduced in Section~\ref{subsec:verif_loop}. We measure cost in \emph{reachability calls per refinement decision}, which is the dominant operation in the verifier, since each reachability call integrates a $2n$-dimensional embedded ODE over the specified time horizon and is orders of magnitude more expensive than any other per-cell operation.

\paragraph{Baseline.} A refinement decision is reached only after a failed verification check on the current cell $\mathcal{X}$. By the time refinement runs, the verifier has already spent one reachability call to obtain $\Omega(\mathcal{X})$. The cost statements below count only \emph{additional} calls beyond this baseline.

\paragraph{Naive.} Picks $i^\star = \arg\max_i r_i$, using only the cell's radius vector $r$. The radius is a property of the cell, known without any computation. Cost: no additional reachability calls and no other meaningful computation.

\paragraph{MSIR.} Computes $s_i = r_i \cdot \max_k \max(|\underline{J}_{ki}|, |\overline{J}_{ki}|)$, which requires the radius (already known) and interval Jacobian bounds $[\underline{J}, \overline{J}]$ over the cell. The Jacobian bounds are obtained by a single symbolic pass through the network's analytic Jacobian using interval arithmetic, which is a different operation from a reachability call. To make the contrast more clear: a reachability call evaluates the network's forward pass thousands of times during the ODE integration, whereas computing the interval Jacobian requires only a single interval-bound pass through the network's layers. Cost: no additional reachability calls, and one Jacobian-bounding call per decision.

\paragraph{ING.} Estimates $g_j \approx \partial w / \partial x_j$ by finite differences. Computing $g_j$ requires $w(\mathcal{X}_{j,\delta})$, the over-approximation width of the cell perturbed by $\delta$ along dimension $j$. Since the reachability call evaluates complex, non-linear ODE dynamics numerically over a time interval, the resulting over-approximation lacks a closed-form analytic expression. Therefore, its exact derivative cannot be calculated directly, and $w(\mathcal{X}_{j,\delta})$ can only be obtained by running the full reachability call on the perturbed cell. One such call is needed per input dimension. Cost: $n$ additional reachability calls per cell being refined.

\paragraph{Summary.} Table~\ref{tab:heuristic_cost} summarizes the additional computational overhead required by each heuristic whenever a cell needs to be split.

\begin{table}[htbp]
\centering
\begin{tabular}{lcc}
\textbf{Heuristic} & \textbf{Extra reach calls} & \textbf{Other cost} \\
\hline
Naive & $0$ & --- \\
MSIR  & $0$ & 1 Jacobian-bounding call \\
ING   & $n$ & --- \\
\end{tabular}
\caption{Cost of each refinement heuristic per refinement decision, in addition to the one reachability call already spent on the failed verification check.}
\label{tab:heuristic_cost}
\end{table}

\paragraph{Practical implications.} Reachability call cost dominates the total verification time, so ING $n$-fold overhead per decision accumulates rapidly across all refinement iterations. Furthermore, our empirical results in Table~\ref{tab:results} show that ING performance is highly inconsistent. Because ING relies on local finite-difference estimates, and it can sometimes find highly efficient splits (as seen in the FPA benchmark). However, this localized view can also lead to poor splitting choices overall, causing the verifier to time out. The effect is highly benchmark-dependent:

\begin{itemize}
\item \textbf{Spiral linear ($n = 2$).} Even at low dimensions, ING completely fails to converge for the case of Spiral linear, after $5000$ iterations ($13.16$ min) with an \texttt{UNKNOWN} verdict. In contrast, Naive and MSIR easily verify the system in $31$ and $29$ iterations, in $6.79$ and $6.5$ seconds. The reason is not that ING makes locally poor choices on this benchmark, but that on a \emph{linear} ODE the quantity ING estimates (the finite-difference sensitivity $g_j$ of the enclosure width $w(\mathcal{X}) = \sum_k (\overline{x}_k - \underline{x}_k)$ to a shift of the cell centre along input dimension $j$) is identically zero, as we prove in Proposition~\ref{prop:ing-degenerate} below, so the heuristic is selecting a split dimension from numerical noise.


\item \textbf{FPA ($n = 5$).} This benchmark demonstrates ING's theoretical appeal, as ING achieves the lowest number of required iterations (only $3$, compared to $17$ for Naive and $9$ for MSIR). However, because each of those decisions costs $5$ additional reachability calls, ING's total wall-clock time ($53.78$ s) remains higher than MSIR ($34.50$ s), and it is the only benchmark of the six on which ING terminates at all.

\item \textbf{Cartpole ($n = 12$).} MSIR shows its clear scalability advantage here, drastically reducing iterations compared to Naive ($63$ vs. $2037$) and cutting the total verification time from $2.73$ min down to $56.43$ s. Meanwhile, ING's compounding $12\times$ reachacility calls overhead, combined with poor local splits, causes it to reach the refinement iteration limit without verifying a single subset.

\item \textbf{ACC ($n = 8$).} Both plant variants are verified \texttt{SAFE} by naive and MSIR, but ING reaches the wall-clock timeout on both without verifying any subset, returning \texttt{UNKNOWN}. This reinforces the pattern that ING's $n$-fold per-decision overhead makes it the least reliable heuristic on the higher-dimensional benchmarks. On the linear variant MSIR is faster than naive ($2.72$ vs $4.78$ min), while on the nonlinear variant MSIR processes fewer iterations than naive ($37$ vs $41$) but does not reduce wall-clock time ($29.61$ vs $20.03$ min), one of the few cases where MSIR iterations reduction do not translate into a time saving.

\end{itemize}

\paragraph{Why ING underperforms.} All four \texttt{UNKNOWN} verdicts of Table~\ref{tab:results} are produced by ING, which converges on one benchmark out of six. Part of the explanation is cost related, as it spends $n$ additional reachability calls, one per input dimension at every refinement decision, and that cost compounded over every refinement iteration of the run, which is what rules ING out on Cartpole ($n=12$) and both ACC variants ($n=8$). But cost cannot be the whole story, since ING also fails on the two-dimensional Spiral linear benchmark, where the overhead is a factor of two. There the cause is that the quantity ING estimates (the sensitivity $g_j$ of the enclosure width $w(\mathcal{X})$ to a shift of the cell centre along dimension $j$) carries no usable signal, which for linear dynamics we can establish exactly.

\begin{proposition}[ING sensitivity estimates $g_j$ vanish on linear dynamics]\label{prop:ing-degenerate}
Let $f(x) = Ax + b$ be linear, so that the Jacobian is the constant matrix $A$ and the decomposition function of Step~2 of Section~\ref{subsec:ctmm} takes the form $g(x, \hat{x}) = A_+ x + A_- \hat{x} + b$ with $A_+ = \max(A, 0)$ and $A_- = \min(A, 0)$ element-wise. Then the width $w(\mathcal{X}) = \sum_k (\overline{x}_k - \underline{x}_k)$ of the over-approximation $\Omega(\mathcal{X})$ depends only on the radius $r$ of the cell $\mathcal{X} = [c-r, c+r]$ and not on its centre $c$. Consequently the finite-difference estimates of the ING heuristic satisfy $g_j = 0$ for every $j$ and every $\delta$.
\end{proposition}

\begin{proof}
Let $(x(t), \hat{x}(t))$ solve the embedded system \eqref{eq:embedded} from $(c - r, c + r)$ and let $e(t) = \hat{x}(t) - x(t)$ be the width vector of the embedded state, i.e. the coordinate-wise extent of the over-approximation at time $t$. Then
\[
  \dot{e} \;=\; \bigl(A_+ \hat{x} + A_- x + b\bigr) - \bigl(A_+ x + A_- \hat{x} + b\bigr) \;=\; (A_+ - A_-) \, e ,
\]
a homogeneous linear ODE in the width vector $e$ alone and its right-hand side contains no term in $x$ or $\hat{x}$ separately, so the position of the cell center never enters the evolution of $e$, whose solution $e(t_f) = \exp\bigl((A_+ - A_-) t_f\bigr) \, 2r$ depends on $r$ only. Since $w(\mathcal{X}) = \mathbf{1}^\top e(t_f)$ the width $w(\mathcal{X})$ of the enclosure $\Omega(\mathcal{X})$ is a function of $r$ alone. The ING perturbation $\mathcal{X}_{j,\delta}$ shifts the center by $\delta$ along $j$ while keeping $r$ fixed, so $w(\mathcal{X}_{j,\delta}) = w(\mathcal{X})$ and $g_j = 0$.
\end{proof}

The consequence is that on a linear neural ODE the ING score would be exactly zero if it were computed without numerical error. What the implementation actually computes, however, is the difference of two over-approximations returned by the adaptive solver, so the computed score consists of nothing but integration noise, of the order of the solver tolerance. The split rule $i^\star = \arg\max_j |g_j|$ of Section~\ref{subsec:verif_loop} then selects a coordinate essentially at random, and it may keep selecting the same coordinate, as nothing forces the radius along the coordinate that actually binds the specification check to shrink,  so the fairness condition of Definition~\ref{def:fair} can be violated forever and the hypothesis of Proposition~\ref{prop:completeness} fails. This is precisely what we observe on $5000$ iterations producing $242$ verified subsets on Spiral linear, against $29$ iterations and $15$ subsets for MSIR.

ING is informative only when the enclosure width varies with the cell center by more than the integration noise, and Proposition~\ref{prop:ing-degenerate} identifies the linear dynamics family where that variation is exactly zero. Nothing guarantees that it is large elsewhere. It is governed by how much the Jacobian varies across the cell, since the width variation that ING measures comes only from the Jacobian changing between the original and the shifted cell. As refinement makes the cells smaller, any smooth vector field becomes close to linear over each individual cell, its Jacobian is nearly constant there, and the width variation can fall below the integration noise. In that regime ING selects its split dimension from noise exactly as it does on linear dynamics, even though the score is no longer exactly zero.

\paragraph{Takeaway.} We recommend MSIR as the \textbf{default} heuristic of TNODEV for the benchmarks evaluated, based on the following evidence and with the following exception. MSIR never requires more iterations than naive on any of the five benchmarks on which refinement is exercised (the Spiral nonlinear instance of Table~\ref{tab:results} is falsified before the loop starts), and its per-decision overhead (one interval Jacobian pass) is negligible compared to a reachability call, so on the evaluated benchmarks it is never substantially worse than naive. The first part of this statement is an observation rather than a property of the heuristic, since MSIR is not fair in the sense of Definition~\ref{def:fair}, and nothing prevents it from requiring more refinement iterations than naive on other dynamics or specifications. The iteration reduction translates into wall-clock gains on four benchmarks, decisively on Cartpole~12D ($63$ against $2037$ iterations, $56.43$ s against $2.73$ min) and clearly on ACC 8D linear ($2.72$ against $4.78$ min). The exception is ACC 8D nonlinear, where MSIR performs fewer iterations than naive ($37$ against $41$) but is $47\%$ slower in wall-clock ($29.61$ against $20.03$~min). We attribute this to the cells MSIR produces being individually more expensive to integrate, as splitting along the most sensitive direction can leave sub-cells whose over-approximation remains close to the safe-set boundary, so the embedded system is integrated over a stiffer regime. This is a hypothesis rather than a measured finding and it is consistent with the observed split between wall-clock time and iteration count.

\paragraph{Scope of the MSIR recommendation.} The recommendation of MSIR is only as general as the evidence on which it rests on, namely six benchmarks from two model classes (pure neural ODE and hybrid neural ODE control systems) with input dimensions between $2$ and $12$, the CTMM reachability backend, and one specification to be verified per benchmark. It does not establish that MSIR is the best out of the three heuristics for every neural ODE, and the ACC nonlinear exception already shows that its iteration savings do not always translate into wall-clock time gains. On a new class of problems, it therefore remains advisable to compare MSIR with the naive heuristic, which costs nothing per refinement decision and is the only one of the three heuristics that is fair by construction (Appendix~\ref{apd:completeness}).

\section{Soundness and completeness of TNODEV}\label{apd:soundness_scope}

This appendix answers a single question in two parts, namely what a verdict of TNODEV means. The first part states what a \texttt{SAFE} verdict guarantees and the four statements (D1)--(D4) on which it rests . The second part, Appendix~\ref{apd:completeness}, asks the complementary question of whether the refinement loop, given as many iterations as it needs, can always replace an \texttt{UNKNOWN} verdict by a conclusive one.

\subsection{Scope of the soundness guarantee}\label{apd:TNODEV_soundness}

This part states precisely what a \texttt{SAFE} verdict of TNODEV guarantees, and on which assumptions it rests. We do this because the soundness results of Section~\ref{sect:verif_arch} and Appendix~\ref{apd:classification_nODE_Hybrid_reachability_pipeline} are statements about exact mathematical quantities, whereas the implementation manipulates their numerical approximations, and the gap between the two should not be left to the reader to reconstruct.

\paragraph{The chain of statements.} A \texttt{SAFE} verdict produced by the CTMM backend rests on the following four statements, in the order in which they are executed. They are not all of the same nature: (D1) is a result obtained from an external tool, (D2) and (D4) hold exactly in the absence of numerical error, and (D3) is a theorem.

\begin{enumerate}
\item[(D1)] \emph{Reachable tube (Step 1).} The interval $\Omega^{\text{tube}}(\Xin)$ returned by CORA's \texttt{nonlinearSys} backend satisfies $\Omega^{\text{tube}}(\Xin) \supseteq \Reach^{\text{tube}}(\Xin)$. CORA computes this over-approximation of the reachable tube by abstracting the nonlinear dynamics and bounding the resulting \emph{abstraction error} by a set $\mathcal{L}$ such that $f(x,u) \in f_{\text{abstract}}(x,u) \oplus \mathcal{L}$ \citep[Sec.~4.2.6]{althoff2015cora}. The over-approximation of (D1) is therefore established by construction, so (D1) is subject to the same floating-point caveat as (D2) and (D4).

\item[(D2)] \emph{Jacobian bounds (Step 1).} The element-wise bounds $[\underline{J}_x, \overline{J}_x]$ obtained by interval propagation through the analytic Jacobian of the network satisfy $J_x(x) \in [\underline{J}_x, \overline{J}_x]$ for all $x \in \Omega^{\text{tube}}(\Xin)$. This statement is exact in exact real arithmetic, i.e. if the interval propagation is carried out on real numbers with no floating-point error. In floating point this is no longer certain, because a rigorous implementation would round every interval endpoint outwards, so that rounding can only ever widen the bounds, whereas ours uses ordinary floating-point arithmetic, in which rounding could in principle shrink an endpoint by a few units in the last place.

\item[(D3)] \emph{Decomposition function (Step 2).} The constructed $g$ is a valid decomposition function of $f$, so that the $2n$-dimensional embedded system \eqref{eq:embedded} brackets the true solution, in the sense that the exact solution $(x(\cdot), \hat{x}(\cdot))$ of \eqref{eq:embedded} satisfies $\Reach(\Xin) \subseteq [x(t_f), \hat{x}(t_f)]$. This is a theorem of \citet{meyer2021interval} and \citet{coogan2020mixed}, and holds unconditionally given (D2).

\item[(D4)] \emph{Numerical integration (Step 3).} The implementation does not compute the exact pair $(x(t_f), \hat{x}(t_f))$ but a numerical approximation $(\tilde{x}(t_f), \tilde{\hat{x}}(t_f))$, obtained with MATLAB's \texttt{ode45} at its default tolerances, i.e. $\mathrm{RelTol} = 10^{-3}$ and $\mathrm{AbsTol} = 10^{-6}$. \texttt{ode45} is an adaptive but \emph{non-validated} solver, it controls an estimate of the local truncation error, and does not return a rigorous over-approximation of the exact solution. Consequently the quantity $\|\tilde{x}(t_f) - x(t_f)\|$ is controlled in practice but not guaranteed by a certificate, and $[\tilde{x}(t_f), \tilde{\hat{x}}(t_f)]$ is not guaranteed to contain $[x(t_f), \hat{x}(t_f)]$.
\end{enumerate}

\paragraph{Consequence and the statement we make.} From the four statements, (D3) is mathematically proved and (D1) is inherited from CORA. The remaining two hold only approximately in the implementation, as (D2) can be violated by floating-point rounding, and (D4) by the integration error of the numerical solver. The correct characterization of TNODEV's guarantee is therefore:

\begin{quote}
A \texttt{SAFE} verdict of TNODEV certifies that $\Reach(\Xin) \subseteq \mathcal{X}_s$, \emph{under the assumption that the Jacobian bounds derived in Step~1 are valid, and that the numerical integration of the embedded system in Step~3 is exact.}
\end{quote}

This is a \emph{deterministic} guarantee in the sense of Table~\ref{tab:tool-comparison}, i.e. it involves no confidence level and no residual probability of missing a trajectory, and it is what distinguishes TNODEV from the probabilistic methods such as SLR and GoTube \citep{grunbacher2021slr,gruenbacher2022gotube}. What the guarantee does not include is a certificate covering the underlying arithmetic itself, as it is established mathematically, but the computation that produces it is carried out with a non-validated ODE solver, so nothing in the implementation verifies that the computed enclosure really contains the exact one.  This reliance on a non-validated solver is to the best of our knowledge the common situation in the field rather than a particular case of TNODEV. On the ODE reachability side, TIRA \citep{meyer2019tira} is likewise driven by standard non-validated solvers, and the zonotope reachability of CORA \citep{althoff2015cora}, which also handles the ODE blocks inside NNVODE and NNV 2.0 \citep{manzanas2022nnvode,lopez2023nnv} is implemented in ordinary floating-point arithmetic without rounding. On the neural network verification side no ODE is integrated, but the analogous gap still exists, as bound-propagation and LP-based verifiers such as $\alpha$-$\beta$-CROWN \citep{wang2021beta} likewise compute their bounds in plain floating point without certified outward rounding.

\paragraph{Closing the gap.} Two routes can replace statement (D4) by an actual guarantee. The first is validated integration, e.g. the guaranteed Runge--Kutta methods of \citet{sandretto2015validated}  which return a rigorous over-approximation of the exact solution.  Its cost is substantial, and wrapping every reachability call of the refinement loop in a validated integrator would forfeit most of the low per-call cost that motivates the CTMM backend. A practical middle ground which we plan to implement (Appendix~\ref{apd:next_version}) is to keep the fast solver for the exploration phase of the loop and to re-verify only the final partition of \texttt{SAFE} cells with the validated integrator. The validated integrator is then run once per verified cell of the final partition, instead of once at every refinement iteration, which is far fewer calls whenever the loop refines more cells than it ultimately certifies. The second route exploits the structure of the embedded system, as the bracketing property is preserved under any integration scheme that is itself monotone and whose numerical errors can only push the computed bounds outwards. Enlarging the over-approximation costs tightness but never soundness, since the enlarged set still contains every trajectory and a \texttt{SAFE} verdict obtained on it therefore remains valid. What must never happen is the opposite, i.e. an error that shrinks the computed bounds below the exact ones. An explicit Euler scheme with an outward-rounded step bound deviates only in the enlarging direction, and thus yields a sound, though coarse, over-approximation without a validated solver.

\paragraph{Empirical sensitivity to the solver tolerance.} We cannot prove that the integration error is harmless, but we can measure whether it is large enough to matter. We therefore re-compute all the six benchmarks of Table~\ref{tab:results}, with MSIR heuristic and with three solver configurations, from the default tolerances down to $\mathrm{RelTol} = \mathrm{AbsTol} = 10^{-8}$.
Table~\ref{tab:tolerance_sweep} reports two things: the largest change of any endpoint of the enclosure\footnote{For ACC benchmark, the enclosure is compared at every control step of the closed loop reached by the unsplit cell.} returned by one CTMM reachability call on the unsplit initial set, and whether the final verdict changes. Tightening the tolerance by five orders of magnitude moves the enclosure by at most $8.6 \cdot 10^{-7}$ on the Cartpole benchmark, which is several orders of magnitude below the widths of the safe sets of the corresponding specification, and leaves every verdict unchanged. This does not constitute a proof, since all three configurations use the same \texttt{ode45} solver, but it does show that the verdicts reported in this paper are not decided at the scale on which the integration error operates.

\begin{table}[htbp]
\centering
\caption{Sensitivity of the reported verdicts to the ODE solver tolerance using MSIR heuristic for all the benchmarks of Table~\ref{tab:results}. Each benchmark is re-run with progressively tighter solver tolerances, and $\|\Omega_{\text{tol}} - \Omega_{\text{default}}\|_\infty$ reports how much the resulting over-approximation moves, i.e. the largest change, over all state coordinates, of the interval bounds returned by one CTMM reachability call on the unsplit initial set $\Xin$, taking the default-tolerance run as the reference.}
\label{tab:tolerance_sweep}
\small
\begin{tabular}{llcc}
\toprule
\textbf{Benchmark} & \textbf{Solver configuration} & $\|\Omega_{\text{tol}} - \Omega_{\text{default}}\|_\infty$ & \textbf{Verdict} \\
\midrule
 \multirow{3}{*}{Spiral linear} & \texttt{ode45}, default ($10^{-3}$, $10^{-6}$) & --- & SAFE \\
  & \texttt{ode45}, $10^{-6}$, $10^{-6}$ & $9.7 \cdot 10^{-8}$ & SAFE \\
  & \texttt{ode45}, $10^{-8}$, $10^{-8}$ & $3.8 \cdot 10^{-7}$ & SAFE \\
\midrule
 \multirow{3}{*}{Spiral nonlinear} & \texttt{ode45}, default ($10^{-3}$, $10^{-6}$) & --- & FALSIFIED \\
  & \texttt{ode45}, $10^{-6}$, $10^{-6}$ & $1.7 \cdot 10^{-9}$ & FALSIFIED  \\
  & \texttt{ode45}, $10^{-8}$, $10^{-8}$ & $2.4 \cdot 10^{-7}$ & FALSIFIED  \\
\midrule
 \multirow{3}{*}{FPA} & \texttt{ode45}, default ($10^{-3}$, $10^{-6}$) & --- & SAFE  \\
  & \texttt{ode45}, $10^{-6}$, $10^{-6}$ & $7.6 \cdot 10^{-8}$ & SAFE  \\
  & \texttt{ode45}, $10^{-8}$, $10^{-8}$ & $1.9 \cdot 10^{-7}$ & SAFE  \\
\midrule
 \multirow{3}{*}{Cartpole} & \texttt{ode45}, default ($10^{-3}$, $10^{-6}$) & --- & SAFE  \\
  & \texttt{ode45}, $10^{-6}$, $10^{-6}$ & $8.2 \cdot 10^{-8}$ & SAFE  \\
  & \texttt{ode45}, $10^{-8}$, $10^{-8}$ & $8.6 \cdot 10^{-7}$ & SAFE  \\
\midrule
 \multirow{3}{*}{ACC linear} & \texttt{ode45}, default ($10^{-3}$, $10^{-6}$) & --- & SAFE  \\
  & \texttt{ode45}, $10^{-6}$, $10^{-6}$ & $6.3 \cdot 10^{-13}$ & SAFE  \\
  & \texttt{ode45}, $10^{-8}$, $10^{-8}$ & $6.2 \cdot 10^{-13}$ & SAFE  \\
\midrule
 \multirow{3}{*}{ACC nonlinear} & \texttt{ode45}, default ($10^{-3}$, $10^{-6}$) & --- & SAFE  \\
  & \texttt{ode45}, $10^{-6}$, $10^{-6}$ & $8.5 \cdot 10^{-14}$ & SAFE  \\
  & \texttt{ode45}, $10^{-8}$, $10^{-8}$ & $8.5 \cdot 10^{-14}$ & SAFE  \\
\bottomrule
\end{tabular}%
\end{table}

\subsection{Completeness of the verification and refinement loop}\label{apd:completeness}

A natural question about any branch and bound verifier is whether it becomes complete when the splitting of the input set is allowed to continue indefinitely, i.e. whether it always returns \texttt{SAFE} or \texttt{FALSIFIED} on every instance, and hence never returns \texttt{UNKNOWN}. Completeness is exactly the absence of the \texttt{UNKNOWN} verdict, so the question is whether enough refinement can always resolve an inconclusive cell one way or the other. This part makes precise the answer sketched in Section~\ref{subsec:verif_loop}, namely that refinement converges to \emph{exact} reachability in the limit, which yields finite termination on every specification satisfied with a positive margin, i.e. such that even a uniform enlargement of the exact reachable set by some sufficiently small $\gamma > 0$ remains inside the safe set, a notion formalized as robust safety in Definition~\ref{def:robustly-safe} below, while completeness on \emph{all} instances, i.e. including violated specifications and specifications satisfied with zero margin, fails for two identifiable reasons.

\paragraph{Why refinement converges.} Refinement converges because a cell that shrinks to a single point yields the exact trajectory, which follows from the defining condition $g(x, x) = f(x)$ of the decomposition function. When a cell shrinks to a single point $u$, the embedded system \eqref{eq:embedded} is initialized with $x(0) = \hat{x}(0) = u$, and its two halves then coincide for all time, since $\dot{x} = g(x, \hat{x})$ and $\dot{\hat{x}} = g(\hat{x}, x)$ evaluate identically on $x = \hat{x}$, and each reduces to $\dot{x} = g(x,x) = f(x)$. The embedded system thus computes two copies of the exact trajectory of \eqref{eq:node} from $u$, and the enclosure $[x(t_f), \hat{x}(t_f)]$ collapses to the exact reachable point $\Phi(t_f, u)$. Notably, this holds \emph{regardless of the Jacobian bounds} used in the sign-stabilization construction of $g$, i.e. whichever shifting matrix $L_x$ is derived from $[\underline{J}_x, \overline{J}_x]$, the additional terms that $L_x$ introduces in $g$ vanish on $x = \hat{x}$. In particular, the reachable tube bound of Section~\ref{subsec:ctmm}, which is computed once on the full initial set $\Xin$ and reused by every cell and therefore becomes increasingly conservative on small sub-cells, does not affect this limit, since the added conservatism slows the convergence of the over-approximation but cannot prevent it. The Lemma below quantifies the convergence rate, and the Proposition converts it into a termination guarantee.

Before stating them, we record the two reasons why completeness on all instances is out of reach, and why Proposition~\ref{prop:completeness} below is correspondingly restricted to robustly safe specifications. First, the verification and refinement loop of Algorithm~\ref{alg:tnodev}, i.e. its second phase, which runs after the falsification pass has completed, has no branch that can emit \texttt{FALSIFIED}, as a cell is either certified or split, so further refinement can only convert \texttt{UNKNOWN} into \texttt{SAFE}. This is a consequence of the falsification pass running only once, before the loop (Section~\ref{subsec:falsification}), and an in-loop variant of the pass (Appendix~\ref{apd:next_version}) may lift it. Second, even with in-loop falsification, a specification violated only on a set of measure zero (a trajectory landing exactly on the boundary of $\mathcal{X}_s$) is neither certifiable by refinement nor falsifiable by the falsification pass, whose randomly drawn points hit a measure-zero set with probability zero, i.e. the violating initial conditions occupy a region of zero volume inside $\Xin$, so a point drawn uniformly at random from $\Xin$ lands in that region with probability zero, no matter how many points are drawn. Returning \texttt{SAFE} or \texttt{FALSIFIED} on every instance, i.e. completeness as defined at the start of this part, is therefore out of reach for this class of algorithms.

What can be established is finite termination on \emph{robustly} satisfied specifications, under a fairness condition on the splitting heuristic. We state it because it also explains what happens in Appendix~\ref{app:refinement_cost} when one of the three heuristics can fail to converge on a benchmark that another verifies in seconds.

\begin{definition}[Fair refinement]\label{def:fair}
A splitting heuristic is \emph{fair} if along every infinite sequence of nested cells it produces, the maximal radius tends to zero: $\|r(\mathcal{X}_k)\|_\infty \to 0$ as $k \to \infty$.
\end{definition}

The naive heuristic is fair by construction, since it always splits along the coordinate with the largest radius, so after $n$ consecutive splits of a cell every radius has been halved at least once. Neither MSIR nor ING is fair in general, as both select $i^\star$ by a score that may keep selecting the same coordinate, in which case the other coordinates' radii never decrease. Empirically this is what separates them, as MSIR certifies every satisfied specification of Table~\ref{tab:results}, whereas ING converges on FPA only, and its non-converging runs are exactly those in which its sensitivity estimates are dominated by integration noise (Appendix~\ref{app:refinement_cost}).

\begin{definition}[Robustly safe specification]\label{def:robustly-safe}
A specification $(\Xin, \mathcal{X}_s)$ is \emph{robustly safe} with margin $\gamma > 0$ if $\Reach(\Xin) \oplus B_\infty(\gamma) \subseteq \mathcal{X}_s$, where $\oplus$ denotes the Minkowski sum.
\end{definition}

\begin{lemma}[The enclosure width is at most proportional to the cell radius]\label{lem:width-contraction}
Let $g$ be the decomposition function constructed in Step~2 of Section~\ref{subsec:ctmm} and assume it is Lipschitz with constant $L$ in each of its two arguments. Let $\mathcal{X} = [c - r, c + r] \subseteq \Xin$ and let $\Omega(\mathcal{X}) = [x(t_f), \hat{x}(t_f)]$ be the over-approximation produced by the exact solution of the embedded system \eqref{eq:embedded} initialized at $(c - r, c + r)$. Then
\[
  \|\hat{x}(t_f) - x(t_f)\|_\infty \;\le\; 2 \, e^{2 L t_f} \, \|r\|_\infty .
\]
\end{lemma}

\begin{proof}
Let $e(t) = \hat{x}(t) - x(t)$. The embedded system is monotone and initialized with $e(0) = 2r \ge 0$, so $e(t) \ge 0$ for all $t \ge 0$. Differentiating, $\dot{e} = g(\hat{x}, x) - g(x, \hat{x})$, and applying the Lipschitz property of $g$ in each argument gives $\|\dot{e}\|_\infty \le 2 L \|e\|_\infty$. Gr\"onwall's inequality \citep{gronwall1919note} yields $\|e(t_f)\|_\infty \le e^{2 L t_f} \|e(0)\|_\infty = 2 e^{2 L t_f} \|r\|_\infty$.
\end{proof}

\begin{proposition}[Finite termination on robustly safe specifications]\label{prop:completeness}
Let $(\Xin, \mathcal{X}_s)$ be robustly safe with margin $\gamma$, let the refinement heuristic be fair, and assume the enclosure is computed exactly, i.e. (D3). Then there exists a finite refinement depth $d^\star$, and hence a finite number of iterations $K^\star$, such that Algorithm~\ref{alg:tnodev} with $K \ge K^\star$ and $T_{\max} = \infty$ returns \texttt{SAFE}.
\end{proposition}

\begin{proof}
Let $C = 2 e^{2 L t_f}$ and consider any cell $\mathcal{X} = [c - r, c + r] \subseteq \Xin$ with $\|r\|_\infty \le \gamma / (2C)$. By Lemma~\ref{lem:width-contraction}, $\Omega(\mathcal{X})$ has width at most $C \|r\|_\infty \le \gamma/2$ in every coordinate. Since $g(x,x) = f(x)$, the exact trajectory from the centre $c$ satisfies $\Phi(t_f, c) \in \Reach(\Xin)$ and $\Phi(t_f, c) \in \Omega(\mathcal{X})$, so every point of $\Omega(\mathcal{X})$ lies within $\gamma$ of the point $\Phi(t_f, c) \in \Reach(\Xin)$ in the $\infty$-norm. Hence $\Omega(\mathcal{X}) \subseteq \Reach(\Xin) \oplus B_\infty(\gamma) \subseteq \mathcal{X}_s$ and the cell is certified by the specification check.
By fairness, every sequence of nested cells produced by the refinement reaches $\|r\|_\infty \le \gamma/(2C)$ after finitely many splits, since each split halves one coordinate of a compact cell, and the number of splits needed along any such sequence is at most $d^\star = n \lceil \log_2 ( 2C \|r(\Xin)\|_\infty / \gamma ) \rceil$, so the total number of processed cells is finite. Every cell of the final partition is certified, so the queue empties and the algorithm returns \texttt{SAFE}.
\end{proof}

The proposition is deliberately modest, and its two hypotheses are exactly the two ways the loop can fail in practice. The margin $\gamma$ enters the bound on $d^\star$ inversely, so specifications whose reachable set comes close to the safe-set boundary require a number of refinement iterations that is exponential in $n$ and can quickly grow beyond the iteration budgets used in practice, such as our $K = 5000$. This is the mechanism behind the \texttt{UNKNOWN} verdicts we report. Fairness explains the rest, as a heuristic that is not fair may never reduce the radius of the coordinate that actually binds. Adding a fairness fallback, i.e. forcing a naive split after a fixed number of consecutive score-based refinements of the same cell, restores the hypothesis of Proposition~\ref{prop:completeness} for MSIR and ING at a negligible cost (Appendix~\ref{apd:next_version}).

\section{Parallel scheduler: configuration and measured speed-up}\label{apd:parallel_scaling}

\paragraph{Configuration.} All experiments reported in the paper use a MATLAB parallel pool with $W = 16$ workers, one per physical core of the machine described at the beginning of the appendix, and the batch size $b = \lambda W = 80$, where $\lambda = 5$ means simply that each worker receives about five cells per batch rather than exactly one. The value of $\lambda$ was not tuned, and it was chosen so that the variation of per-cell computation time is averaged over several cells per worker. These are the only two parameters of the scheduler, and they are held fixed across all benchmarks and heuristics. The refinement iteration limit is $K = 5000$ and the wall-clock timeout $T_{\max} = 120$ min throughout. 

\paragraph{Two Controlled measurements.} We report two measurements of what the parallel scheduler gives. Both are controlled, in the sense that the compared runs perform the same computation and differ only in the parallel pool settings, so that the measured speed-up can be attributed to the parallelization. Since every split decision of the three heuristics depends only on the cell being split, a verification run that empties the queue processes the same number of cells whatever the batch size $b$, and $b$ only changes the order in which these cells are processed. The sweep over the number of workers of Table~\ref{tab:parallel_sweep}, in which every configuration processes exactly the same $2037$ cells, measures how the scheduler scales on the run of our benchmark set that keeps the pool busiest, and it is the source of the speed-up mentioned in Section~\ref{subsec:parallel}. Table~\ref{tab:parallel_scaling} then compares the serial and the parallel configurations of the toolbox.

\paragraph{Scaling with the number of workers.}  
Table~\ref{tab:parallel_sweep} and Figures~\ref{fig:parallel_scaling_time}--\ref{fig:parallel_scaling_speedup} report the complete sweep $W \in \{1, 2, 4, 8, 16\}$ on the Cartpole~12D benchmark with the naive heuristic, which processes the largest number of cells of our benchmark set ($2037$ processed cells with $1019$ verified subsets) and therefore keeps the refinement queue full enough for the worker count, rather than the queue length, to be the limiting factor. The batch size is scaled with the worker count ($b = 5W$) so that the per-worker load is held constant, and the verdict is \texttt{SAFE} with the same $2037$ processed cells at every $W$, so the five rows time the same computation. To separate a genuinely sequential fraction of the computation from the time lost to managing the pool itself, i.e. dispatching work to the workers and synchronizing them, which grows with the pool size, the last column of Table~\ref{tab:parallel_sweep} reports the Karp--Flatt metric \citep{karp1990measuring}, i.e. the experimentally determined serial fraction $e$. It is obtained by inverting Amdahl's law \cite{hill2008amdahl}, which states that if a fraction $e$ of a computation is sequential, then the speed-up achievable on $W$ workers is $S = 1/\bigl(e + (1-e)/W\bigr)$, and solving this relation for $e$ from the measured speed-up gives $e = (1/S - 1/W)/(1 - 1/W)$, which is exactly the definition of \citet{karp1990measuring}. A value of $e$ that stays constant across the sweep indicates an actual sequential fraction of that size in the computation, whereas a value that grows with $W$ indicates that the pool-management time increases with the pool size.

\begin{table}[htbp]
\centering
\caption{Scaling of the verifier with the number of workers, on Cartpole~12D with the naive heuristic ($b = 5W$), $e$ is the Karp--Flatt experimentally determined serial fraction.}
\label{tab:parallel_sweep}
\small
\begin{tabular}{cccccc}
\toprule
\textbf{Workers $W$} & \textbf{Batch size $b$} & \textbf{ time (min)} & \textbf{Speed-up $S$} & \textbf{Karp--Flatt $e$} \\
\midrule
1  & 5   & 12.35 & --- & ---   \\
2  & 10  & 7.20 & $1.7\times$ & 0.18 \\
4  & 20  & 4.75 & $2.6\times$ & 0.17 \\
8  & 40  & 3.31 & $3.7\times$ & 0.16 \\
16 & 80  & 2.73  & $4.5\times$ & 0.17 \\
\bottomrule
\end{tabular}
\end{table}

\begin{figure}[htbp]
\centering
\includegraphics[width=0.6\textwidth]{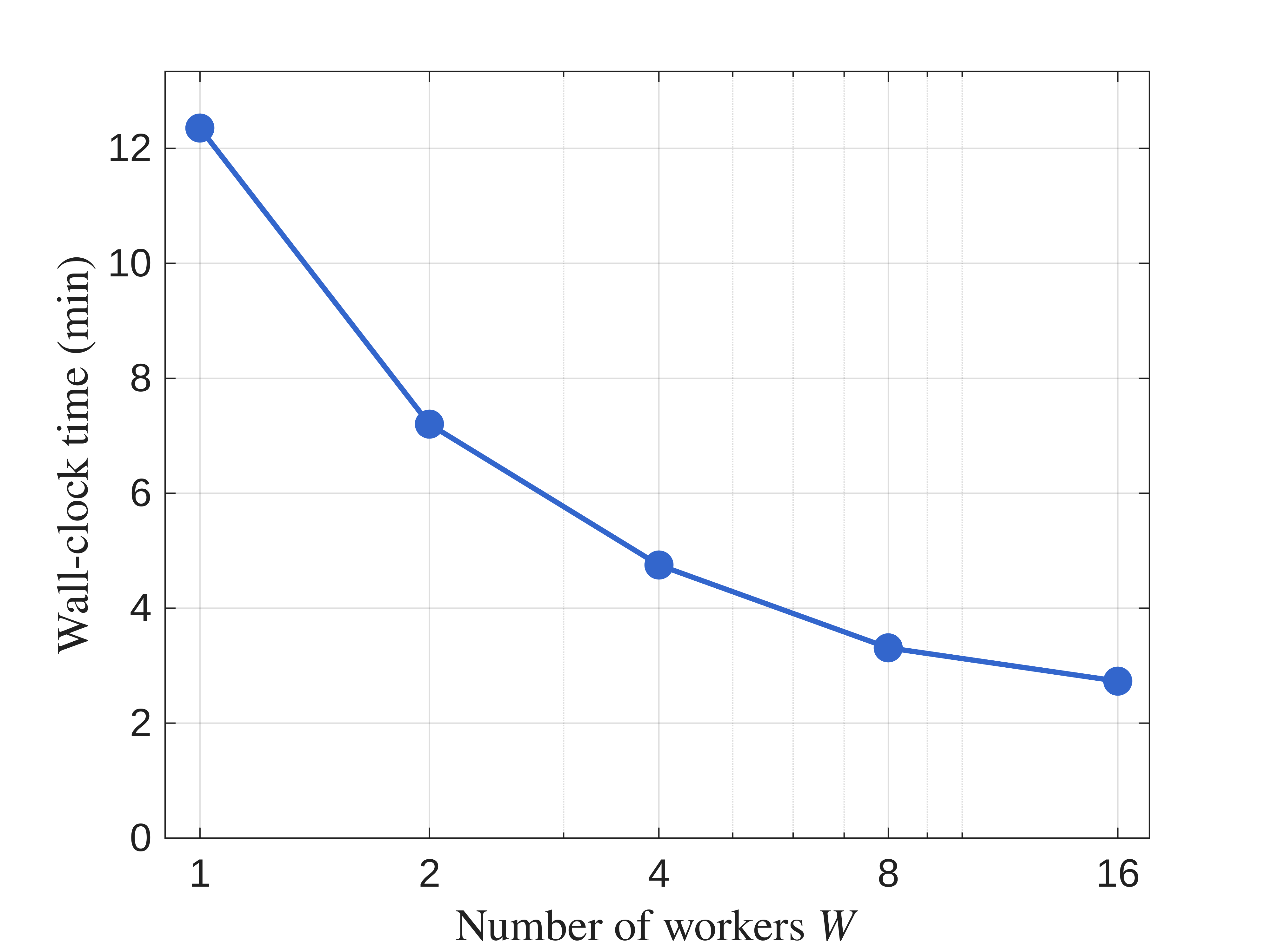}
\caption{Wall-clock time of the verifier as a function of the number of workers $W$, on Cartpole~12D with the naive heuristic ($b = 5W$, $\lambda = 5$). Every run processes the same $2037$ cells regardless of $W$, so the five points differ only in how the work is distributed across the workers, not in the amount of work performed. The sweep stops at the $16$ physical cores of the evaluation machine.}
\label{fig:parallel_scaling_time}
\end{figure}

\begin{figure}[htbp]
\centering
\includegraphics[width=0.6\textwidth]{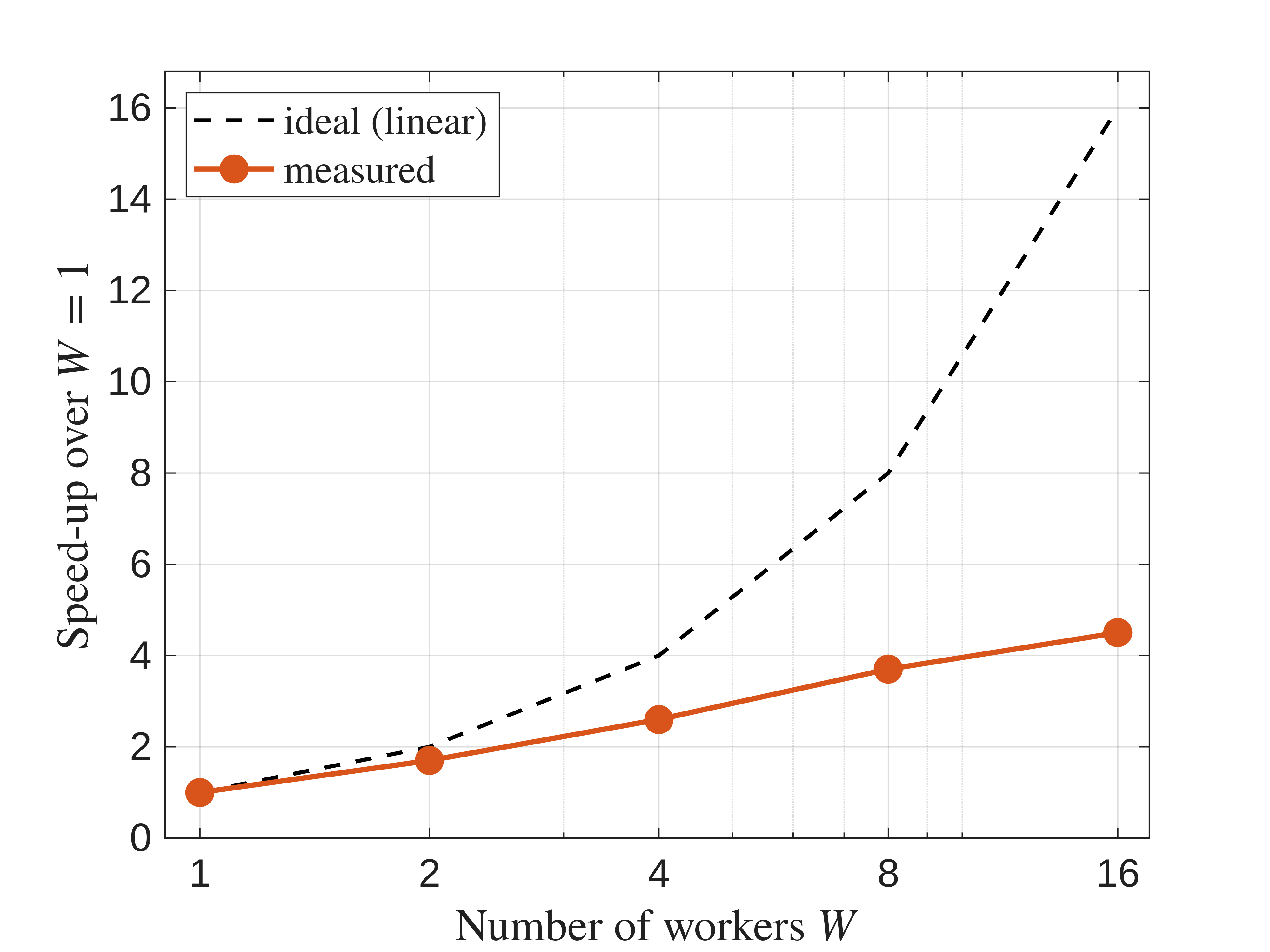}
\caption{Speed-up of the same runs relative to $W = 1$, against the ideal linear reference, which is the line $S = W$, i.e. the speed-up of a perfectly parallelizable computation, with no sequential fraction and no overhead, in which $W$ workers finish exactly $W$ times faster than a single worker. The measured curve lies below this reference and reaches $4.5\times$ on the $16$ workers of the largest configuration we measure. The gap to the ideal line corresponds to the sequential fraction quantified by the Karp--Flatt column of Table~\ref{tab:parallel_sweep}, together with the batch barrier discussed in the text.}
\label{fig:parallel_scaling_speedup}
\end{figure}

Adding workers always reduces the wall-clock time, but by less than the ideal linear rate, from $1.7\times$ on two workers to $4.5\times$ on sixteen. The Karp--Flatt column identifies the cause, as the experimentally determined serial fraction is essentially constant across the sweep, at $e \approx 0.17$, which is the signature of an actual sequential fraction of the computation rather than of pool-management time that grows with the worker count. Amdahl's law turns this fraction into a hard limit, since a sequential part of the computation must be executed one step at a time no matter how many workers are available, so with $e = 0.17$ the speed-up can never exceed $1/e \approx 6\times$. The $4.5\times$ measured on $16$ workers matches what the law predicts for this fraction, as $1/(0.17 + 0.83/16) = 4.5$. This sequential fraction collects everything that does not run inside a batch, namely popping and pushing the queue, dispatching each batch to the pool, and the first and last iterations of the loop, in which the queue holds fewer cells than there are workers. On top of this, the batch barrier adds the difference between the slowest and the average cell of each batch.

\paragraph{Speed-up over serial execution.} Table~\ref{tab:parallel_scaling} quantifies what the parallel scheduler gives in computational speed, by re-running a subset of the benchmarks while restricted to a single worker ($W = 1$, hence $b = 5$), i.e. with the batch processed sequentially, and comparing against the parallel configuration used throughout the paper. We use the MSIR heuristic, and each configuration keeps its own batch size $b = \lambda W$ , since imposing
the serial value b = 5 on the parallel run would leave 11 of its 16 workers idle at every batch and slow
it down artificially.

\begin{table}[htbp]
\centering
\caption{Wall-clock time of the verifier with parallelization ($W = 16$, $b = 80$) vs. without parallelization ($W = 1$, $b = 5$), using MSIR heuristic.}
\label{tab:parallel_scaling}
\small
\begin{tabular}{lcccc}
\toprule
\textbf{Benchmark} & \textbf{Iter.} & \textbf{Serial ($W{=}1$)} & \textbf{Parallel ($W{=}16$)} & \textbf{Speed-up} \\
\midrule
FPA              & 9  & 55.27 s  & 34.50 s  & $1.6\times$ \\
Cartpole         & 63 & 108.5 s  & 56.43 s & $1.9\times$ \\
ACC linear       & 7  & 46.21 min & 2.72 min & $16.9\times$ \\
\bottomrule
\end{tabular}
\end{table}

On the first two rows the measured speed-ups are modest, $1.6\times$ on FPA and $1.9\times$ on Cartpole, and the reason is visible in the iteration column, as both runs terminate within at most $63$ iterations. The scheduler can only exploit the cells that are simultaneously in the queue, so the available parallelism is limited by $|\mathcal{Q}|$ rather than by $W$, and on runs this short the queue rarely holds enough cells to occupy the whole pool. A second effect, the batch barrier, compounds this, as the loop synchronizes at the end of each batch, so an iteration costs the time of its slowest cell rather than of its average one, and the loss grows with the spread of per-cell integration times. Setting $\lambda = 5$ mitigates this loss but does not remove it. The ACC linear row behaves differently, as its measured factor of $16.9\times$ is of the order of the pool size itself and even slightly exceeds it, since $16.9 > W = 16$. A speed-up above $W$ cannot come from distributing identical work across $W$ workers, and it indicates that the two runs do not perform identical work, since changing the batch size from $b = 80$ to $b = 5$ also changes the order in which cells are popped and refined, and on this benchmark, whose per-cell reachability call is by far the most expensive of the set as it propagates the closed loop over all $50$ control steps, the serial schedule follows a more expensive refinement path. The verdicts of the two runs nevertheless agree. This is why we present the parallel scheduler as an engineering feature of the implementation rather than as a scientific contribution of the paper, since what it buys is governed by how much refinement the query needs, not by the verifier's design.

\paragraph{A remark on the interaction with the refinement heuristic.} Comparing Table~\ref{tab:parallel_sweep} with the Cartpole row of Table~\ref{tab:parallel_scaling} shows how the refinement heuristic and the scheduler interact, as the same benchmark and the same code scale by $4.5\times$ under the naive heuristic and only by $1.9\times$ under MSIR. There is no contradiction, and the explanation is the one given above, since MSIR reduces the number of processed cells from $2037$ to $63$, and in doing so it also removes most of the parallelism the scheduler could have exploited. Better splitting decisions and more workers are therefore substitutes rather than complements in this pipeline, which is a further reason to treat the parallel scheduler as an implementation convenience rather than as a source of the results reported in this paper.

\section{Planned extensions for the next version of TNODEV}\label{apd:next_version}

Each item below addresses a limitation that the experiments of Section~\ref{sect:experiments} made concrete, and points to the section or appendix in which that limitation is documented. They are listed in the order in which we expect to implement them, which reflects the ratio of the gain to the implementation effort rather than the size of the gain alone.

\paragraph{A fallback for the ING heuristic.} ING is the only heuristic we do not recommend compared to the other two, and Appendix~\ref{app:refinement_cost} separates the two reasons: a cost that grows with the input dimension, and more fundamentally sensitivity estimates that are dominated by integration noise whenever the enclosure width barely varies with the cell center. Both are detectable at run time, and the second reveals itself when the computed estimates are all of the order of the solver tolerance. The planned version tests for that condition and falls back to MSIR when it holds, which would retain ING's splitting behavior on the benchmarks where it is informative while removing the failure mode that produces all of our \texttt{UNKNOWN} verdicts. A related planned change is the fairness fallback of Appendix~\ref{apd:completeness}, i.e. forcing a naive split after a fixed number of consecutive score-based refinements of the same cell, which restores the hypothesis of Proposition~\ref{prop:completeness} for both MSIR and ING at a negligible cost.

\paragraph{Validated integration.} A \texttt{SAFE} verdict of the current implementation rests on the four statements (D1)--(D4) of Appendix~\ref{apd:soundness_scope}, of which only the last, the numerical integration of the embedded system by \texttt{ode45}, is not backed by a mathematical guarantee, since the solver controls an estimate of its integration error rather than proving a bound on it. This concerns the \emph{soundness} chain of Appendix~\ref{apd:TNODEV_soundness} only, and is independent of the termination result of Proposition~\ref{prop:completeness}, which assumes exact integration of the embedded system such that the approximation describe in (D4) introduces no error. Replacing that solver by a validated integrator \citep{sandretto2015validated} would close the gap, at a cost that is prohibitive if paid at every reachability call. The practical form we intend to implement is the two-phase one, in which the fast solver is kept  while the loop is still splitting cells and searching for a certifiable partition, and only the cells of the final partition, i.e. those the loop reports as verified, are re-verified once with the validated integrator, so that its cost is paid once per verified cell rather than once per refinement iteration. The empirical solver-tolerance sensitivity check of Table~\ref{tab:tolerance_sweep} would then be superseded by a mathematical guarantee that the computed enclosure contains the exact solution of the embedded system.

\paragraph{Moving falsification inside the verification and refinement loop.} The falsification pass runs only once before the verification and refinement loop, which is what makes a \texttt{FALSIFIED} verdict obtainable at a negligible computational cost when a counter-example is easy to find. A violating region of small relative measure in $\Xin$ can escape this single pass, however, and since no later stage of the pipeline can produce \texttt{FALSIFIED}, the query is then reported \texttt{UNKNOWN}. The planned in-loop extension draws points inside each cell that fails the specification check, i.e. each cell whose over-approximation is not fully contained in the safe set $\mathcal{X}_s$, at the cost of a constant number of forward integrations per refinement iteration. Refinement itself makes this sampling progressively more likely to find a violation, since the loop focuses its refinement
effort exactly on the cells that could still contain a violation, i.e. those whose over-approximation intersects the unsafe region. As these cells shrink around the violating region, a fixed number of points drawn inside them covers that region more and more densely, so the chance of hitting a violation may increase with every refinement iteration.
Implementing this change would be fairly easy, but the real challenge of this extension is finding the number of sampled trajectories to compute at each falsification iteration to avoid hindering the faster computation of each iteration that our CTMM reachability provides compare to other tools.

\paragraph{Native reachability for discrete layers.} In the GNODE setting, the ODE block is surrounded by discrete neural network layers, over which TNODEV exercises no reachability, as they are outside its current support scope, and it therefore depends on an external tool (Appendix~\ref{apd:classification_nODE_Hybrid_reachability_pipeline}). This is the reason the refinement loop is not exercised on the MNIST benchmark, and hence the reason that comparison evaluates a single call of TNODEV against a single call of NNV~2.0 rather than the full verifier pipeline the paper is proposing. This will make the GNODE pipeline self-contained and let refinement operate on the classifier input.

\paragraph{Richer set representations.} The reachability backend currently propagates intervals, which is what makes a single call cheap enough to be embedded in a refinement loop, but it is also the source of the one systematic disadvantage we report. On a \emph{linear} ODE block the Star-set propagation of NNV~2.0 is exact, whereas an interval enclosure loses precision at every step, which is why our MNIST verification rate is below NNV~2.0's (Section~\ref{subsec:exp_classification}) and why the ING sensitivity estimates vanish on linear dynamics (Proposition~\ref{prop:ing-degenerate}). Adding support for zonotopes, or polynomial zonotopes, as an alternative backend to CTMM within the same verifier would remove that disadvantage where it matters while retaining the interval backend for the regime in which its per-call cost pays off. The verifier architecture of Section~\ref{sect:verif_arch} was designed to make the reachability block easily changeable for exactly this reason, and also to allow other researchers to plug their own reachability analysis methods for neural ODE into a complete verification pipeline and evaluate them within it.

\end{document}